\documentclass[a4paper,11pt]{article}

\usepackage{xr}
\makeatletter

\newcommand*{\addFileDependency}[1]{
\typeout{(#1)}
\@addtofilelist{#1}
\IfFileExists{#1}{}{\typeout{No file #1.}}
}\makeatother

\usepackage{xr-hyper}
\usepackage{mattsstyle}
\usepackage{tcolorbox}
\usepackage{soul}
\usepackage{bm}
\usepackage{comment}
\usepackage{xfrac}
\usepackage{float}
\usepackage{subcaption}
\usepackage{tikz}
\usepackage{tikz-cd}
\usepackage{makecell}

\usepackage{booktabs,tabularx,ragged2e}
\newcolumntype{Y}{>{\RaggedRight\arraybackslash}X}
\usepackage{adjustbox} 

\def\Lp#1{\mathrm{L}^{#1}}
\def\Hk#1{\mathrm{H}^{#1}}

\def\Wkp#1#2{\mathrm{W}^{#1,#2}}
\def\Hk#1{\mathrm{H}^{#1}}
\def\spaceBar{\, | \,}

\DeclareMathOperator{\range}{ran}

\long\def\aw#1{%
  \ifmmode
    {\color{darkgreen}#1}%
  \else
    {\color{darkgreen}#1}%
  \fi
}

\def\commentOut#1{}

\definecolor{ETHBlue}{RGB}{33,116,174}	
\definecolor{ETHGold}{RGB}{255,209,0}    

\definecolor{UCLADarkerBlue}{RGB}{0,85,135} 
\definecolor{UCLADarkestBlue}{RGB}{0,59,92} 

\definecolor{UCLADarkerGold}{RGB}{255,199,44} 
\definecolor{UCLADarkestGold}{RGB}{255,184,28} 

\definecolor{UCLAPurple}{RGB}{130,55,255} 

\allowdisplaybreaks

\usepackage{pifont}
\title{Kernel Methods for Learning Operators with \\ Multiple Inputs and Outputs}
\author[a]{Adrien Weihs\thanks{Corresponding author. Email: \texttt{weihs@math.ucla.edu}.}}
\author[b]{Chunyang Liao \thanks{This work was initiated while the author was at the University of California Los Angeles.}}
\author[c]{Jingmin Sun}
\author[a]{Hayden Schaeffer}

\affil[a]{Department of Mathematics,\protect\\ University of California Los Angeles,\protect\\ Los Angeles, CA 90095, USA. \vspace{\baselineskip}}
\affil[b]{Department of Mathematical Sciences, \protect\\ University of Arkansas, \protect\\ Fayetteville, AR, 72701, USA. \vspace{\baselineskip}}
\affil[c]{Department of Applied Mathematics and Statistics,\protect\\ Johns Hopkins University,\protect\\ Baltimore, MD 21218 , USA. \vspace{\baselineskip}}

\usepackage{xspace}

\newcommand{\KernelMO}{KernelMO\xspace}
\newcommand{\KernelMOOV}{KernelMO-OV\xspace}
\newcommand{\KernelMOPS}{KernelMO-PS\xspace}

\newcommand{\KernelO}{KernelO\xspace}

\usepackage{pgfplots}
\pgfplotsset{compat=1.18}
\usepgfplotslibrary{colormaps}

\date{ }

\begin{document}

\maketitle

\begin{abstract}
\noindent Learning mappings between infinite-dimensional objects is a central challenge in scientific machine learning. We introduce a general kernel-based encoder-decoder framework for operator learning that separates observation, representation, learning, and reconstruction. We develop this framework for multi-input, multi-output operator learning, where operators map between products of potentially distinct function spaces. Our approximation theory shows that, although the number of inputs and outputs can increase, the convergence rate is governed by the most challenging constituent approximation problem rather than the overall problem dimension. The framework leads to practical kernel methods with closed-form training and inference, combining mathematical tractability with computational efficiency. 
We further specialize the approach to multiple operator learning by introducing \KernelMO, a family of kernel methods with complementary operator-valued and product-space formulations.
Across five families of parametric partial differential equations, the proposed methods achieve competitive or state-of-the-art predictive accuracy while reducing training and inference costs relative to neural operator architectures and deep learning based models, offering an efficient and lightweight alternative.\footnote{An implementation of \KernelMO, together with the code used to reproduce all numerical experiments, is available at \url{https://github.com/liaochunyang/kernelMO}.}

\end{abstract}



\section{Introduction}

\begin{figure}
\centering
\begin{tikzpicture}
\begin{axis}[
    hide axis,
    axis equal image,
    clip=false,
    xmin=-5.3, xmax=5.3,
    ymin=-4.4, ymax=5.1,
    scale only axis,
    width=9cm,
    height=7.4cm
]

\addplot [
    patch,
    shader=interp,
    mesh/color input=explicit,
    mesh/colorspace explicit color output=rgb,
    data cs=polar,
] coordinates {
    (90,4)  [color=UCLAPurple!80]
    (210,4) [color=red!80]
    (-30,4) [color=UCLADarkerBlue]
};

\node[text=black, font=\bfseries\small, align=center] at (axis cs:0,4.6) {Theory};
\node[text=black, font=\bfseries\small, align=center] at (axis cs:-3,-2.8) {Empirical performance};
\node[text=black, font=\bfseries\small, align=center] at (axis cs:3.7,-2.8) {Scalability};

\addplot[
    only marks,
    mark=*,
    mark size=3pt,
    black
] coordinates {
    (-0,-1.6)   
    (-0.0,0)    
    (-1.5,0.6)    
};

\node[text=white, font=\bfseries\footnotesize, align=center] at (axis cs:1.4,-1.45) {Foundation\\models};
\node[text=white, font=\bfseries\footnotesize] at (axis cs:0.8,-0.05) {MNO};
\node[text=white, font=\bfseries\footnotesize, align=center] at (axis cs:-0.4,1.1) {Kernel\\methods};

\end{axis}
\end{tikzpicture}
\caption{
Conceptual illustration of the trade-offs between theoretical guarantees, empirical performance, and scalability in learning operators with multiple inputs and outputs. Foundation models prioritize scalability across large collections of operators, dedicated neural architectures such as MNO seek a balance between scalability and predictive performance, while the kernel-based methods proposed in this work target the moderate-data regime by emphasizing theoretical guarantees together with competitive predictive accuracy. 
}
\label{fig:tradeoff}
\end{figure}
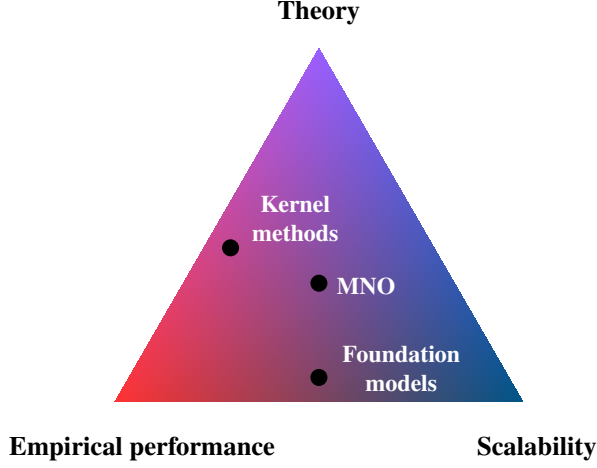

Learning families of related operators between function spaces is an increasingly important problem in scientific machine learning, arising from parametric partial differential equations, inverse problems, and multi-fidelity simulation \cite{weihs2026generalizationboundsstatisticalguarantees}. This setting, commonly referred to as \emph{multiple operator learning} or \emph{multi-task operator learning}, seeks to learn a map
\[
G:W\longrightarrow \{G[\alpha]:U\to V\}_{\alpha\in W},
\]
which assigns to each parameter (or task) \(\alpha\in W\) an operator
\(
G[\alpha]:U\to V,
\)
where \(W\), \(U\), and \(V\) are typically function spaces. 
A variety of approaches have been proposed for multiple operator learning, see for example \cite{sun2025foundation,liu2024prose,mccabe2023multiple,cao2024vicon,zhang2024modno,liu2025bcat,ye2025pdeformer,Jollie_2025,herde2024poseidon,weihs2025MNO,wang2026opinfllmparametricpdesolving}. As illustrated conceptually in Figure~\ref{fig:tradeoff}, existing methods can be viewed as occupying different positions in the trade-off between scalability, empirical performance, and theoretical guarantees. Large-scale foundation models leverage massive datasets to learn highly expressive representations across diverse families of PDEs and scientific simulations. These models have demonstrated remarkable empirical performance and excellent scalability, although their theoretical understanding remains comparatively limited. Dedicated architectures, such as Multiple Neural Operators (MNO) \cite{weihs2025MNO,weihs2026generalizationboundsstatisticalguarantees,weihs2026multipleneuraloperatorsachieve}, instead exploit the structure of the learning problem to achieve strong empirical performance while benefiting from increasing theoretical support.

Many scientific computing applications; however, are not constrained by scalability to large training datasets. Instead, they prioritize predictive accuracy, mathematical guarantees, and computational efficiency in moderate-data settings. This motivates the development of learning architectures tailored to this regime by utilizing the general encoder-decoder formulation illustrated in Figure~\ref{fig:cd-general-form}: inputs and outputs are first encoded into latent spaces, a surrogate is learned between these representations, and the prediction is subsequently decoded to the original output space.
A key feature of this framework is that the surrogate can be chosen independently of the encoders and decoders: whereas existing approaches predominantly employ deep neural networks for this purpose \cite{deepOnet,mionet,lanthalerPCAnet}, we instead utilize kernel-based surrogate models. 

Importantly, the proposed kernel-based encoder-decoder approach naturally extends beyond classical multiple operator learning to general multi-input, multi-output operator learning problems, i.e., learning maps of the form
\begin{equation} \label{eq:intro:multiMap}
    G:X_1\times\cdots\times X_m
\longrightarrow
Y_1\times\cdots\times Y_n,
\end{equation}
where each \(X_i\) and \(Y_j\) may be a distinct function space. We show that our approach is mathematically tractable in this generalized setting: it admits rigorous approximation guarantees, scales favorably with the numbers of input and output tasks, and leads to highly efficient computational implementations.

\subsection{Main Contributions}

The main contributions of this work can be summarized as follows. 

\begin{enumerate}
    \item \textbf{A general kernel framework for encoder--decoder learning.}
    We develop a general kernel-based framework for learning maps between Hilbert spaces within an encoder--decoder architecture. Specifically, in Theorem~\ref{thm:induced-kernel-original-spaces}, we show that every kernel defined on the latent surrogate space induces a corresponding kernel on the original input and output spaces. In addition, in Corollary~\ref{cor:induced-measurement-space-equivalence}, we establish an equivalence between learning in the latent reproducing kernel Hilbert space and learning with the induced kernel on the original spaces. Together, these results provide a rigorous theoretical foundation for kernel-based encoder--decoder learning.

\item \textbf{An encoder--decoder error decomposition.}
We derive an encoder--decoder error decomposition in Theorem \ref{thm:encoder-decoder-learning-error} of the form
\[
\text{Approximation error}
\;\lesssim\;
\text{Representation error}
+
\text{Learning error}
+
\text{Data consistency error},
\]
thereby separating the contributions of the encoder--decoder architecture, the learned surrogate, and the measurement process.

    \item \textbf{Approximation theory for multi-input, multi-output operator learning.}
    We specialize this framework to generalized multi-input, multi-output operator learning on product Hilbert spaces of functions. For a map as in \eqref{eq:intro:multiMap}, we establish rigorous approximation guarantees for the proposed kernel-based encoder--decoder architecture in Theorem~\ref{thm:total-product-sobolev-error}:
    \begin{align*}
        \text{Approximation error}
\;&\lesssim\;
\max_{1 \leq i \leq m} \{\text{Input reconstruction error}_i\} + \max_{1 \leq j \leq n}\{\text{Output reconstruction error}_j\} \\
&\qquad +
\text{Learning error}
+
\text{Data consistency error}.
    \end{align*}
    Importantly, these guarantees scale favorably with the numbers of input and output tasks: the approximation rate is governed by the most challenging individual task rather than deteriorating with the overall problem dimension.

    \item \textbf{Kernel formulations for multiple operator learning.}
In Section~\ref{sec:multiple}, we instantiate the proposed framework for multiple operator learning through the \KernelMO family of methods. Depending on the representation of the target map, this yields the operator-valued formulation \KernelMOOV, which directly approximates
\(
G:W\to\{G[\alpha]:U\rightarrow V\}_{\alpha\in W},
\)
or the product-space formulation \KernelMOPS, which instead approximates the equivalent map
\(
G:W\times U\rightarrow V.
\)
Both formulations inherit the approximation guarantees of the general framework while exhibiting different computational complexities and scalability with respect to the number of training samples.

    \item \textbf{Comprehensive empirical evaluation.}
    We validate the proposed methods on a diverse collection of multiple operator learning benchmarks arising from parametric PDEs in Section \ref{sec:experiments}. Our experiments demonstrate competitive predictive accuracy, substantial reductions in training and inference time compared with state-of-the-art neural operator architectures, and illustrate the practical advantages of kernel-based encoder--decoder learning in moderate-data regimes.
    
\end{enumerate}

\subsection{Related Works}

\paragraph{Multi-operator learning}

Multiple operator learning arises in a variety of settings where one seeks to approximate not a single operator, but an entire collection of related input--output mappings. Such collections may emerge naturally from parameterized models, changing geometries, varying boundary conditions, or different governing equations. Alternatively, they may be constructed deliberately to enable information sharing across related learning tasks, with the goal of improving sample efficiency, robustness, and generalization. These viewpoints have motivated a rapidly expanding literature on multiple operator learning and related frameworks; see, for example, \cite{sun2025foundation,liu2024prose,mccabe2023multiple,yang2023incontext,yang2023prompting,cao2024vicon,zhang2024modno,zhang2024d2no,liu2025bcat,ye2025pdeformer,zhang2025probabilistic,Jollie_2025,herde2024poseidon,bacho2025operatorlearningmachineprecision,weihs2025MNO,weihs2026generalizationboundsstatisticalguarantees,wang2026opinfllmparametricpdesolving, zhu2025pi, zhang2025deeponet, sun2024lemon}. Recently, \cite{yang2026generalizationguaranteesmultiinputneural} extended the analysis to the general multi-input, multi-output operator learning framework.

From a modeling perspective, existing approaches can largely be grouped into two categories. The first treats each operator as an independent learning problem, training a separate surrogate for every member of the collection. While simple, this approach cannot exploit relationships between operators and therefore offers limited transfer to new tasks. The second instead represents the collection as a single parameterized operator family, in which an auxiliary variable identifies the desired operator. This variable may encode, for example, physical parameters, a task identifier, the governing equations, or even textual descriptions of the problem \cite{sun2025foundation,liu2024prose,yang2023prompting,negrini2025multimodal,liu2024prosefd,weihs2025MNO,weihs2026generalizationboundsstatisticalguarantees,weihs2026multipleneuraloperatorsachieve}. By conditioning on such auxiliary information, a single model can share information across related operators and often exhibits improved transfer, zero-shot capabilities, and out-of-distribution generalization. This structure forms the basis of most recent work on PDE foundation models.

\paragraph{Kernel-based surrogate modeling}

Kernel-based methods have long served as a fundamental framework for surrogate modeling due to their strong approximation property and rigorous theoretical foundations.
By representing the target function in a reproducing kernel Hilbert space (RKHS), these methods provide flexible, nonparametric approximations from scatter data while naturally supporting interpolation, regularization, and uncertainty quantification. 
Classical kernel surrogates include kernel interpolation, kernel ridge regression, and Gaussian process regression, all of which have found widespread use in scientific computing; see, for example, \cite{Stuart2010Inverse,Harlim2020Kernel,CHEN20211kernelPDE,Hou2023Sparse,BATLLE2025Error,jalalian2025dataefficientkernelmethodslearning}.

These ideas have been extended to operator learning, where the objective is to learn nonlinear mappings between function spaces rather than finite-dimensional input-output pairs \cite{jin2023minimax,BATLLE2024112549, bacho2025operatorlearningmachineprecision, zhangBelnet, zhang2025discretization, MORA2025operator,kempf2026kernelbasedoperatorlearningerror}.
By extending reproducing kernel Hilbert space (RKHS) techniques from finite-dimensional regression to mappings between function spaces, it provides a principled framework for learning nonlinear operators while admitting rigorous analyses of approximation accuracy, stability, and generalization.  
Numerical studies have demonstrated that kernel-based operator learning models can achieve competitive performance compared with neural networks-based operator learning methods. These advantages make kernel operator learning both a mathematically principled framework for constructing surrogates of complex solution operators and an important benchmark for evaluating the accuracy and generalization capability.

Although kernel operator learning has a theoretical foundation, simple implementation, and competitive performance, existing methods rely on computationally expensive kernel representation. 
This motivates the development of scalable kernel surrogate models that retain the theoretical advantages of kernel-based methods while providing computational efficiency; see, for example, \cite{Nelson2024operatorRFM,Liao2025Cauchy,Yu2026Regularized}.

\paragraph{Organization of the paper} The remainder of the paper is organized as follows. In Section \ref{sec:back}, we introduce the assumptions and mathematical framework underlying our analysis. Section \ref{sec:main} presents the main theoretical results, including the general encoder--decoder framework, its approximation theory for multi-input, multi-output operator learning, and its specialization to the \KernelMO family of methods for multiple operator learning. In Section \ref{sec:experiments}, we numerically evaluate the proposed methods across a range of multiple operator learning problems. Proofs of all theoretical results are collected in Appendix \ref{sec:proofs}.

\section{Background} \label{sec:back}

In this section, we review the mathematical background required throughout the paper. We begin by introducing the framework of minimum-norm recovery from bounded linear measurements, which provides a canonical reconstruction procedure and forms the basis for the encoder--decoder construction developed in subsequent sections. We then briefly recall the theory of operator-valued reproducing kernel Hilbert spaces, which serves as the learning space for the latent surrogate operators considered in this work.

\subsection{General Notation}

For a space $X$, the map $\Id_X$ denotes the identity map on $X$. We write $\cL(X)$ for the set of bounded linear operators mapping $X$ to $X$. For a map $L$, we denote its domain by $D(L)$, its adjoint by $L^*$, its kernel by $\ker(L)$, and its range by $\range(L)$. We say that a map is boundedly invertible if it is bijective and its inverse is bounded. For $x\in\mathbb{R}$, we write
\(
(x)_+ := \max\{x,0\}
\)
for its positive part. For $x\in X$ and $r>0$, we denote by
\(
B_X(x,r):=\{y\in X:\|x-y\|_X<r\}
\)
the open ball of radius $r$ centered at $x$, and write $B_X(r):=B_X(0,r)$. For an open set $\Omega\subset\mathbb{R}^d$, we denote by $\Wkp{s}{p}(\Omega)$ the Sobolev space of order $s\ge0$ and integrability exponent $1\le p\le\infty$. Its seminorm is denoted by
\(
|\cdot|_{\Wkp{s}{p}(\Omega)},
\)
and its norm by
\(
\|\cdot\|_{\Wkp{s}{p}(\Omega)}.
\)
When $p=2$, we write
\(
\Hk{s}(\Omega):=\Wkp{s}{2}(\Omega).
\)

\subsection{Optimal Recovery from Linear Measurements} \label{sec:optimalRecovery}

In this section, we introduce minimum-norm recovery from bounded linear measurements in a Hilbert-space setting. This provides a canonical way of reconstructing functions from observations. We allow the observation space to be an arbitrary Hilbert space, thereby encompassing finite collections of measurements as well as infinite-dimensional observation spaces. Typical examples include local averages \cite{Babenko10092023}, integral functionals, and Fourier coefficients \cite{magaril_il_yaev_2025_tnfsf-8yb65,K_Yu_Osipenko_2001}. In operator learning, the most common setting is a finite collection of pointwise measurements of the input and output functions \cite{BATLLE2024112549}. The following result gives explicit formulas for the reconstructed function. For completeness, we provide a proof in Section~\ref{sec:proof:back}.

\begin{theorem}[Minimum-norm recovery]\label{thm:minimumNorm}
    Let \(X\) and \(Z\) be Hilbert spaces and let
\(
    L:X\to Z
\)
be a bounded linear measurement operator. Given a measurement vector
\(S\in Z\) and $\gamma >0$, consider the minimum-norm recovery problem
\[
    \overline{x}(S)
    =
    \argmin_{x\in X}
    \{\|x\|_X \spaceBar Lx=S\}.
\]
and the regularized recovery problem
\[
    \overline{x}_\gamma(S)
    =
    \argmin_{x\in X}
    \left\{
        \|x\|_X^2+\gamma^{-1}\|Lx-S\|_Z^2
    \right\}.
\]
\begin{enumerate}
    \item If $L$ is surjective, the unique minimizer $\overline{x}(S)$ is given by 
\begin{equation}
    \label{eq:master-minimum-norm-formula}
    \overline{x}(S)
    =
    L^\ast(LL^\ast)^{-1}S.
\end{equation}
\item The unique minimizer $\overline{x}_\gamma(S)$ is given by
\begin{equation}
    \label{eq:master-regularized-minimum-norm-formula}
    \overline{x}_\gamma(S)
    =
    L^\ast(LL^\ast+\gamma I_Z)^{-1}S.
\end{equation}
\end{enumerate}

\end{theorem}

In the encoder--decoder framework developed in Section \ref{sec:general}, the measurement operator \(L\) of Theorem \ref{thm:minimumNorm} will play the role of an encoder, while the corresponding minimum-norm or regularized recovery map will provide a natural choice of decoder.

\subsection{Reproducing Kernel Hilbert Spaces}

In this section, we review the necessary background on kernels and RKHSs. In particular, RKHSs are of interest because their structure equips the general encoder--decoder framework with mathematical structure and leads to well-posed learning problems (see Theorem \ref{thm:induced-kernel-original-spaces} and Corollary \ref{cor:induced-measurement-space-equivalence}). Moreover, RKHSs that embed continuously into Sobolev spaces (see Remark~\ref{rem:RKHSEmbedding}), provide a setting in which quantitative reconstruction error estimates can be derived, as developed in Section~\ref{sec:mainMulti}.

We first define an operator-valued kernel and then the corresponding function-valued RKHS, which were introduced in \cite{hachem2016operator}. Let $\cX$ and $\cY$ be separable Hilbert spaces endowed with the inner products $\langle\cdot,\cdot\rangle_\cX$ and $\langle\cdot,\cdot\rangle_\cY$. 

\begin{mydef}[Operator-valued Kernel]
We call $K:\cX\times\cX\to \cL(\cY)$ an operator-valued kernel if 
\begin{enumerate}
\item $K$ is Hermitian, i.e. $K(x,x') = K(x',x)^\top$ for all $x,x'\in\cX$, writing $A^\top$ for the adjoint of the operator $A$ with respect to $\langle\cdot,\cdot\rangle_\cY$.
\item $K$ is non-negative, i.e. for all $m\in \bbN$ and any set of points $\{(x^{(i)},y^{(i)}\}_{i=1}^m \subset \cX\times\cY$ it holds that $\sum_{i,j=1}^m\langle y^{(i)},K(x^{(i)},x^{(j)})y^{(j)}\rangle_\cY \geq 0$
\end{enumerate}
\end{mydef}
We call $K$ non-degenerate if $\sum_{i,j=1}^m\langle y^{(i)},K(x^{(i)},x^{(j)})y^{(j)}\rangle_\cY = 0$ implies $y^{(i)}=0$ for all $i$ whenever $x^{(i)}\neq x^{(j)}$ for all $i\neq j$.

\begin{mydef}[Function-valued Reproducing Kernel Hilbert Space]
A Hilbert space $\cH$ of functions from $\cX$ to $\cY$ is called a function-valued reproducing kernel Hilbert space if
there is a nonnegative $\cL(\cY)$-valued kernel $K$ on $\cX\times\cX$ such that:
\begin{enumerate}
\item the function $z\mapsto K(w,z)g$ belongs to $\cH$, for every $z,w\in\cX$ and $g\in\cY$,
\item for every $H\in\cH$, $w\in\cX$, and $g\in\cY$, $\langle H, K(w,\cdot)\rangle_\cH = \langle H(w),g\rangle_\cY$.
\end{enumerate}
\end{mydef}
There exists a bijection between non-negative kernel and RKHS, which is summarized in the following theorem \cite[Theorem 2]{hachem2016operator}.
\begin{theorem}[Bijection between function-valued RKHS and operator-valued kernel]
A $\cL(\cY)$-valued kernel $K$ on $\cX\times\cX$ is the reproducing kernel of some Hilbert space $\cH$, if and only if it is non-negative.    
\end{theorem}
The above theorem is parallel to the theorem of bijection between scalar-valued kernels and RKHS, which was first established by Aronszajn in \cite{Aronszajn1950RKHS}.

We note that the classical scalar- and vector-valued RKHS can be viewed as special cases of the function-valued RKHS. In particular, when we take the output space $\cY=\bbR$, the kernel function $K:\cX\times\cX\to \cL(\cY)$ reduces to a scalar-valued function $K:\cX\times\cX\to\bbR$, which is exactly the classical scalar-valued kernel function \cite{Foucart2022DS}.  
When the output space is $\cY=\bbR^d$, then bounded linear operators mapping $\bbR^d$ to $\bbR^d$ are simply all $d\times d$ matrices, i.e. $\cL(\cY) = \bbR^{d\times d}$. Therefore, we obtain a matrix-valued kernel function $K:\cX\times\cX\to \bbR^{d\times d}$ corresponding to the classical vector-valued RKHS setting \cite{alvarez2012vectorkernel}.
Scalar- and matrix-valued kernels are widely used in scalar- and vector-valued function approximation problems \cite{Wendland2004Scatter,alvarez2012vectorkernel}.
In practice, a simple choice of a matrix-valued kernel is the block-diagonal kernel
\begin{equation*}
K(x,x') = S(x,x') \Id_{d\times d}, 
\end{equation*}
where $S(x,x'):\cX\times\cX\to\bbR$ is a scalar-valued kernel. 
The block-diagonal matrix-valued kernel also implies that we can separately approximate each component of a vector-valued function.

\begin{remark}[RKHS interpolation as minimum-norm recovery]
Minimum-norm recovery from Theorem \ref{thm:minimumNorm} encompasses classical interpolation in reproducing kernel Hilbert spaces. When the ambient Hilbert space is an operator-valued RKHS and the observations consist of pointwise evaluations, the minimum-norm recovery problem is precisely the standard RKHS interpolation problem and the minimum-norm interpolant is therefore given by the classical vector-valued kernel interpolant (see the proof of Corollary \ref{cor:induced-measurement-space-equivalence}).
\end{remark}

\section{Main Results} \label{sec:main}

In this section, we develop the main theoretical framework of the paper. We first establish an abstract encoder--decoder formulation together with a decomposition of the prediction error into representation, learning, and data consistency terms. We then derive approximation guarantees for the multi-input, multi-output operator learning problem and characterize its approximation complexity in terms of the approximation complexities of the individual input-output operator learning tasks. Finally, we specialize the general framework to multiple operator learning, yielding practical learning algorithms. 

\subsection{Assumptions}

We review the assumptions used throughout our results. 

\begin{assumptions} We make the following assumption on our spaces.

\begin{enumerate}[label=\textbf{S.\arabic*}]

\item The spaces
\(
\mathcal H
\bigl(
J,
(\Omega_j)_{j=1}^J,
(n_j)_{j=1}^J,
(s_j)_{j=1}^J,
(p_j)_{j=1}^J,
(t_j)_{j=1}^J,
(q_j)_{j=1}^J,
(A_j)_{j=1}^J
\bigr)
\)
and \(
X
\bigl(
J,
(\Omega_j)_{j=1}^J, \allowbreak
(t_j)_{j=1}^J, \allowbreak
(q_j)_{j=1}^J
\bigr)
\)
are function sets such that
\begin{enumerate}
\item 
    \(
    \Omega_j\subset\mathbb R^{n_j}
    \)
    is a bounded domain with Lipschitz boundary;
\item the space
\(
\mathcal H
=
\prod_{j=1}^J\mathcal H_j
\)
is the Hilbert product of spaces \(\mathcal H_j\) of real-valued functions on
\(\Omega_j\), each satisfying the continuous embedding
\(
\mathcal H_j
\hookrightarrow
\Wkp{s_j}{p_j}(\Omega_j),
\)
and is equipped with the canonical product norm
\(
\|u\|_{\mathcal H}^2
=
\sum_{j=1}^J
\|u_j\|_{\mathcal H_j}^2.
\)
    \item the space X is given by \(
    X = \prod_{j=1}^J
\Wkp{t_j}{q_j}(\Omega_j)
    \)
    and is equipped with the the product norm
\(
\|u\|_X^2
=
\sum_{j=1}^J
\|u_j\|_{\Wkp{t_j}{q_j}(\Omega_j)}^2.
\)

    \item for each \(j=1,\ldots,J\), \(p_j,q_j\in[1,\infty]\) and
    \[
    s_j\ge n_j
    \quad\text{if }p_j=1,
    \qquad
    s_j>\frac{n_j}{p_j}
    \quad\text{if }1<p_j<\infty,
    \qquad
    s_j\in\mathbb N^*
    \quad\text{if }p_j=\infty;
    \]

    \item for each \(j=1,\ldots,J\), defining
    \(
    \ell_{0,j}
    :=
    s_j
    -
    n_j
    \left(
    \frac{1}{p_j}-\frac{1}{q_j}
    \right)_+,
    \)
    we have
    \[
    t_j\in\mathbb N_0,
    \qquad
    0\le t_j<\ell_{0,j};
    \] 
    \item for each \(j=1,\ldots,J\), let
    \(
    A_j
    =
    \{x_j^1,\ldots,x_j^{m_j}\}
    \subset\Omega_j
    \)
    be a finite sampling set with fill distance
    \[
    h_j
    :=
    \sup_{x\in\Omega_j}
    \min_{a\in A_j}
    \|x-a\|_2;
    \]
\end{enumerate}

\label{assumption:productSobolevSpace}

\end{enumerate}

Since
\(
s_j-t_j
>
n_j
\left(
\frac{1}{p_j}-\frac{1}{q_j}
\right)_+,
\)
the Sobolev embedding theorem \cite{adams2003sobolev}, together with the finite-measure
embedding when \(q_j<p_j\), yields
\(
\Wkp{s_j}{p_j}(\Omega_j)
\hookrightarrow
\Wkp{t_j}{q_j}(\Omega_j)
\)
continuously.
Consequently,
\(
\mathcal H_j
\hookrightarrow
\Wkp{t_j}{q_j}(\Omega_j)
\)
continuously for every \(j\), and hence
\(
\mathcal H\hookrightarrow X
\)
continuously. Moreover, the Sobolev embedding theorem also yields the continuous embedding
\(
\mathcal H_j
\hookrightarrow W^{s_j,p_j}(\Omega_j)
\hookrightarrow
C^0(\overline{\Omega_j}).
\)

The spaces \(\mathcal H\) and \(X\) play distinct roles throughout the paper.
The space \(X\) is the domain of the operator \(G\), and all approximation
errors are measured in the weaker norm of \(X\). In contrast, \(\mathcal H\)
serves as a regularity space: its stronger norm quantifies the additional
smoothness required to obtain interpolation and reconstruction estimates from
finitely many observations. This separation mirrors classical approximation
theory, where errors are typically measured in a weak norm while convergence
rates depend on higher-order regularity. For example, finite element estimates
take the form
\[
\|u-I_hu\|_{\Lp{2}(\Omega)}
\le
Ch^k
\|u\|_{\Hk{k}(\Omega)},
\]
where the error is measured in the weaker \(\Lp{2}\)-norm, whereas the rate is
governed by the stronger \(\Hk{k}\)-norm.

\end{assumptions}

\begin{remark}[Hilbert spaces embedded in Sobolev spaces] \label{rem:RKHSEmbedding}
An important class of examples for $\cH$ is provided by reproducing kernel Hilbert spaces.
Indeed, for several kernels used in approximation theory, such as Mat{\'e}rn and Wendland
kernels, the associated RKHS is continuously embedded in, or
is norm-equivalent to, a Sobolev space of finite smoothness (see
\cite[Chapter 10]{Wendland2004Scatter}).
\end{remark}

\begin{assumptions} We make the following assumptions on our encoding and decoding maps.
\begin{enumerate}[label=\textbf{M.\arabic*}]

\item For $R > 0$, we have $D_X E_X B_R(X) \subset B_R(X)$. \label{assumptions:M2}

\item The encoding/decoding maps and measurement space
\(
(E_{\cH},D_{\cH},Z_\cH)
\)
associated with a space
\[
\mathcal H\bigl(
J,
(\Omega_j)_{j=1}^J,
(n_j)_{j=1}^J,
(s_j)_{j=1}^J,
(p_j)_{j=1}^J,
(t_j)_{j=1}^J,
(q_j)_{j=1}^J,
(A_j)_{j=1}^J
\bigr)
\]
satisfying Assumption~\ref{assumption:productSobolevSpace} are defined as follows:
\begin{enumerate}
    \item we define the point-evaluation map $E_j:\mathcal H_j\to\mathbb R^{m_j}$ by $E_ju_j
    :=
    \bigl(
    u_j(x_j^1),\ldots,u_j(x_j^{m_j})
    \bigr)$
    and let
    \(
    Z_j:=\operatorname{Ran}(E_j);
    \)

    \item let
    \(
    D_j:Z_j\to\mathcal H_j
    \)
    denote the minimum-norm interpolant from Theorem~\ref{thm:minimumNorm};

    \item we define the measurement space
    \(
    Z_{\cH}
    :=
    \prod_{j=1}^J Z_j,
    \)
    together with the componentwise encoding and decoding maps
    \[
    E_\cH:\mathcal H\to Z_\cH,
    \qquad
    E_{\cH}u
    :=
    (E_ju_j)_{j=1}^J,
    \]
    and
    \[
    D_\cH:Z_\cH \to\mathcal H,
    \qquad
    D_\cH U
    :=
    (D_jU_j)_{j=1}^J.
    \]
\end{enumerate} \label{assumptions:M1}

\end{enumerate}

As shown in the proof of Theorem~\ref{thm:minimumNorm}, if $D_X$ is the minimum-norm interpolant, then Assumption~\ref{assumptions:M2} holds automatically.

\end{assumptions}

\begin{assumptions}
We make the following regularity assumptions on the map \(G:X\to Y\).
\begin{enumerate}[label=\textbf{O.\arabic*}]
\item
There exist \(R>0\) and a nondecreasing function
\(
\omega:[0,\infty)\to[0,\infty)
\)
with 
$\omega(0)=0$
such that, $
B_R(X)\subset\mathcal D(G)$ and, for every \(x,x'\in B_R(X)\), 
\[
\|G(x)-G(x')\|_Y
\le
\omega\left(\|x-x'\|_X\right).
\]
\label{assumption:regularityG}

\item Let
\(
\mathcal H_X
=
\mathcal H\bigl(
J_X,
(\Omega_{X,j})_{j=1}^{J_X},
(n_{X,j})_{j=1}^{J_X},
(s_{X,j})_{j=1}^{J_X}, \allowbreak
(p_{X,j})_{j=1}^{J_X},
(t_{X,j})_{j=1}^{J_X},
(q_{X,j})_{j=1}^{J_X},
(A_{X,j})_{j=1}^{J_X}
\bigr)
\)
and
\(
X
=
X\bigl(
J_X,
(\Omega_{X,j})_{j=1}^{J_X},
(t_{X,j})_{j=1}^{J_X},
(q_{X,j})_{j=1}^{J_X}
\bigr)
\)
satisfy Assumption~\ref{assumption:productSobolevSpace}.
Likewise, let
\(
\mathcal H_Y
=
\mathcal H\bigl(
J_Y,
(\Omega_{Y,k})_{k=1}^{J_Y}, \allowbreak
(n_{Y,k})_{k=1}^{J_Y},
(s_{Y,k})_{k=1}^{J_Y},
(p_{Y,k})_{k=1}^{J_Y},
(t_{Y,k})_{k=1}^{J_Y},
(q_{Y,k})_{k=1}^{J_Y}, \allowbreak
(A_{Y,k})_{k=1}^{J_Y}
\bigr)
\)
and
\(
Y
=
X\bigl(
J_Y,
(\Omega_{Y,k})_{k=1}^{J_Y},
(t_{Y,k})_{k=1}^{J_Y}, \allowbreak
(q_{Y,k})_{k=1}^{J_Y}
\bigr)
\)
satisfy Assumption~\ref{assumption:productSobolevSpace}.
\begin{itemize}
    \item There exist \(R>0\) and a nondecreasing function
\(
\omega:[0,\infty)\to[0,\infty)
\)
with
\(
\omega(0)=0
\)
such that
\(
B_R(\mathcal H_X)\subset \mathcal D(G)
\)
and, for every \(x,x'\in B_R(\mathcal H_X)\),
\[
\|G(x)-G(x')\|_Y
\le
\omega\left(\|x-x'\|_X\right).
\]
\item
The operator \(G\) maps \(B_R(\mathcal H_X)\) into \(\mathcal H_Y\), and there exists a constant \(M_R>0\) such that
\[
\sup_{x\in B_R(\mathcal H_X)}
\|G(x)\|_{\mathcal H_Y}
\le
M_R.
\]
\end{itemize}
\label{assumption:regularityG2}
\end{enumerate}

Assumption~\ref{assumption:regularityG} is an abstract continuity assumption on the operator
\(G:X\to Y\), expressed solely in terms of the natural spaces \(X\) and \(Y\).
It is sufficient to control the error introduced by reconstructing the input. Assumption~\ref{assumption:regularityG2} specializes to the Sobolev setting. Besides continuity in
the weaker \(Y\)-norm, it requires that \(G\) maps sufficiently regular inputs
into the stronger space \(\mathcal H_Y\) with uniformly bounded
\(\mathcal H_Y\)-norm on bounded subsets of \(\mathcal H_X\). This additional
regularity is needed to derive quantitative Sobolev reconstruction rates.

\end{assumptions}

\begin{assumptions}[Kernel learning in the encoder-decoder framework]
We make the following assumptions on our learning setup. 

\begin{enumerate}[label=\textbf{L.\arabic*}]
    \item Let
\(
\Gamma:Z_X\times Z_X\to\mathcal L(Z_Y)
\)
be an operator-valued kernel with reproducing kernel Hilbert space
\(
\mathcal H_\Gamma.
\)

Let
\(
x_1,\ldots,x_N\in X
\)
be training inputs and define measurements 
\(
U_i:=E_Xx_i\in Z_X.
\)
Let \(
S_i\in Z_Y,
\) for $i=1,\ldots,N$ be prescribed measured output data.  Define the observation operator
\[
L_{\Gamma,A}:\mathcal H_\Gamma\to Z_Y^N,
\qquad
L_{\Gamma,A}f
:=
(Af(U_1),\ldots,Af(U_N)),
\]
where
\(
A:=E_YD_Y:Z_Y\to Z_Y.
\) 
Write
\(
\mathbf U=(U_1,\ldots,U_N)
\), 
\(
\mathbf S=(S_1,\ldots,S_N),
\)
and define the block operator
\(
    \Gamma_A(\mathbf U,\mathbf U):Z_Y^N\to Z_Y^N
\)
by
\[
    \bigl(\Gamma_A(\mathbf U,\mathbf U)c\bigr)_i
    :=
    \sum_{j=1}^N
    A\Gamma(U_i,U_j)A^\ast c_j,
    \qquad c=(c_1,\ldots,c_N)\in Z_Y^N.
\]
For \(U\in Z_X\), define
\(
    \Gamma_A(U,\mathbf U):Z_Y^N\to Z_Y
\)
by
\[
    \Gamma_A(U,\mathbf U)c
    :=
    \sum_{j=1}^N
    \Gamma(U,U_j)A^\ast c_j.
\]
Whenever \(A=\Id_{Z_Y}\), we simply write
\(
\Gamma(\mathbf U,\mathbf U)
\)
and
\(
\Gamma(U,\mathbf U).
\) \label{assumption:kernel-learning}

\item 
Assume Assumption~\ref{assumption:kernel-learning}. In addition, let $S_i = E_Y G(x_i)$ and define $G_{\mathrm{enc}}(U_i):= E_YG(D_X U_i) \in Z_Y$ for $i=1,\ldots,N$. Let \(\lambda\ge 0\). If \(\lambda=0\), assume that
\(
    \Gamma(\mathbf U,\mathbf U):Z_Y^N\to Z_Y^N
\)
is boundedly invertible. 
Denote by $\widehat G_\lambda(z)$ and $\widehat G_{\mathrm{enc},\lambda}(z)$ the kernel interpolant (for $\lambda = 0$) or kernel ridge-regression estimator of the measured data $S_i$ and ideal encoded data $G_{\mathrm{enc}}(U_i)$, respectively. Let $\widehat G = \widehat G_\lambda$ and 
define
\(
        \overline G
        :=
        D_Y\circ \widehat G_\lambda \circ E_X
        :
        X\to Y.
\)
Define the residual smoother
\[
\mathcal R_{\mathbf U,\lambda} \eta: Z_X \to Z_Y, \qquad
    (\mathcal R_{\mathbf U,\lambda}\eta)(U)
    :=
    \Gamma(U,\mathbf U)
    \bigl(
        \Gamma(\mathbf U,\mathbf U)+\lambda \Id_{Z_Y^N}
    \bigr)^{-1}
    \eta, \qquad U \in Z_X
\]
where $\eta=(\eta_1,\ldots,\eta_N)\in Z_Y^N$ with $\eta_i
    :=
    S_i-G_{\mathrm{enc}}(U_i)$. \label{assumption:kernel-learning2}

\item  Let
\(
\mathcal H_X
=
\mathcal H\bigl(
J_X,
(\Omega_{X,j})_{j=1}^{J_X},
(n_{X,j})_{j=1}^{J_X},
(s_{X,j})_{j=1}^{J_X}, \allowbreak
(p_{X,j})_{j=1}^{J_X},
(t_{X,j})_{j=1}^{J_X},
(q_{X,j})_{j=1}^{J_X},
(A_{X,j})_{j=1}^{J_X}
\bigr)
\)
and
\(
X
=
X\bigl(
J_X,
(\Omega_{X,j})_{j=1}^{J_X},
(t_{X,j})_{j=1}^{J_X},
(q_{X,j})_{j=1}^{J_X}
\bigr)
\)
satisfy Assumption~\ref{assumption:productSobolevSpace}.
Likewise, let
\(
\mathcal H_Y
=
\mathcal H\bigl(
J_Y,
(\Omega_{Y,k})_{k=1}^{J_Y}, \allowbreak
(n_{Y,k})_{k=1}^{J_Y},
(s_{Y,k})_{k=1}^{J_Y},
(p_{Y,k})_{k=1}^{J_Y},
(t_{Y,k})_{k=1}^{J_Y},
(q_{Y,k})_{k=1}^{J_Y}, \allowbreak
(A_{Y,k})_{k=1}^{J_Y}
\bigr)
\)
and
\(
Y
=
X\bigl(
J_Y,
(\Omega_{Y,k})_{k=1}^{J_Y},
(t_{Y,k})_{k=1}^{J_Y}, \allowbreak
(q_{Y,k})_{k=1}^{J_Y}
\bigr)
\)
satisfy Assumption~\ref{assumption:productSobolevSpace}. Let $G$ satisfy Assumption \ref{assumption:regularityG2}. Assume Assumption \ref{assumption:kernel-learning2}. In addition, suppose that
\(
Z_{\mathcal H_X} \subseteq \mathbb R^{d_X}
\)
and 
\(
Z_{\mathcal H_Y}\subset \mathbb R^{d_Y}.
\)
Let
\(
\Upsilon\subset Z_{\mathcal H_X}
\)
be a bounded domain with Lipschitz boundary such that
\(
E_{\mathcal H_X}\bigl(B_R(\mathcal H_X)\bigr)
\subset \Upsilon.
\)
Define the encoded operator by
\(
G_{\mathrm{enc}}
:=
E_{\mathcal H_Y}
\circ G
\circ D_{\mathcal H_X}
:
\Upsilon\to Z_{\mathcal H_Y}.
\)
Assume that the operator-valued kernel is defined on
\(
\Gamma:
\Upsilon\times\Upsilon
\to
\mathcal L(\mathbb R^{d_Y}),
\)
and that its associated RKHS satisfies, for some
\(
\tau>d_X/2,
\)
we have
\(
\mathcal H_\Gamma
\hookrightarrow
\Hk{\tau}(\Upsilon;\mathbb R^{d_Y})
\)
continuously. Assume furthermore that
\(
G_{\mathrm{enc}}\in\mathcal H_\Gamma.
\)

Let the training inputs satisfy
\(
x_1,\ldots,x_N\in B_R(\mathcal H_X),
\)
so that
\[
U_i
=
E_{\mathcal H_X}x_i
\in\Upsilon,
\qquad
S_i
=
E_{\mathcal H_Y}G(x_i),
\qquad
i=1,\ldots,N.
\]
Define the fill distance of the encoded training inputs in \(\Upsilon\) by
\(
h_{\mathrm{tr}}
:=
\sup_{U\in\Upsilon}
\min_{1\le i\le N}
\|U-U_i\|_2.
\)
\label{assumption:sobolev-learning}

\end{enumerate}

Assumptions \ref{assumption:kernel-learning}-\ref{assumption:sobolev-learning} introduce increasingly specialized learning settings. Assumption \ref{assumption:kernel-learning} specifies the abstract kernel learning framework on the encoded spaces, including the training data and kernel notation, without making any assumptions on the origin of the measurements. Assumption \ref{assumption:kernel-learning2} specializes this framework to the encoder–decoder setting by assuming that the measurements arise from an underlying operator G, thereby introducing the encoded operator, its kernel approximation, and the associated reconstruction operator. Finally, Assumption \ref{assumption:sobolev-learning} specializes further to the Sobolev setting, where the encoder and decoder are constructed from product Sobolev spaces, the encoded domain is a bounded Lipschitz subset of a Euclidean space, and additional Sobolev regularity assumptions are imposed on both the kernel RKHS and the encoded operator. These assumptions enable the quantitative approximation estimates derived in the subsequent sections.

\end{assumptions}

\begin{remark}[Regularity of the encoded operator]
\label{rem:rkhs-regularity}
The assumption that the encoded operator \(G_{\mathrm{enc}}\) belongs to the
RKHS \(\mathcal H_\Gamma\) in Assumption \ref{assumption:productSobolevSpace} can be verified from regularity assumptions on
the original operator \(G\). Indeed, if the encoders and decoders are
bounded linear operators (for example, the decoders are the minimum-norm interpolants from Theorem \ref{thm:minimumNorm}), Fréchet differentiability is preserved under
composition. More precisely, if
\(
G\in C^k(X,Y),
\)
then
\(
G_{\mathrm{enc}}=E_Y\circ G\circ D_X
\)
also belongs to \(C^k(Z_X,Z_Y)\), with
\[
\|D^kG_{\mathrm{enc}}(U)\|
\le
\|E_Y\|\,\|D_X\|^k
\|D^kG(D_XU)\|
\]
by \cite[Lemma 3.3]{BATLLE2024112549}.
If the chosen RKHS \(\mathcal H_\Gamma\) contains sufficiently
smooth functions (for example, through a continuous embedding of an appropriate
Sobolev space into \(\mathcal H_\Gamma\)), then the assumption
\(G_{\mathrm{enc}}\in\mathcal H_\Gamma\) follows directly from corresponding
regularity assumptions on \(G\).
\end{remark}

\subsection{General Learning Framework} \label{sec:general}

The diagrams in Figure~\ref{fig:general} summarize the abstract
encoder--decoder formulation used throughout this section. The proofs of the results in this section can be found in Appendix \ref{sec:proofs:general}.

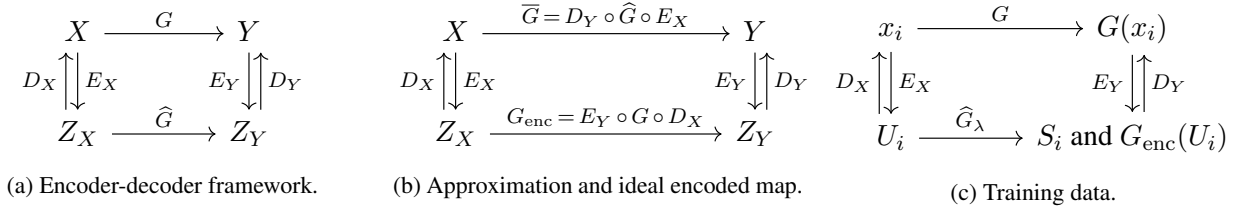
\begin{figure}[H]
\centering
\begin{subfigure}[t]{0.32\linewidth}
\centering
\begin{tikzcd}[column sep=3.6em, row sep=2.0em]
X \arrow[r, "G"] \arrow[d, "E_X" right]
  & Y \arrow[d, "E_Y" left] \\
Z_X \arrow[r, "\widehat{G}"]
  & Z_Y
    \arrow[u, "D_Y" right, shift right=0.9ex]
\arrow[from=2-1, to=1-1, "D_X" left, shift left=0.9ex]
\end{tikzcd}
\caption{Encoder-decoder framework.}
\label{fig:cd-general-form}
\end{subfigure}
\hfill
\begin{subfigure}[t]{0.32\linewidth}
\centering
\begin{tikzcd}[column sep=8em, row sep=2.0em]
X \arrow[r, "\overline G \,= \, D_Y \, \circ \, \widehat G \, \circ \, E_X "] \arrow[d, "E_X" right]
  & Y \arrow[d, "E_Y" left] \\
Z_X \arrow[r, "G_{\mathrm{enc}}
    \, = \, 
    E_Y\, \circ \, G \, \circ \, D_X"]
  & Z_Y
    \arrow[u, "D_Y" right, shift right=0.9ex]
\arrow[from=2-1, to=1-1, "D_X" left, shift left=0.9ex]
\end{tikzcd}
\caption{Approximation and ideal encoded map.}
\label{fig:approximation}
\end{subfigure}
\hfill
\begin{subfigure}[t]{0.32\linewidth}
\centering
\begin{tikzcd}[column sep=3.6em, row sep=2.0em]
x_i \arrow[r, "G"] \arrow[d, "E_{X}" right]
  & G(x_i) \arrow[d, "E_Y" left] \\
U_i \arrow[r, "\widehat G_\lambda"]
  & S_i \text{ and } G_{\mathrm{enc}}(U_i)
    \arrow[u, "D_Y" right, shift right=0.9ex]
\arrow[from=2-1, to=1-1, "D_{X}" left, shift left=0.9ex]
\end{tikzcd}
\caption{Training data.}
\label{fig:training}
\end{subfigure}
\caption{
Schematic description of the proposed learning framework.
(a) Encoder--decoder formulation for learning a map
$G:X\to Y$. The maps $E_X:X\to Z_X$ and $E_Y:Y\to Z_Y$
encode inputs and outputs into measurement spaces, while
$D_X:Z_X\to X$ and $D_Y:Z_Y\to Y$
reconstruct representatives in the original spaces.
The surrogate $\widehat G:Z_X\to Z_Y$ is learned between the
measurement spaces.
(b) The learned surrogate induces the approximation
$\overline G =D_Y\circ\widehat G\circ E_X$
on the original spaces, while the target map induces the
encoded-space map
$G_{\mathrm{enc}}:=E_Y\circ G\circ D_X$.
Theorem \ref{thm:encoder-decoder-learning-error} bounds the error between $\overline G$ and $G$.
(c) Location of the training data in the encoder--decoder framework.
The training inputs \(x_i\in X\) are observed through the measurements
\(U_i=E_Xx_i\in Z_X\), while the measured outputs are
\(S_i \in Z_Y\).
Learning is performed entirely on the measurement spaces
\(Z_X\) and \(Z_Y\) through a kernel learner $\widehat G_\lambda$ (see Corollary \ref{cor:induced-measurement-space-equivalence}).
The diagram also highlights the encoded target values
\(G_{\mathrm{enc}}(U_i)=E_YG(D_XU_i)\).
When the training outputs are chosen as \(S_i=E_YG(x_i)\), one error term in Theorem~\ref{thm:encoder-decoder-learning-error} depends on the discrepancy
\(\eta_i:=S_i-G_{\mathrm{enc}}(U_i)\),
which measures the mismatch between the measured data and the ideal encoded target.
}
\label{fig:general}
\end{figure}

Our goal is to approximate a target map
\(
    G:X\to Y
\)
between Hilbert spaces from measured data. In many settings, however, the data
do not provide direct access to elements of \(X\) and \(Y\). Instead, inputs and
outputs are observed only through bounded linear measurement, or encoding, maps
\(
    E_X:X\to Z_X
\)
and
\(
    E_Y:Y\to Z_Y,
\)
where \(Z_X\) and \(Z_Y\) are Hilbert measurement spaces. We therefore pose the learning problem on the latter spaces: given measured input--output pairs
\(
    (E_X(x_i), E_Y(G(x_i)))\in Z_X\times Z_Y,
\)
one first constructs a surrogate
\[
    \widehat G:Z_X\to Z_Y.
\]
To obtain an approximation of the original map \(G:X\to Y\), we then decode the predicted output measurement and define
\begin{equation} \label{eq:main:surrogate}
    \overline G
    :=
    D_Y\circ \widehat G\circ E_X
    :
    X\to Y,
\end{equation}
where \(D_Y:Z_Y\to Y\) is an output reconstruction map. Thus, the method learns where the data are observed and then lifts the learned prediction back to the original output space. This formulation is particularly useful when \(X\) and \(Y\) are infinite-dimensional function spaces, while \(Z_X\) and \(Z_Y\) are finite-dimensional, lower-dimensional, or otherwise computationally tractable representations.

If the target map \(G\) were known, the natural map to use between the
measurement spaces would be
\[
    G_{\mathrm{enc}}
    :=
    E_Y\circ G\circ D_X
    :
    Z_X\to Z_Y,
\]
where \(D_X:Z_X\to X\) is a chosen input reconstruction map. Indeed, if the decoding maps are chosen so that
\(
    D_X \circ E_X\approx \Id_X \)
and
\(
    D_Y \circ E_Y\approx \Id_Y,
\)
then using \(G_{\mathrm{enc}}\) in place of \(\widehat G\) in \eqref{eq:main:surrogate} gives
\[
    D_Y\circ G_{\mathrm{enc}}\circ E_X
    =
    D_Y\circ E_Y\circ G\circ D_X\circ E_X
    \approx
    G.
\]
This observation separates the encoder--decoder construction in
Figure~\ref{fig:cd-general-form} into two coupled tasks:
\begin{itemize}
    \item Reconstruction: choose decoding maps \(D_X\) and \(D_Y\) that lift
    measurements back to meaningful representatives in the original spaces;
    \item Learning: approximate the unknown measurement-space map
    \(
        G_{\mathrm{enc}}:Z_X\to Z_Y
    \)
    from measured training pairs by constructing a surrogate
    \(
        \widehat G:Z_X\to Z_Y.
    \)
\end{itemize}

We first address the reconstruction task. The given maps $E_X$ and $E_Y$ typically aim to produce simpler or more tractable
representations of objects in $X$ and $Y$, and are therefore not expected to be injective. A single measurement may consequently correspond to many elements of the original space, so the decoding maps \(D_X\) and \(D_Y\) cannot be obtained by ordinary inversion.  To resolve this ambiguity, we choose a canonical decoder by minimum-norm recovery:
\[
    D_X(U)
    :=
    \argmin_{x\in X}\{\|x\|_X : E_Xx=U\},
    \qquad
    D_Y(S)
    :=
    \argmin_{y\in Y}\{\|y\|_Y : E_Yy=S\}.
\]
This selects, among all elements compatible with the measurements, the one of
smallest Hilbert norm, where the norm encodes the chosen notion of complexity on
the underlying space. 

In some situations however, exact consistency with the measurements may be undesirable
or ill-conditioned, for instance when the measurements are noisy or when the
operators \(E_X \circ E_X^\ast\) and/or \(E_Y\circ E_Y^\ast\) are poorly conditioned. One may then
replace the hard constraints by a regularized recovery problem. For
\(\gamma>0\), this gives
\[
    D_{X,\gamma}(U)
    :=
    \argmin_{x\in X}
    \left\{
        \|x\|_X^2+\gamma^{-1}\|E_Xx-U\|_{Z_X}^2
    \right\},
\]
and
\[
    D_{Y,\gamma}(S)
    :=
    \argmin_{y\in Y}
    \left\{
        \|y\|_Y^2+\gamma^{-1}\|E_Yy-S\|_{Z_Y}^2
    \right\}.
\]
A major advantage of both the constrained and regularized formulations is that they
lead to explicit reconstruction formulas given in Theorem \ref{thm:minimumNorm}.

We now turn to the learning task. The choice of learning method depends on
the structure of the measurement spaces \(Z_X\) and \(Z_Y\), and involves a
trade-off between expressivity, numerical complexity, and theoretical
tractability. In this way, richer model classes may capture more complex maps between
measurements, while more structured methods often lead to explicit formulas and
sharper analysis. We focus on kernel methods for the learning step. This choice is
motivated by the analytical and computational structure they provide:
minimum-norm interpolation, regularized regression, deterministic error bounds, and kernel-based uncertainty quantification can all be written in closed form. In addition, kernel methods interact naturally with the
encoder--decoder construction. As our first result shows, a kernel on the
measurement spaces induces a corresponding kernel on the original spaces.
Specifically, any operator-valued kernel
\[
    \Gamma: Z_X\times Z_X\to \cL(Z_Y)
\]
induces an operator-valued kernel
\[
    K:X\times X\to \cL(Y).
\]

\begin{theorem}[Encoder--decoder induced operator-valued kernel]
\label{thm:induced-kernel-original-spaces}
Assume that
\(
    D_Y:Z_Y\to Y
\)
is a bounded linear output decoder. Let
\(
    \Gamma:Z_X\times Z_X\to \mathcal L(Z_Y)
\)
be an operator-valued kernel with associated RKHS
\(\mathcal H_\Gamma\) of maps \(Z_X\to Z_Y\). Define
\(
    K:X\times X\to\mathcal L(Y)
\)
by
\[
    K(x,x')
    :=
    D_Y\,\Gamma(E_Xx,E_Xx')\,D_Y^\ast.
\]
Then \(K\) is an operator-valued kernel on \(X\) with values in
\(\mathcal L(Y)\). Moreover, the associated RKHS is
\[
    \mathcal H_K
    =
    \{D_Y\circ f\circ E_X : f\in\mathcal H_\Gamma\},
\]
equipped with the minimal-representative norm
\[
    \|F\|_{\mathcal H_K}
    =
    \inf\left\{
        \|f\|_{\mathcal H_\Gamma}
        :
        F=D_Y\circ f\circ E_X
    \right\}.
\]

Suppose, in addition, that \(E_X\) and \(E_Y\) are surjective and that the
decoders \(D_X:Z_X\to X\) and \(D_Y:Z_Y\to Y\) are given by the exact
minimum-norm recovery maps from Theorem \ref{thm:minimumNorm}. Then, the representation
\[
    F=D_Y\circ f\circ E_X
\]
is unique and the minimal-representative norm reduces to the identity
\[
    \|D_Y\circ f\circ E_X\|_{\mathcal H_K}
    =
    \|f\|_{\mathcal H_\Gamma}.
\]
\end{theorem}

As a consequence of Theorem \ref{thm:induced-kernel-original-spaces}, when the
surrogate \(\widehat G\) is constructed by kernel interpolation or regression on
the measurement spaces, the full approximation
\[
    \overline G
    =
    D_Y\circ \widehat G\circ E_X
    :
    X\to Y
\]
can itself be interpreted as a kernel method acting directly from \(X\) to
\(Y\). In other words, although the learning problem is posed on the
measurement spaces \(Z_X\) and \(Z_Y\), the resulting reconstructed operator
belongs to an operator-valued RKHS of maps from \(X\) to \(Y\). 

We now detail the equivalence of learning frameworks.

\begin{corollary}[Equivalence of induced and measurement-space learning problems]
\label{cor:induced-measurement-space-equivalence}
Assume the setting of Theorem~\ref{thm:induced-kernel-original-spaces} and that Assumption \ref{assumption:kernel-learning} is satisfied.
For \(\lambda\ge 0\), consider the induced RKHS problem
\[
    \min_{F\in\mathcal H_K}
    \left\{
        \lambda^{-1} \sum_{i=1}^N
        \|E_YF(x_i)-S_i\|_{Z_Y}^2
        +
        \|F\|_{\mathcal H_K}^2
    \right\}.
\]
For \(\lambda=0\), this is understood as the minimum-norm interpolation problem
\[
    \min_{F\in\mathcal H_K}
    \|F\|_{\mathcal H_K}^2
    \quad
    \text{subject to}
    \quad
    E_YF(x_i)=S_i,
    \qquad i=1,\ldots,N.
\]
This problem is equivalent to the encoded measurement-space problem
\[
    \min_{f\in\mathcal H_\Gamma}
    \left\{
        \lambda^{-1} \sum_{i=1}^N
        \|Af(U_i)-S_i\|_{Z_Y}^2
        +
        \|f\|_{\mathcal H_\Gamma}^2
    \right\}.
\]
Again, for \(\lambda=0\), this is understood as the minimum-norm interpolation
problem
\[
    \min_{f\in\mathcal H_\Gamma}
    \|f\|_{\mathcal H_\Gamma}^2
    \quad
    \text{subject to}
    \quad
    Af(U_i)=S_i,
    \qquad i=1,\ldots,N.
\]

If \(\lambda=0\), further assume that \(L_{\Gamma,A}\) is surjective, equivalently that
\(
    \Gamma_A(\mathbf U,\mathbf U)
    =
    L_{\Gamma,A}L_{\Gamma,A}^\ast
\)
is boundedly invertible on \(Z_Y^N\). Then, the encoded measurement-space solution and the
corresponding induced RKHS solution are
\[
    \overline f_\lambda(U)
    =
    \Gamma_A(U,\mathbf U)
    \bigl(\Gamma_A(\mathbf U,\mathbf U)+\lambda \Id_{Z_Y^N}\bigr)^{-1}
    \mathbf S \quad \text{and} \quad \overline F_\lambda(x)
    =
    D_Y\,
    \Gamma_A(E_Xx,\mathbf U)
    \bigl(\Gamma_A(\mathbf U,\mathbf U)+\lambda \Id_{Z_Y^N}\bigr)^{-1}
    \mathbf S.
\]

Suppose now, in addition, that \(E_X\), \(E_Y\) are surjective and that
\(D_X\), \(D_Y\) are the exact minimum-norm recovery maps from
Theorem~\ref{thm:minimumNorm}. Then, the encoded measurement-space problems reduce to the standard
measurement-space problems
\[
    \min_{f\in\mathcal H_\Gamma}
    \|f\|_{\mathcal H_\Gamma}^2
    \quad
    \text{subject to}
    \quad
    f(U_i)=S_i,
    \qquad i=1,\ldots,N,
\]
and
\[
    \min_{f\in\mathcal H_\Gamma}
    \left\{
        \lambda^{-1} \sum_{i=1}^N
        \|f(U_i)-S_i\|_{Z_Y}^2
        +
        \|f\|_{\mathcal H_\Gamma}^2
    \right\}.
\]
If $\lambda=0$, further assume that \(\Gamma(\mathbf U,\mathbf U)\) is boundedly invertible. Then, the encoded measurement-space solution and the
corresponding induced RKHS solution are
\[
    \overline f_\lambda(U)
    =
    \Gamma(U,\mathbf U)
    \bigl(\Gamma(\mathbf U,\mathbf U)+\lambda \Id_{Z_Y^N}\bigr)^{-1}
    \mathbf S, \quad \overline F_\lambda(x)
    =
    E_Y^\ast(E_YE_Y^\ast)^{-1}
    \Gamma(E_Xx,\mathbf U)
    \bigl(\Gamma(\mathbf U,\mathbf U)+\lambda \Id_{Z_Y^N}\bigr)^{-1}
    \mathbf S.
\]
\end{corollary}

\begin{remark}[Pointwise learning in the encoded space]
\label{rem:pointwise}
While the encoders \(E_X\) and \(E_Y\) may be constructed from arbitrary bounded linear measurements, the induced learning problem is simply a pointwise regression problem on the encoded spaces \(Z_X\) and \(Z_Y\), where the training data consist of the
encoded pairs
\(
(U_i,S_i).
\)
Consequently, the kernel learning theory depends only
on the geometry of the encoded spaces and is independent of the particular
measurement modality used to construct \(E_X\) and \(E_Y\).
\end{remark}

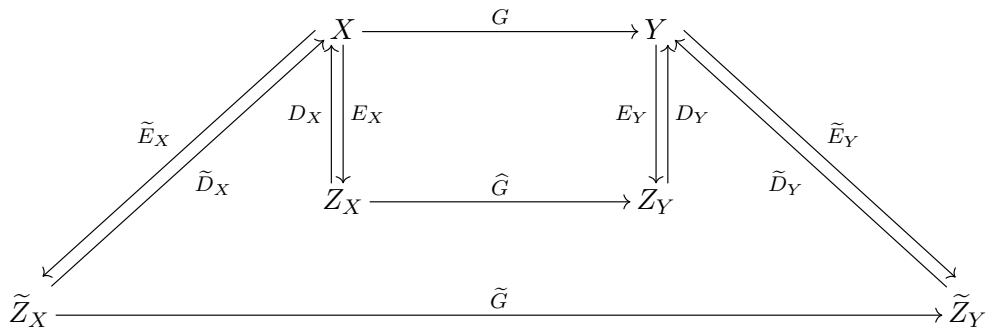
\begin{figure}[b]
\centering
\begin{tikzcd}[
    column sep=4.5em,
    row sep=2.4em,
    cells={nodes={inner sep=2pt}}
]
&& &
X
    \arrow[rr, "G", black]
    \arrow[dd, "E_X" right]
    \arrow[dddll, "\widetilde E_X"', shift right=1.5ex]
&
&
Y
    \arrow[dd, "E_Y" left]
    \arrow[dddrr, "\widetilde E_Y", shift right = -1.5ex]
&&
\\
&&&&&&&& \\
&&&
Z_X
    \arrow[rr, "\widehat G"]
    \arrow[uu, "D_X" left, shift left=0.9ex,]
&
&
Z_Y
    \arrow[uu, "D_Y" right, shift right=0.9ex,]
&&
\\
&
\widetilde Z_X
    \arrow[uuurr, "\widetilde D_X"', shift left=0.5ex,]
    \arrow[rrrrrr, "\widetilde G"]
&
&
&
&
&
&
\widetilde Z_Y
    \arrow[uuull, "\widetilde D_Y", shift right=0.5ex]
\end{tikzcd}

\caption{Measurement transferability of the encoder-decoder learning framework. The maps \(E_X:X\to Z_X\) and \(E_Y:Y\to Z_Y\) encode inputs
and outputs into measurement spaces, while \(D_X:Z_X\to X\) and
\(D_Y:Z_Y\to Y\) reconstruct representatives in the original spaces.
The surrogate \(\widehat G:Z_X\to Z_Y\) is learned between the measurement
spaces. The additional measurement and reconstruction maps induce an alternative surrogate
\(
    \widetilde G := \widetilde E_Y \circ D_Y \circ \widehat G \circ E_X \circ \widetilde D_X :    \widetilde Z_X\to \widetilde Z_Y
\)
and approximation $\widetilde D_Y \circ \widetilde G \circ \widetilde E_X$ of $G$.
}
\label{fig:cd-general-form-extended}
\end{figure}

\begin{remark}[Measurement transferability]
The encoder--decoder construction in
Figure~\ref{fig:cd-general-form} naturally induces a notion of
measurement transferability, as illustrated in
Figure~\ref{fig:cd-general-form-extended} and similarly discussed in \cite[Section~2.4]{BATLLE2024112549}. Consider an alternative pair of measurement spaces
\(\widetilde Z_X\) and \(\widetilde Z_Y\), together with measurement and
reconstruction maps
\[
    \widetilde E_X:X\to\widetilde Z_X,
    \qquad
    \widetilde D_X:\widetilde Z_X\to X,
\]
and
\[
    \widetilde E_Y:Y\to\widetilde Z_Y,
    \qquad
    \widetilde D_Y:\widetilde Z_Y\to Y.
\]
The previously learned surrogate \(\widehat G\) then induces an alternative
measurement-space surrogate
\[
    \widetilde G
    :=
    \widetilde E_Y
    \circ D_Y
    \circ \widehat G
    \circ E_X
    \circ \widetilde D_X
    :
    \widetilde Z_X\to\widetilde Z_Y.
\]
Thus, data represented in \(\widetilde Z_X\) can first be reconstructed in
\(X\), encoded into the original learning space \(Z_X\), propagated through
\(\widehat G\), reconstructed in \(Y\), and finally remeasured in
\(\widetilde Z_Y\). The corresponding approximation of the original map is
\[
    \widetilde D_Y
    \circ \widetilde G
    \circ \widetilde E_X
    :
    X\to Y.
\]
In this sense, the learned surrogate is transferable across different choices
of input and output measurements: it can be transported to alternative
measurement systems through the associated reconstruction and remeasurement
maps, without retraining \(\widehat G\).

A particularly important case arises when \(X\) and \(Y\) are spaces of
continuous functions and
\[
    E_X:X\to\mathbb R^{c_W},
    \qquad
    E_Y:Y\to\mathbb R^{c_Y}
\]
are point-evaluation maps on prescribed input and output grids. Alternative
maps \(\widetilde E_X\) and \(\widetilde E_Y\) then correspond to different
sets of evaluation points. In this setting, measurement transferability becomes
mesh transferability: a surrogate trained using one pair of grids can be
transferred to another pair of grids through reconstruction and resampling. This mesh transferability is a defining feature of (multiple) operator learning. Indeed, methods such as DeepONet \cite{deepOnet}, FNO \cite{li2021fourier}, MNO \cite{weihs2025MNO,weihs2026generalizationboundsstatisticalguarantees,weihs2026multipleneuraloperatorsachieve} and related architectures are designed to learn mappings between function spaces rather than between fixed finite-dimensional vectors, allowing predictions on discretizations different from those used during training. Remark~\ref{rem:mesh-free-alternative}
discusses an alternative approach in which the learned representation is
directly evaluable independently of a prescribed output mesh.
\end{remark}

The next result separates the total
error into an encoder--decoder error and a learning error. The first term
measures how well the measurement and reconstruction maps preserve the action of
the continuum map \(G\). The second term measures how well the surrogate
\(\widehat G\) approximates the encoded-space target
\(
    G_{\mathrm{enc}}.
\)
When \(\widehat G\) is chosen by kernel interpolation on the measurement space,
the learning error can be further bounded by a kernel residual term, together
with a data-consistency term accounting for the possible mismatch between the
measured training outputs \(E_YG(x_i)\) and the encoded target values
\(G_{\mathrm{enc}}(E_Xx_i)\) (see Figure \ref{fig:training}).

\begin{theorem}[Encoder--decoder and learning error decomposition]
\label{thm:encoder-decoder-learning-error} We have the following error decompositions. 
\begin{enumerate}
    \item Let \(
        \widehat G:Z_X\to Z_Y
    \)
be any surrogate, and define
\(
        \overline G
        :=
        D_Y\circ \widehat G\circ E_X
        :
        X\to Y.
\)
    Then, for every \(x\in X\),
    \[
    \begin{aligned}
        \|G(x)-\overline G(x)\|_Y
        &\le
        \|G(x)-D_YE_YG(D_XE_Xx)\|_Y  \\
        &\quad
        +
        \|D_Y(G_{\mathrm{enc}}-\widehat G)(E_Xx)\|_Y.
    \end{aligned}
    \]
\item Assume that Assumption \ref{assumption:kernel-learning2} is satisfied.

\begin{enumerate}
    \item For every \(x\in X\),
\[
\begin{aligned}
\|G(x)-\overline{G}(x)\|_Y
    &\le
    \|G(x)-D_YE_YG(D_XE_Xx)\|_Y    +
    \|D_Y\|_{\mathrm{op}}\,
    \|(G_{\mathrm{enc}}-\widehat G_{\mathrm{enc},\lambda})(E_Xx)\|_{Z_Y} \\
    &+
    \|D_Y\|_{\mathrm{op}}\,
    \|\mathcal R_{\mathbf U,\lambda}\eta(E_Xx)\|_{Z_Y}.
\end{aligned}
\]
\item Define the regularized residual kernel
\[
    \Gamma_{\lambda}^{\perp}(U,U')
    :=
    \Gamma(U,U')
    -
    \Gamma(U,\mathbf U)
    \bigl(
        \Gamma(\mathbf U,\mathbf U)+\lambda \Id_{Z_Y^N}
    \bigr)^{-1}
    \Gamma(\mathbf U,U'),
    \qquad U,U'\in Z_X
\]
and furthermore assume that, for every \(U\in Z_X\),
\(
    \operatorname{Tr}\bigl[\Gamma_{\lambda}^{\perp}(U,U)\bigr]<\infty.
\)
With
\[
    Q_{N,\lambda}(U)
    := \left( \operatorname{Tr}\bigl[\Gamma_{\lambda}^{\perp}(U,U)\bigr]
        +
        \lambda \operatorname{Tr}(\Id_{Z_Y})
    \right)^{1/2},  
\]
for every \(x\in X\),
\[
\begin{aligned}
\|G(x)-\overline{G}(x)\|_Y
    &\le
    \|G(x)-D_YE_YG(D_XE_Xx)\|_Y +
    \|D_Y\|_{\mathrm{op}}\,
    Q_{N,\lambda}(E_Xx)\,
    \|G_{\mathrm{enc}}\|_{\mathcal H_{\Gamma,\lambda}} \\
    &
    +
    \|D_Y\|_{\mathrm{op}}\,
    \|\mathcal R_{\mathbf U,\lambda}\eta(E_Xx)\|_{Z_Y}
\end{aligned}
\]
where $\cH_{\Gamma,\lambda}$ is the RKHS with kernel $\Gamma(U,U') + \lambda \Id_{Z_Y}\delta_{U,U'}$.  
\end{enumerate}

\end{enumerate}

\end{theorem}

The goal of Section \ref{sec:mainMulti} is to progressively sharpen the error decomposition of
Theorem~\ref{thm:encoder-decoder-learning-error} by imposing additional
assumptions on the operator \(G\), the spaces \(X\) and \(Y\), the
encoder--decoder pairs \((E_X,D_X)\) and \((E_Y,D_Y)\), and the
available data.

\subsection{Kernel Learning for Multi-Input, Multi-Output Operators} \label{sec:mainMulti}

In this section, we focus on the setting of multi-input, multi-output operator learning, where the underlying function spaces are Sobolev spaces, the encoders are given by pointwise sampling operators, and the decoders reconstruct functions from these pointwise observations. This setting encompasses many problems arising from partial differential equations and allows us to derive explicit error bounds. The proofs of the results in this section can be found in Appendix \ref{sec:proofs:multi}.

We begin by examining the data-consistency term in Theorem \ref{thm:encoder-decoder-learning-error}.

\begin{remark}[Consistent measured training data]
If the measured training outputs are consistent with the encoded-space target,
namely
\[
    E_YG(x_i)
    =
    G_{\mathrm{enc}}(E_Xx_i)
    =
    E_YG(D_XE_Xx_i),
    \qquad i=1,\ldots,N,
\]
then
\(
    \eta_i
    =
    S_i-G_{\mathrm{enc}}(U_i)
    =
    0.
\)
Consequently,
\(
    \mathcal R_{\mathbf U,\lambda}\eta=0
\)
and the data-consistency term in
Theorem~\ref{thm:encoder-decoder-learning-error} vanishes: for every
\(x\in X\),
\[
    \|G(x)-\overline{G}(x)\|_Y
    \le
    \|G(x)-D_YE_YG(D_XE_Xx)\|_Y
    +
    \|D_Y\|_{\mathrm{op}}\,
    \|(G_{\mathrm{enc}}-\widehat G_{\mathrm{enc},\lambda})(E_Xx)\|_{Z_Y}.
\]
A sufficient condition for this consistency is exact input reconstruction on the
training data, that is
\[
    D_XE_Xx_i=x_i,
    \qquad i=1,\ldots,N.
\]
\end{remark}

Next, we analyze the reconstruction error of the map $G$.

\begin{proposition}[Encoder--decoder reconstruction error]
\label{prop:encoder-decoder-reconstruction-error}
Suppose that there exist functions
\(
    \delta_X:X\to[0,\infty)
\)
and
\(
    \delta_Y:Y\to[0,\infty),
\)
such that \begin{equation} \label{eq:prop:reconstructionError}
    \begin{cases}
        \|x-D_XE_Xx\|_X
    \le
    \delta_X(x), & x\in X,\\
    \|y-D_YE_Yy\|_Y
    \le
    \delta_Y(y), &y\in Y.
    \end{cases}
\end{equation}
Assume that $G$ satisfies Assumption \ref{assumption:regularityG} and that Assumption \ref{assumptions:M2} is satisfied.

Then, for every \(x\in B_R(X)\),
\[
\|G(x)-D_YE_YG(D_XE_Xx)\|_Y
\le
\omega\left(\delta_X(x)\right)
+
\delta_Y\left(G(D_XE_Xx)\right).
\]
If, in addition, \(X=\prod_{j=1}^{J_X}X_j,\) and $Y=\prod_{j=1}^{J_Y}Y_j,$
with product norms \(\|x\|_X^2
    =
    \sum_{j=1}^{J_X}\|x_j\|_{X_j}^2,\) and \(\|y\|_Y^2
    =
    \sum_{j=1}^{J_Y}\|y_j\|_{Y_j}^2,\)
and there exist componentwise reconstruction bounds
\begin{equation*}
    \begin{cases}
        \|x_j-(D_XE_Xx)_j\|_{X_j}
    \le
    \delta_{X,j}(x), & x\in X, \, j=1,\ldots,J_X,\\
    \|y_j-(D_YE_Yy)_j\|_{Y_j}
    \le
    \delta_{Y,j}(y), &y\in Y, \, j=1,\ldots,J_Y,
    \end{cases}
\end{equation*}
then
\[
\begin{aligned}
\|G(x)-D_YE_YG(D_XE_Xx)\|_Y
&\le
\omega\left(
    \left[
        \sum_{j=1}^{J_X}
        \delta_{X,j}(x)^2
    \right]^{1/2}
\right) +
\left[
    \sum_{j=1}^{J_Y}
    \delta_{Y,j}\left(
        G(D_XE_Xx)
    \right)^2
\right]^{1/2}.
\end{aligned}
\]
\end{proposition}

\begin{remark}[Componentwise and coupled measurements]
\label{rem:componentwise-coupled-measurements}

Proposition~\ref{prop:encoder-decoder-reconstruction-error} does not
require the encoder or decoder to act componentwise. While independent
measurements of each component are the most common situation in operator
learning, more general measurement procedures naturally arise in several
applications \cite{Robey,vanLeeuwen}.
\end{remark}

When one considers products of Sobolev space, we can explicitly bound the reconstruction errors in terms of the fill distance as the next result shows.

\begin{lemma}[Recovery estimate for products of Sobolev-embedded Hilbert spaces]
\label{lem:sobolev}

Let
\(
\mathcal H
=
\mathcal H\bigl(
J,
(\Omega_j)_{j=1}^J, \allowbreak
(n_j)_{j=1}^J,
(s_j)_{j=1}^J,
(p_j)_{j=1}^J,
(t_j)_{j=1}^J,
(q_j)_{j=1}^J,
(A_j)_{j=1}^J
\bigr)
\)
and
\(
X
=
X\bigl(
J,
(\Omega_j)_{j=1}^J,
(t_j)_{j=1}^J,
(q_j)_{j=1}^J
\bigr)
\)
satisfy Assumption~\ref{assumption:productSobolevSpace}, and let
\(
(E_{\mathcal H},D_{\mathcal H},Z_{\mathcal H})
\)
be the associated encoding/decoding maps and measurement space from
Assumption~\ref{assumptions:M1}.

Then, there exist constants \(h_j^0>0\) and \(C_j>0\), independent of
\(h_j\) and \(u_j\), such that, whenever
\(
h_j\le h_j^0,
\)
one has
\[
\|u_j-D_jE_ju_j\|_{\Wkp{t_j}{q_j}(\Omega_j)}
\le
C_jh_j^{s_j-t_j
-
n_j
\left(
\frac{1}{p_j}-\frac{1}{q_j}
\right)_+}
\|u_j\|_{\mathcal H_j}
\]
for every \(u_j\in\mathcal H_j\). Consequently, for every \(u\in\mathcal H\),
\[
\|u-D_{\mathcal H}E_{\mathcal H}u\|_X
\le
\left(
\sum_{j=1}^J
C_j^2
h_j^{2\l s_j-t_j
-
n_j
\left(
\frac{1}{p_j}-\frac{1}{q_j}
\right)_+ \r}
\|u_j\|_{\mathcal H_j}^2
\right)^{1/2}.
\]
\end{lemma}

\begin{remark}[Fractional Sobolev orders]
For simplicity, Lemma~\ref{lem:sobolev} is stated for integer Sobolev orders \(t_j\in\mathbb N_0\). This is because its proof relies on \cite[Theorem 4.1]{arcangeli2007}, which establishes the required sampling inequalities for integer target orders. By instead appealing to \cite[Theorem 3.1]{arcangeli2021}, the same argument extends to fractional Sobolev orders, yielding an analogous result for arbitrary \(t_j\ge 0\) within the admissible range.
\end{remark}

We continue by considering the learning-error term. The latter can be estimated in several ways. In this work, we adopt an approach based on Sobolev embeddings and sampling inequalities, which yields explicit convergence rates under suitable regularity assumptions on the reproducing kernel Hilbert space. This framework naturally accommodates vector-valued kernels and is particularly well suited to the encoder--decoder setting considered here.

\begin{lemma}[Learning error in a finite-dimensional RKHS]
\label{lem:finite-dimensional-learning-error}

Let
\(
\Upsilon\subset\mathbb R^{d_X}
\)
be a bounded domain with Lipschitz boundary, and let
\(
\mathbf U
=
\{U_1,\ldots,U_N\}
\subset\Upsilon
\)
be a finite sampling set with fill distance
\[
h_{\mathrm{tr}}
:=
\sup_{U\in\Upsilon}
\min_{1\le i\le N}
\|U-U_i\|_2.
\]

Let
\(
\Gamma:
\Upsilon\times\Upsilon
\to
\mathcal L(\mathbb R^{d_Y})
\)
be an operator-valued kernel with associated RKHS
\(
\mathcal H_\Gamma
\).
Assume that, for some
\(
\tau>\frac{d_X}{2},
\)
there is a continuous embedding
\(
\mathcal H_\Gamma
\hookrightarrow
\Hk{\tau}(\Upsilon;\mathbb R^{d_Y})
\simeq
\Hk{\tau}(\Upsilon)^{d_Y}.
\)

Let
\(
f\in\mathcal H_\Gamma
\),
and denote by
\(
\widehat f_\lambda
\)
the kernel interpolant, when \(\lambda=0\), or the kernel
ridge-regression estimator, when \(\lambda>0\), associated with the exact data
\(
\bigl\{(U_i,f(U_i))\bigr\}_{i=1}^N.
\)
If \(\lambda=0\), assume that
\(
\Gamma(\mathbf U,\mathbf U):
(\mathbb R^{d_Y})^N\to(\mathbb R^{d_Y})^N
\)
is boundedly invertible. Define
\(
e_\lambda
:=
f-\widehat f_\lambda.
\)

Let
\(q\in[1,\infty]\),
set
\(
\gamma:=\max\{2,q\},
\)
and let
\(
0\le s\le\ell(\tau,q,d_X),
\)
where
\[
\ell(\tau,q,d)
=
\begin{cases}
\ell_0(\tau,q,d),
&
\parbox[t]{0.62\linewidth}{if
\(
\tau\in\mathbb N
\)
and either
\(
q=2
\)
or
\(
q>2
\)
with
\(
\ell_0(\tau,q,d)\in\mathbb N,
\)}
\\[2ex]
\lceil\ell_0(\tau,q,d)\rceil-1,
&
\text{otherwise},
\end{cases}
\]
with
\[
\ell_0(\tau,q,d)
:=
\tau
-
d\left(\frac12-\frac1q\right)_+.
\]
When
\(q=\infty\),
assume additionally that
\(
s\in\mathbb N_0.
\)

Then, there exist constants \(h_0>0\) and \(C>0\), independent of
\(h_{\mathrm{tr}}\), \(\lambda\), and \(f\), such that, whenever
\(
h_{\mathrm{tr}}\le h_0,
\)
one has
\[
\begin{aligned}
|e_\lambda|_{\Wkp{s}{q}(\Upsilon;\mathbb R^{d_Y})}
\le
C\Bigg(
&
h_{\mathrm{tr}}^{
\tau-s-d_X\left(\frac12-\frac1q\right)_+
}
+
h_{\mathrm{tr}}^{d_X/\gamma-s}\sqrt{\lambda}
\Bigg)
\|f\|_{\mathcal H_\Gamma}.
\end{aligned}
\label{eq:finite-dimensional-learning-Sobolev}
\]
Consequently, for every \(U\in\Upsilon\),
\[
\|e_\lambda(U)\|_2
\le
C
\left(
h_{\mathrm{tr}}^{\tau-d_X/2}
+
\sqrt{\lambda}
\right)
\|f\|_{\mathcal H_\Gamma}.
\label{eq:finite-dimensional-pointwise-error}
\]
\end{lemma}

\begin{remark}[Alternative estimates via the power function]
The Sobolev-based approach of Lemma \ref{lem:finite-dimensional-learning-error} is not the only way to estimate the learning error. An alternative
is to bound the interpolation error directly in terms of the associated
power function (see Theorem \ref{thm:encoder-decoder-learning-error}). For example, in the scalar-valued kernel case, the decay of the power function can be estimated from the smoothness of the reproducing kernel, leading to convergence rates in terms of the
fill distance; see, for example, \cite{Wendland2004Scatter}.
\end{remark}

\begin{theorem}[Reconstruction guarantees for kernel multi-input, multi-output learning]
\label{thm:total-product-sobolev-error}

Let
\(
\mathcal H_X
=
\mathcal H\bigl(
J_X, \allowbreak
(\Omega_{X,j})_{j=1}^{J_X}, \allowbreak
(n_{X,j})_{j=1}^{J_X},
(s_{X,j})_{j=1}^{J_X}, \allowbreak
(p_{X,j})_{j=1}^{J_X},
(t_{X,j})_{j=1}^{J_X},
(q_{X,j})_{j=1}^{J_X},
(A_{X,j})_{j=1}^{J_X}
\bigr)
\)
and
\(
X
=
\allowbreak X\bigl(
J_X, \allowbreak
(\Omega_{X,j}\allowbreak)_{j=1}^{J_X},\) \(\allowbreak
\allowbreak (t_{X,j})_{j=1}^{J_X},
(q_{X,j})_{j=1}^{J_X}
\bigr)
\)
satisfy Assumption~\ref{assumption:productSobolevSpace}.
Likewise, let
\(
\mathcal H_Y
=
\mathcal H\bigl(
J_Y,
(\Omega_{Y,k})_{k=1}^{J_Y},
(n_{Y,k})_{k=1}^{J_Y},
(s_{Y,k})_{k=1}^{J_Y}, \allowbreak
(p_{Y,k})_{k=1}^{J_Y}, \allowbreak
(t_{Y,k})_{k=1}^{J_Y},
(q_{Y,k})_{k=1}^{J_Y}, \allowbreak
(A_{Y,k})_{k=1}^{J_Y}
\bigr)
\)
and
\(
Y
=
X\bigl(
J_Y,
(\Omega_{Y,k})_{k=1}^{J_Y},
(t_{Y,k})_{k=1}^{J_Y},
(q_{Y,k})_{k=1}^{J_Y}
\bigr)
\)
satisfy Assumption~\ref{assumption:productSobolevSpace}.
Let
\(
(E_{\mathcal H_X},D_{\mathcal H_X},Z_{\mathcal H_X})
\)
and
\(
(E_{\mathcal H_Y},D_{\mathcal H_Y},Z_{\mathcal H_Y})
\)
be the associated encoding maps, decoding maps, and measurement spaces from Assumption~\ref{assumptions:M1}. Assume that $G$ satisfies Assumption \ref{assumption:regularityG2}. Let
\(
d_X:=\sum_{j=1}^{J_X} m_{X,j}
\)
and 
\(
d_Y:=\sum_{k=1}^{J_Y} m_{Y,k},
\)
so that
\(
Z_{\mathcal H_X} \subseteq \mathbb R^{d_X}
\)
and 
\(
Z_{\mathcal H_Y} \subseteq \mathbb R^{d_Y}.
\)
Assume that Assumption \ref{assumption:sobolev-learning} is satisfied.

Then, there exist constants $C > 0$, $h_{\mathrm{tr}}^0>0$, $h_{X,j}^0>0$, $h_{Y,k}^0>0$
independent of the fill distances, \(\lambda\), and the functions being
reconstructed, such that, whenever
\(
h_{\mathrm{tr}}\le h_{\mathrm{tr}}^0,
\) \(
h_{X,j}\le h_{X,j}^0
\) for $j=1,\ldots,J_X$
and
\(
h_{Y,k}\le h_{Y,k}^0
\) for $k=1,\ldots,J_Y$,
one has, for every
\(
x=(x_1,\ldots,x_{J_X})\in B_R(\mathcal H_X),
\)
the estimate
\begin{align}
    \|G(x)-\overline{G}(x)\|_Y &\leq \omega \l C R \max_{1 \leq j \leq J_X}
h_{X,j}^{s_{X,j}-t_{X,j}
-
n_{X,j}
\left(
\frac{1}{p_{X,j}}-\frac{1}{q_{X,j}}
\right)_+}   \r \notag \\
&+  C M_R  \max_{1 \leq k \leq J_Y} h_{Y,k}^{s_{Y,k}-t_{Y,k}
-
n_{Y,k}
\left(
\frac{1}{p_{Y,k}}-\frac{1}{q_{Y,k}}
\right)_+} 
\notag \\
&+ \|D_Y\|_{\mathrm{op}}\,
    C
\left(
h_{\mathrm{tr}}^{\tau-d_X/2} +
\sqrt{\lambda}
\right) \|G_{\mathrm{enc}}\|_{\mathcal H_\Gamma}+ \|D_Y\|_{\mathrm{op}}\,
    \|\mathcal R_{\mathbf U,\lambda}\eta(E_Xx)\|_{2} \notag
\end{align}

\end{theorem}

\begin{remark}[Sample-sensor complexity trade-off]
One important results would be to derive a joint complexity bound that balances the number of sensors against the number of sampled training operators. Such a result is difficult because the statistical learning error depends on the fill distance of the encoded training samples in $Z_X$, which is generally not directly tractable. The encoded training points are obtained by applying the encoder to the original training inputs, and there is no reason for them to form a quasi-uniform or otherwise space-filling design in the encoded space. Existing convergence results in kernel-based operator learning typically circumvent this difficulty by assuming that the encoded training points become dense in the encoded domain (see \cite[Theorem 3.4]{BATLLE2024112549}), without explicitly quantifying the decay of the corresponding fill distance as a function of the number of training samples. A possible approach towards such quantitative estimates would be to study optimal experimental design in $Z_X$, selecting training inputs whose encoded representations minimize the fill distance. Under suitable geometric assumptions on the encoded manifold, or by constructing quasi-uniform designs directly in the encoded space and mapping them back to admissible inputs via suitable preimages or the decoder, one may then obtain explicit trade-offs between the number of sensors and the number of required training operators.
\end{remark}

\begin{remark}[Scaling with the number of tasks]
The reconstruction terms in Theorem~\ref{thm:total-product-sobolev-error}
exhibit the same qualitative scaling behaviour as the approximation results
obtained for multiple operator learning with neural operators. As shown in
\cite{weihs2026multipleneuraloperatorsachieve}, passing from a single operator
to multiple related operators, either through concatenation
\cite[Theorem 3.29]{weihs2026multipleneuraloperatorsachieve} or through
dedicated architectures
\cite[Theorem 3.19]{weihs2026multipleneuraloperatorsachieve}, need not alter
the approximation exponent, i.e., the approximation rates are governed by the least regular (or most difficult) task, rather than suffering from a compounding deterioration of the approximation exponent with the number of tasks. An analogous phenomenon appears in the reconstruction part of the present
kernel-based estimate. The input and output reconstruction errors are
controlled by the largest componentwise discretization errors. Consequently,
the associated convergence exponents are determined by the least favorable
Sobolev regularity among the component spaces, rather than by the number of
components itself.

The kernel-learning term requires a separate discussion. Its convergence rate is governed by
\(
    h_{\mathrm{tr}}^{\tau-d_X/2},
\)
and therefore depends on the dimension \(d_X\) of the encoded measurement space
and the regularity of the encoded operator \(G_{\mathrm{enc}}\). Since adding
tasks typically increases the total number of measurements, and hence the
dimension of the encoded measurement space, the learning rate may deteriorate
accordingly. This dependence is inherited from the underlying single-operator
kernel approximation theory rather than from the multi-task formulation itself.
\end{remark}

\begin{remark}[Extension to other measurement operators]
The Sobolev reconstruction results presented in Theorem \ref{thm:total-product-sobolev-error} are derived for pointwise
measurements. However, as noted in Remark \ref{rem:pointwise},  the subsequent learning analysis is independent of the
particular choice of measurements and only requires the encoded training data
\(
E_X(x_i)\in Z_X
\)
and
\(
E_Y(y_i)\in Z_Y
\).
Consequently, as noted in \cite{BATLLE2024112549}, the present framework extends immediately to other bounded linear
measurement operators, provided that suitable reconstruction estimates are
available (i.e. analogues of Lemma \ref{lem:sobolev}).

More precisely, suppose that the encoders \(E_X\) and \(E_Y\) are constructed
from a family of bounded linear measurements (e.g., local averages as in \cite[Lemma 14.34]{Owhadi_Scovel_2019}, moments,
Fourier coefficients, or other sensor functionals), and that the corresponding
decoders satisfy approximation estimates analogous to those established here for
point evaluations. Then the encoder--decoder error decomposition, the RKHS
construction, and all subsequent kernel learning results remain unchanged. Only
the reconstruction estimates need to be replaced by the corresponding sampling
inequalities for the chosen measurements. This separation highlights the modular structure of the framework: the
measurement modality enters exclusively through the approximation properties of
the encoder-decoder pair, whereas the kernel learning theory depends only on
the resulting finite-dimensional encoded spaces.

Extending the framework to genuinely
infinite-dimensional measurement spaces \(Z_X\) and \(Z_Y\) would require two
additional ingredients. First, one would need sampling inequalities for
bounded linear measurement operators taking values in infinite-dimensional
Hilbert spaces, such as trace operators, distributed observation operators, or
other function-valued measurements. Second, the Sobolev reconstruction theory
developed in this section would need to be extended to encoder--decoder pairs
with infinite-dimensional encoded spaces, yielding approximation estimates
analogous to Lemma~\ref{lem:finite-dimensional-learning-error}. Once such reconstruction results are
available, the abstract encoder--decoder framework and the subsequent kernel
learning analysis apply without essential modification.

\end{remark}

\subsection{Kernel Methods for Multi-Task Operator Learning} \label{sec:multiple}

Theorem~\ref{thm:total-product-sobolev-error} provides a general error estimate for kernel-based encoder--decoder learning of maps between product spaces of functions,
\[
G:X_1\times\cdots\times X_{J_X}
\longrightarrow
Y_1\times\cdots\times Y_{J_Y},
\]
where each component may represent a different modality, parameter, or function space. As a special case, it naturally recovers the kernel operator learning framework of \cite{BATLLE2024112549}, depicted in figure~\ref{fig:cd-classical-operator} and corresponding to the choice \(J_X=J_Y=1\). As a further important example and application, we now specialize this framework to multi-task/multiple operator learning, where the objective is to learn
\[
G:W\longrightarrow\{G[\alpha]:U\rightarrow V\}_{\alpha\in W},
\]
assigning to each task parameter \(\alpha\in W\) an operator \(G[\alpha]:U\rightarrow V\). While we focus on this representative application, the insights developed below, including the different formulations, their statistical and computational properties, and the settings in which each formulation is most appropriate, apply equally to general product-space learning problems.

We begin by introducing two equivalent viewpoints of the multiple operator learning problem. These are illustrated in figures~\ref{fig:cd-operator} and~\ref{fig:cd-product}. The first learns the operator-valued map assigning an operator to each parameter, while the second casts the same problem as learning the product-space map
\[
G':W\times U\to V,\qquad
G'(\alpha,u)=G[\alpha](u).
\]

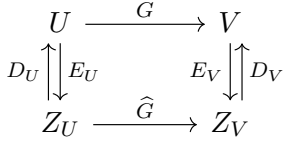
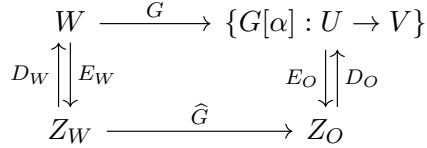
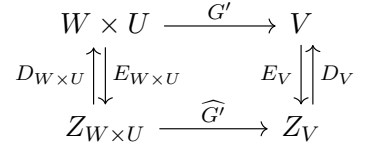
\begin{figure}[H]
\centering
\begin{subfigure}[t]{0.32\linewidth}
\centering
\begin{tikzcd}[column sep=3.6em, row sep=2.0em]
U \arrow[r, "G"] \arrow[d, "E_U" right]
  & V \arrow[d, "E_V" left] \\
Z_U \arrow[r, "\widehat G"]
  & Z_V
    \arrow[u, "D_V" right, shift right=0.9ex]
\arrow[from=2-1, to=1-1, "D_U" left, shift left=0.9ex]
\end{tikzcd}
\caption{Operator learning.}
\label{fig:cd-classical-operator}
\end{subfigure}
\hfill
\begin{subfigure}[t]{0.32\linewidth}
\centering
\begin{tikzcd}[column sep=3.6em, row sep=2.0em]
W \arrow[r, "G"]
  \arrow[d, "E_W" right]
&
\{G[\alpha]:U\to V\}
  \arrow[d, "E_O" left]
\\
Z_W \arrow[r, "\widehat G"]
&
Z_O
  \arrow[u, "D_O" right, shift right=0.9ex]
\arrow[from=2-1, to=1-1, "D_W" left, shift left=0.9ex]
\end{tikzcd}
\caption{Operator-valued learning.}
\label{fig:cd-operator}
\end{subfigure}
\hfill
\begin{subfigure}[t]{0.32\linewidth}
\centering
\begin{tikzcd}[column sep=3.6em, row sep=2.0em]
W \times U \arrow[r, "G'"] \arrow[d, "E_{W\times U}" right]
  & V \arrow[d, "E_V" left] \\
Z_{W\times U} \arrow[r, "\widehat{G'}"]
  & Z_V
    \arrow[u, "D_V" right, shift right=0.9ex]
\arrow[from=2-1, to=1-1, "D_{W\times U}" left, shift left=0.9ex]
\end{tikzcd}
\caption{Product-space learning.}
\label{fig:cd-product}
\end{subfigure}
\caption{
Encoder--decoder formulations for classical operator learning and multiple operator learning.
(a) Classical operator learning, where a single operator
\(G:U\to V\) is learned from measurements in the Hilbert spaces
\(Z_U\) and \(Z_V\) \cite{BATLLE2024112549}.
(b) Operator-valued multiple operator learning, where the objective is to learn the map
\(G:W\to\{G[\alpha]:U\to V\}_{\alpha\in W}\)
that assigns an operator to each parameter, using measurement spaces
\(Z_W\) and \(Z_O\).
(c) Product-space multiple operator learning, where the same problem is represented by the map
\(G':W\times U\to V\), defined by
\(G'(\alpha,u)=G[\alpha](u)\), with measurement spaces
\(Z_{W\times U}\) and \(Z_V\).
In each formulation, the encoder maps \(E\) map the original inputs and outputs into abstract Hilbert measurement spaces, the decoder maps \(D\) reconstruct representatives in the original spaces, and \(\widehat G\) (or \(\widehat G'\) in figure~(c)) denotes the learned surrogate.
}
\label{fig:cd-views}
\end{figure}

Although both formulations solve the same learning problem, they represent different prediction tasks. The operator-valued formulation learns an entire solution operator \(G[\alpha]:U\to V\) for each parameter \(\alpha\), so that a prediction consists of reconstructing the complete operator: this is therefore particularly attractive when one wishes to repeatedly query the same operator for many different input functions. In contrast, the product-space formulation learns individual operator evaluations through the map \(G':W\times U\to V\). Thus, it directly predicts the solution corresponding to a single parameter--input pair \((\alpha,u)\), making it a natural choice when only isolated evaluations are required or when different parameters are associated with different input functions.

These differing representations lead to distinct statistical and computational properties. The operator-valued formulation generalizes only across the parameter space \(W\), while the product-space formulation simultaneously generalizes across both the parameter space \(W\) and the input space \(U\). On the other hand, the operator-valued formulation requires one training sample per operator, whereas the product-space formulation requires one sample for every parameter--input pair. As a result, the size of the kernel matrix in the operator-valued approach scales with the number of operators, while the product-space formulation scales with the total number of parameter--input pairs. The operator-valued formulation is therefore expected to be considerably more computationally efficient whenever each operator is observed on many input functions, whereas the product-space formulation offers greater flexibility by learning directly over the joint parameter--input domain. Table~\ref{tab:mol-comparison} summarizes the principal differences between the two formulations. In Section~\ref{sec:experiments}, we compare both approaches empirically in terms of predictive performance and computational efficiency.

\begin{table}
\centering
\small
\renewcommand{\arraystretch}{2}
\setlength{\tabcolsep}{6pt}
\begin{tabularx}{\linewidth}{>{\bfseries}p{3.2cm} Y Y}
\toprule
& \textbf{Operator-valued learning} & \textbf{Product-space learning} \\
\midrule

Learned map
&
$G:W\rightarrow\{U\rightarrow V\}$
&
$G':W\times U\rightarrow V$
\\

Prediction
&
Entire operator $G[\alpha]$
&
Single operator evaluation $G[\alpha](u)$
\\

Generalization
&
Parameter space $W$
&
Parameter space $W$ and input space $U$
\\

Typical setting
&
Multiple tasks sharing common input functions
&
Multiple tasks with varying input functions
\\

Training samples
&
$n_{\alpha}$ operators
&
$n_{\alpha}n_u$ parameter--input pairs
\\

Size of kernel matrix
&
$n_{\alpha}\times n_{\alpha}$
&
$(n_{\alpha}n_u)\times(n_{\alpha}n_u)$
\\
Theoretical guarantees
&
Theorem~\ref{thm:total-product-sobolev-error} with \(J_X=J_Y=1\)  (for integral operators)
&
Theorem~\ref{thm:total-product-sobolev-error} with \(J_X=2,\;J_Y=1\)
\\
\bottomrule
\end{tabularx}
\caption{Comparison of the two multiple operator learning formulations. The operator-valued formulation predicts entire operators and generalizes across the parameter space, whereas the product-space formulation predicts individual operator evaluations and simultaneously generalizes across both the parameter and input spaces. Here, \(n_\alpha\) denotes the number of operators (or parameter instances) in the training set, and \(n_u\) denotes the number of input functions associated with each operator.}
\label{tab:mol-comparison}
\end{table}

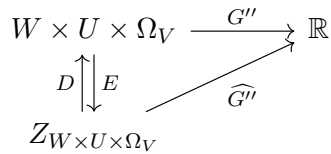
\begin{figure}[b]
\centering
\begin{tikzcd}[column sep=3.6em, row sep=2.0em]
W \times U \times \Omega_V \arrow[r, "G''"] \arrow[d, "E" right]
  & \mathbb{R} \\
Z_{W \times U \times \Omega_V} \arrow[ur, "\widehat{G''}"']
\arrow[from=2-1, to=1-1, "D" left, shift left=0.9ex]
  &
\end{tikzcd}
\caption{
Complete product-space formulation of multiple operator learning. Instead of learning the operator-valued map
\(G:W\to\{U\to V\}\) or the product-space map
\(G':W\times U\to V\), this formulation learns the pointwise evaluation map
\(G'':W\times U\times\Omega_V\to\mathbb{R}\), defined by
\(G''(\alpha,u,x)=G[\alpha](u)(x)\).
The encoder \(E\) maps parameter--input--location triples into a measurement space, while the decoder \(D\) reconstructs representatives in the original product space. The learned surrogate \(\widehat{G''}\) predicts the solution value at individual spatial locations.
}
\label{fig:cd-product-complete}
\end{figure}

\begin{remark}[Alternative mesh-transferability]  \label{rem:mesh-free-alternative}
An alternative formulation to those depicted in Figure \ref{fig:cd-views} is obtained by learning the pointwise evaluation map
\[
G'':W\times U\times\Omega_V\rightarrow\mathbb{R},
\qquad
(\alpha,u,x)\mapsto G[\alpha](u)(x),
\]
rather than the operator-valued map \(G:W\rightarrow\{U\rightarrow V\}\) or the product-space map \(G':W\times U\rightarrow V\); see Figure \ref{fig:cd-product-complete}. In this setting, the spatial coordinate is treated as an additional input variable, so that the learned surrogate directly predicts the solution value at any query location. Consequently, the resulting model is naturally mesh-independent and can be evaluated on arbitrary output discretizations directly.

The increased flexibility comes at a significant computational cost. Each training sample now corresponds to a parameter--input--location triple \((\alpha,u,x)\), so that the number of training examples grows proportionally to the product of the number of parameters, input functions, and output locations. Accordingly, kernel methods require solving linear systems whose size scales with the total number of such triples, leading to substantially higher memory and computational complexity than either the operator-valued or product-space formulations. For this reason, we focus on the two formulations in Figures \ref{fig:cd-operator} and \ref{fig:cd-product}, which already provide mesh-transferability through the encoder--decoder framework while remaining computationally tractable.
\end{remark}

Among the two formulations introduced above, the product-space formulation in Figure~\ref{fig:cd-product} fits directly within the setting of Theorem~\ref{thm:total-product-sobolev-error} by choosing \(J_X=2\) and \(J_Y=1\), with input space \(W\times U\) and output space \(V\). As a consequence, all approximation, reconstruction, and learning guarantees developed in Section \ref{sec:general} apply immediately. 

The operator-valued formulation of Figure \ref{fig:cd-operator} is more subtle. Although the encoder-decoder framework itself is formulated abstractly, the quantitative reconstruction theory developed in this work specializes to products of Hilbert spaces of functions admitting Sobolev embedding so that Sobolev sampling inequalities provide reconstruction error estimates as in Lemma \ref{lem:sobolev}. Consequently, applying the present theory to operator-valued learning requires the operator space to admit an appropriate realization as a function space. This is naturally the case for many important classes of operators, including integral operators \cite{conway2007course,brezis2010functional}
\[
(Tf)(x)=\int_{\Omega} k(x,y)\,f(y)\,\mathrm dy,
\]
which are canonically identified with their kernels \(k\). Such operators arise throughout scientific computing, for example as Green's operators for boundary value problems, kernel integral transforms such as the Fourier and Laplace transforms, and many classes of nonlocal operators \cite{Evans2010PDE,Duffy2015Green,Davies2002}. Under this identification, the operator-valued learning problem becomes a classical operator learning problem in which the outputs are the kernel functions \(k\). In other words, the operator-valued map
\(
G:W\rightarrow\{U\rightarrow V\}
\)
is identified with a function-valued map
\(
G:W\rightarrow K,
\)
so that the general encoder-decoder theory of Theorem \ref{thm:total-product-sobolev-error} applies directly with \(J_X=J_Y=1\). The resulting learning guarantees are therefore those of kernel operator learning. From the perspective of the encoder–decoder framework, the principal difference is that the outputs now belong to a Sobolev space over the product domain \(\Omega\times\Omega\), rather than \(\Omega\). Thus, the reconstruction rates inherited from the Sobolev sampling inequalities depend on the dimension of the product domain, typically leading to more stringent sampling requirements than in the standard setting.

We leave the extension of the reconstruction theory developed here directly to general spaces of operators, without relying on an underlying functional representation, as a future work. Such a theory would substantially broaden the applicability of kernel-based encoder-decoder learning and could enable analogous approximation, reconstruction, and learning guarantees for genuinely operator-valued prediction problems. 

\begin{remark}[Learning maps between measures]
\label{rem:measure-learning}
As with the integral operators discussed above, the encoder--decoder framework can also accommodate learning problems involving
measures that are absolutely continuous with respect to a fixed reference
measure and whose densities belong to the function spaces considered in
Section~\ref{sec:mainMulti}. Identifying each measure with its density reduces
the measure-valued learning problem to a function-valued one, allowing the
approximation theory developed above to be applied directly.
\end{remark}

\section{Numerical Experiments} \label{sec:experiments}

In this section, we evaluate the performance and efficiency of our proposed kernel-based multiple operator learning methods for both operator-valued learning and product-space learning, depicted in Figures \ref{fig:cd-operator} and \ref{fig:cd-product}, respectively. The code implementing both methods and reproducing all numerical experiments is publicly available at \url{https://github.com/liaochunyang/kernelMO}.

\paragraph{Learning task.}
We consider five representative parametric PDEs which were also studied in \cite{weihs2025MNO}. In all experiments, we learn the solution operator
\(
(\alpha,u_0)\mapsto u,
\)
mapping a parametric function/parameters \(\alpha\) and an initial condition \(u_0\) to the complete solution trajectory
\(
u:(0,2]\times[0,2]\rightarrow\mathbb{R},\) 
\(
(t,x_{\mathrm{spatial}})\mapsto u(t,x_{\mathrm{spatial}}).
\)
Thus, our models predict the full spatio-temporal evolution of the PDE solution, whereas the experiments in \cite{BATLLE2024112549} predict the solution only at a fixed time \(t=t_0\).

\paragraph{Sampling of initial conditions and parameter functions.}
We generate the initial conditions using the sinusoidal construction proposed in \cite{takamoto2022pdebench}:
\begin{equation}
\label{eq:IC}
u_0(x)=\sum_{i=1}^{2}A_i\sin(\pi n_i x+\phi_i),
\end{equation}
where the amplitudes \(A_i\) are sampled uniformly from \([0,1]\), the frequencies \(n_i\) are sampled uniformly from the integers \(\{1,\ldots,4\}\), and the phases \(\phi_i\) are sampled uniformly from \((0,2\pi)\). As in \cite{takamoto2022pdebench}, each sampled initial condition is post-processed by flipping its sign with probability \(0.5\). The construction of the parameter functions \(\alpha\) is PDE-specific and is described below.

\paragraph{Training datasets}
For each learning formulation, we train all models on the same underlying collection of parametric functions and initial conditions. The two learning formulations differ only in how this data is organized.
For the operator-valued formulation, each training sample consists of an entire operator represented through its evaluations on a common set of probe functions. The training dataset is therefore
\[
\mathcal D_{\mathrm{op}}
=
\left\{
\left(
\alpha_i,
\bigl\{
G[\alpha_i](u_j)
\bigr\}_{j=1}^{20}
\right)
\right\}_{i=1}^{320},
\]
where the initial conditions
\(
u_1,\ldots,u_{20}
\)
are fixed across all parameter instances.
In contrast, the product-space formulation treats each parameter--input pair as an individual training sample. The corresponding dataset is
\[
\mathcal D_{\mathrm{prod}}
=
\left\{
(\alpha_i,u_j,G[\alpha_i](u_j))
\;:\;
i=1,\ldots,200,\;
j=1,\ldots,50
\right\}.
\]

\begin{remark}[Fixed initial conditions in the operator-valued formulation] \label{rem:fixedIC}
In the operator-valued learning formulation in figure \ref{fig:cd-operator}, the initial conditions are kept fixed across all parameter instances. This reflects the fact that each training sample corresponds to an entire operator rather than a single operator evaluation. More precisely, \(E_O\) is the common encoder, mapping each operator to its evaluations on a fixed collection of input functions together with a fixed discretization of the output domain. This ensures that all operators are represented in the same measurement space \(Z_O\). Varying the initial conditions would represent different encoding maps and thus alter the operator-valued learning problem. In contrast, the product-space formulation naturally permits different initial conditions for each parameter instance, since each training sample corresponds to a single parameter--input pair rather than an entire operator.
\end{remark}

We additionally evaluate the generalization capabilities of each framework using several out-of-distribution (OOD) datasets. In the operator-valued formulation, the initial conditions are fixed and shared across all parameter instances, so we only consider OOD parameter functions. In the product-space formulation, where both the parameter functions and initial conditions are sampled, we consider OOD parameter functions, OOD initial conditions, and the joint OOD setting in which both are sampled from distributions not encountered during training. Details of the construction of all OOD datasets are provided in Appendix~\ref{Appendix:ood}. In Table \ref{tab:number_of_test}, we report the numbers of parametric functions, initial conditions, and test samples in in-distribution and out-of-distribution test sets under both settings.

\begin{table}[t]
\centering
\begin{tabular}{ccccc}
\hline
Test set & Setting & $\#$ parametric functions & $\#$ initial conditions & $N_{\rm test}$ \\
\hline
\multirow{2}{*}{In-distribution} 
& Operator-valued & 80 & 20 & $80\times20$ \\
& Product-space & 80 & 50 & $80\times50$ \\
\hline
\multirow{2}{*}{Out-of-distribution}
 & Operator-valued & 40 & 20 & $40\times20$ \\
& Product-space & 80 & 50 & $80\times50$ \\
\hline
\end{tabular}
\caption{Number of parametric functions, number of initial conditions, and number of test samples for in-distribution and out-of-distribution sets under both settings. In the operator-valued setting, the initial conditions are fixed across all settings. In the product-space setting, the initial conditions are distinct.}
\label{tab:number_of_test}
\end{table}

\paragraph{Methods and benchmarks.}
We evaluate the proposed kernel-based multiple operator learning framework against both classical operator learning methods and existing multiple operator learning architectures.

\begin{itemize}
    \item \textit{Proposed methods.}
    We consider the proposed kernel-based multiple operator learning framework, denoted by \KernelMO, with operator-valued and product-space variants \KernelMOOV{} and \KernelMOPS{}, respectively.

    \item \textit{Classical Operator Learning benchmarks.}
    These methods do not explicitly account for multiple operators. We consider:
    \begin{itemize}
        \item \KernelO{} \cite{BATLLE2024112549}, the single operator learning kernel method corresponding to the formulation in figure~\ref{fig:cd-classical-operator};
        \item DeepONet \cite{deepOnet}, the standard neural operator architecture corresponding to the same formulation.
    \end{itemize}
    Both methods treat the initial condition as the sole input and the PDE solution as the output, without incorporating the parameter function.

    \item \textit{Multiple Operator Learning benchmarks.} 
    We compare against several architectures specifically designed or adapted for multiple operator learning:
    \begin{itemize}
        \item DeepONet-C \cite{weihs2025MNO,weihs2026multipleneuraloperatorsachieve}, a variant of DeepONet in which the parameter function and initial condition are concatenated into a single input function;
        \item MIONet \cite{mionet}, which employs separate branch networks for each input function and combines their latent representations through a tensor-product operation;
        \item MNO \cite{weihs2025MNO}, which uses separate neural networks for the parametric function, input function, and spatial coordinates before combining their latent representations.
    \end{itemize}
\end{itemize}

All kernel-based models are trained by solving a minimum-norm interpolation optimization problem. 
All neural-network-based method are trained with mean squared loss. 

\begin{remark}[Well-posedness of the operator-valued learning problem]
The classical operator learning baselines (\KernelO{} and DeepONet) cannot be accurately trained in the operator-valued setting because they ignore the parameter function \(\alpha\). Their induced training set would be
\[
\{(u_j,G[\alpha_i](u_j))\}_{i=1,\ldots,N}^{j=1,\ldots,m},
\]
where the probe functions \(u_1,\ldots,u_m\) are fixed across all parameter instances (see Remark~\ref{rem:fixedIC}). Since there generally exist \(\alpha_1\neq\alpha_2\) such that
\(
G[\alpha_1](u_j)\neq G[\alpha_2](u_j),
\)
the same input \(u_j\) is paired with multiple outputs. Thus, the induced training set does not define a function \(U\to V\), and classical operator learning methods are not applicable. In contrast, the product-space formulation is trained on individual triples
\(
(\alpha,u,G[\alpha](u)),
\)
rather than collections of outputs evaluated on a fixed set of probe functions. Thus, every input is associated with a unique output, and the resulting supervised learning problem is well defined.
\end{remark}

For the proposed \KernelMO{} methods, the theoretical framework employs matrix-valued kernels. In practice, we use separable kernels of the form
\[
K(x,x') = k(x,x')\,\Id,
\]
where \(k\) is a scalar-valued kernel and \(\Id\) denotes the identity matrix of appropriate dimension. For \KernelMOOV{}, the inputs consist solely of the parametric functions, and we therefore choose
\[
k:W\times W\rightarrow\mathbb R.
\] For \KernelMOPS{}, the inputs are pairs \((\alpha,u)\in W\times U\) and we use the product kernel
\[
k\big((\alpha_1,u_1),(\alpha_2,u_2)\big)
=
k_W(\alpha_1,\alpha_2)
k_U(u_1,u_2)
\]
for kernels $k_W:W\times W \rightarrow \bbR$ and $k_U: U \times U \to \bbR$. The hyperparameters for \(k\), \(k_W\), and \(k_U\) are selected independently. For completeness, we also train \KernelMOPS on the Cartesian set of
$(\alpha,u)$ pairs underlying the operator-valued dataset. This
allows the two kernel formulations to be compared using exactly the same PDE evaluations.

We also consider two representations of the sampled data for the proposed kernel methods. In the first, each function is represented directly by its pointwise values at the observation locations, and kernel regression is performed on these high-dimensional observations. Since numerical simulations only provide sampled function values, this corresponds to the standard discrete kernel learning setting. Although this does not employ the encoder--decoder framework, it still constitutes a novel kernel-based multiple operator learning method under the proposed operator-valued and product-space formulations. We therefore include it as a high-dimensional baseline to isolate the effect of the encoder--decoder approach.
In the second, we employ principal component analysis (PCA) to obtain a low-dimensional latent representation of the sampled observations. Kernel regression is then performed in this latent space, and the learned representation is mapped back to the observation space through the inverse PCA transform. This corresponds to the full encoder--decoder framework analyzed in Section~\ref{sec:general}. Throughout the experiments, we refer to the corresponding operator-valued and product-space variants as \KernelMOOV (PCA) and \KernelMOPS (PCA), respectively.

\begin{remark}[Construction of the PCA representation]
The PCA bases are computed from the training data. In the product-space formulation, we construct a single PCA basis for the parametric functions, a single PCA basis for the initial conditions, and a single PCA basis for the solution functions. These bases are shared across all training samples. In the operator-valued formulation, a single PCA basis is likewise constructed for the parametric functions. The probe functions \(u_1,\ldots,u_m\) are fixed and are therefore not encoded. Instead, for each probe function \(u_j\), we compute a separate PCA basis from the collection of corresponding solution functions
\[
\{G[\alpha_i](u_j)\}_{i=1}^N.
\]
Thus, the outputs associated with each fixed probe function are encoded in their own latent space, which is shared across all parameter instances.
\end{remark}

\paragraph{Evaluation metric.} We evaluate all models on withheld test datasets consisting of in-distribution samples and out-of-distribution samples (see Table \ref{tab:number_of_test}). 
For each test case, the predicted and reference solution functions are evaluated 
on a $128\times64$ spacetime grid over $[0,2]\times[0,2]$. 
For each test sample, we compute the relative $\Lp{2}$ error
\[
e_i
=
\frac{\|u^{(i)}_{\rm pred}-u^{(i)}_{\rm target}\|_2}
{\|u^{(i)}_{\rm target}\|_2}
\]
where $(u^{(i)}_{\rm pred}, u^{(i)}_{\rm target})$ corresponds to the $i$-th pair of predicted and reference solution functions, each evaluated at all points of discretized spacetime grid. 
We report the mean relative $L^2$ error
\[
\frac{1}{N_{\rm test}}
\sum_{i=1}^{N_{\rm test}} e_i,
\]
together with the corresponding standard deviation computed over the test samples.

\paragraph{Hyperparameter optimization.} 
For the kernel-based methods, we consider both radial basis function (RBF, denoted by R) and Matérn (denoted by M) kernels and perform a simple hyperparameter search. We denote the corresponding methods by appending the latent representation and kernel choice. For example, \KernelMOOV{} (PCA) / R and \KernelMOOV{} (PCA) / M denote the operator-valued method with a PCA representation and an RBF or Matérn kernel, respectively. For the product-space formulation, we specify the kernel family used on each input component separately. Thus, \KernelMOPS{} (PCA) / M \(\times\) M denotes Matérn kernels on both the parameter and initial-condition spaces, while \KernelMOPS{} (PCA) / R \(\times\) R denotes an RBF kernel on the parameter space and a RBF kernel on the initial-condition space. The same notation applies to the \KernelO method as well.

For the RBF kernel, we search over the length-scale parameter $\gamma=\{0.01,0.1,1,10,100\}$. For the Mat{\'e}rn kernel, we search over the length-scale parameter $\gamma=\{0.01,0.1,1,10,100\}$ and the smoothness parameter $\nu=\{1/2,3/2,5/2,7/2\}$. Exact hyperparameter choices are detailed in Appendix \ref{sec:hyperparameters}. 
For neural-networks-based methods, we consider three different network sizes, referred to as small, medium, and large. In Table~\ref{tab:NN_size}, we summarize the architectures and the corresponding numbers of trainable parameters.

\begin{table}[!ht]
\centering
\begin{tabular}{c|c|c|c}
 & small & medium & large  \\\hline
DeepONet & 0.29M & 2.15M & 9.48M \\ 
DeepONet-C & 0.31M & 2.17M & 9.53M \\
MIONet & 0.32M & 2.22M & 9.73M \\
MNO & 0.60M & 1.22M & 16.70M \\\hline
\end{tabular}
\caption{Summary of number of trainable parameters for each architecture and network size. "M" denotes million.}
\label{tab:NN_size}
\end{table}

\subsection{Performance} \label{sec:performance}

We begin this section with a summary of the main experimental findings. Figures~\ref{fig:summary_heatmap_ov} and~\ref{fig:summary_heatmap_ps} provide an overview of the relative performance of all methods across the five PDE benchmarks for the operator-valued and product-space formulations, respectively. In addition, Figures~\ref{fig:qualitative_grid_framework1_operator_test} and~\ref{fig:qualitative_grid_framework2_product_test} show representative sample trajectories from the test datasets for both formulations on all PDE benchmarks. Detailed quantitative results for each PDE benchmark are presented in the subsequent subsections, while representative results on some of the out-of-distribution datasets are shown in Figures~\ref{fig:qualitative_framework1_operator_conservation_ood_ood_h5}-\ref{fig:qualitative_framework2_product_param_wave_ood_ood_par_var_h5}.

%

\definecolor{hmone}{RGB}{255,248,220}    
\definecolor{hmtwo}{RGB}{255,236,170}
\definecolor{hmthree}{RGB}{255,220,100}
\definecolor{hmfour}{RGB}{255,209,0}     
\definecolor{hmfive}{RGB}{185,218,238}
\definecolor{hmsix}{RGB}{120,180,220}
\definecolor{hmseven}{RGB}{33,116,174}   
\definecolor{hmeight}{RGB}{0,85,135}     


\begin{figure}
\centering
\resizebox{\textwidth}{!}{%
\begin{tikzpicture}[x=1.42cm,y=0.58cm,font=\sffamily]
\node[rotate=35,anchor=west] at (0.5,1.05) {Conservation Law};
\node[rotate=35,anchor=west] at (1.5,1.05) {Diffusion-Reaction-Advection};
\node[rotate=35,anchor=west] at (2.5,1.05) {Nonlinear Klein--Gordon};
\node[rotate=35,anchor=west] at (3.5,1.05) {Parametric Diffusion-Reaction};
\node[rotate=35,anchor=west] at (4.5,1.05) {Parametric Wave};

\node[anchor=east,font=\small] at (-0.12,0) {\textbf{KernelMO-OV / M}};
\node[draw=white,line width=0.35pt,fill=hmone,minimum width=1.40cm,minimum height=0.56cm,text=black,font=\small] at (0.5,0) {1.00};
\node[draw=white,line width=0.35pt,fill=hmthree,minimum width=1.40cm,minimum height=0.56cm,text=black,font=\small] at (1.5,0) {1.52};
\node[draw=white,line width=0.35pt,fill=hmtwo,minimum width=1.40cm,minimum height=0.56cm,text=black,font=\small] at (2.5,0) {1.32};
\node[draw=white,line width=0.35pt,fill=hmone,minimum width=1.40cm,minimum height=0.56cm,text=black,font=\small] at (3.5,0) {1.00};
\node[draw=white,line width=0.35pt,fill=hmthree,minimum width=1.40cm,minimum height=0.56cm,text=black,font=\small] at (4.5,0) {1.71};

\node[anchor=east,font=\small] at (-0.12,-1) {\textbf{KernelMO-OV / R}};
\node[draw=white,line width=0.35pt,fill=hmfive,minimum width=1.40cm,minimum height=0.56cm,text=black,font=\small] at (0.5,-1) {4.55};
\node[draw=white,line width=0.35pt,fill=hmone,minimum width=1.40cm,minimum height=0.56cm,text=black,font=\small] at (1.5,-1) {1.03};
\node[draw=white,line width=0.35pt,fill=hmfour,minimum width=1.40cm,minimum height=0.56cm,text=black,font=\small] at (2.5,-1) {2.24};
\node[draw=white,line width=0.35pt,fill=hmseven,minimum width=1.40cm,minimum height=0.56cm,text=white,font=\small] at (3.5,-1) {10.3};
\node[draw=white,line width=0.35pt,fill=hmsix,minimum width=1.40cm,minimum height=0.56cm,text=black,font=\small] at (4.5,-1) {5.18};

\node[anchor=east,font=\small] at (-0.12,-2) {\textbf{KernelMO-OV (PCA) / M}};
\node[draw=white,line width=0.35pt,fill=hmeight,minimum width=1.40cm,minimum height=0.56cm,text=white,font=\small] at (0.5,-2) {90.0};
\node[draw=white,line width=0.35pt,fill=hmthree,minimum width=1.40cm,minimum height=0.56cm,text=black,font=\small] at (1.5,-2) {1.89};
\node[draw=white,line width=0.35pt,fill=hmthree,minimum width=1.40cm,minimum height=0.56cm,text=black,font=\small] at (2.5,-2) {1.98};
\node[draw=white,line width=0.35pt,fill=hmseven,minimum width=1.40cm,minimum height=0.56cm,text=white,font=\small] at (3.5,-2) {19.8};
\node[draw=white,line width=0.35pt,fill=hmtwo,minimum width=1.40cm,minimum height=0.56cm,text=black,font=\small] at (4.5,-2) {1.31};

\node[anchor=east,font=\small] at (-0.12,-3) {\textbf{KernelMO-OV (PCA) / R}};
\node[draw=white,line width=0.35pt,fill=hmeight,minimum width=1.40cm,minimum height=0.56cm,text=white,font=\small] at (0.5,-3) {90.3};
\node[draw=white,line width=0.35pt,fill=hmtwo,minimum width=1.40cm,minimum height=0.56cm,text=black,font=\small] at (1.5,-3) {1.47};
\node[draw=white,line width=0.35pt,fill=hmfour,minimum width=1.40cm,minimum height=0.56cm,text=black,font=\small] at (2.5,-3) {2.88};
\node[draw=white,line width=0.35pt,fill=hmseven,minimum width=1.40cm,minimum height=0.56cm,text=white,font=\small] at (3.5,-3) {20.5};
\node[draw=white,line width=0.35pt,fill=hmone,minimum width=1.40cm,minimum height=0.56cm,text=black,font=\small] at (4.5,-3) {1.02};

\node[anchor=east,font=\small] at (-0.12,-4) {\textbf{KernelMO-PS / M$\times$M}};
\node[draw=white,line width=0.35pt,fill=hmone,minimum width=1.40cm,minimum height=0.56cm,text=black,font=\small] at (0.5,-4) {1.00};
\node[draw=white,line width=0.35pt,fill=hmthree,minimum width=1.40cm,minimum height=0.56cm,text=black,font=\small] at (1.5,-4) {1.79};
\node[draw=white,line width=0.35pt,fill=hmtwo,minimum width=1.40cm,minimum height=0.56cm,text=black,font=\small] at (2.5,-4) {1.32};
\node[draw=white,line width=0.35pt,fill=hmfour,minimum width=1.40cm,minimum height=0.56cm,text=black,font=\small] at (3.5,-4) {2.17};
\node[draw=white,line width=0.35pt,fill=hmthree,minimum width=1.40cm,minimum height=0.56cm,text=black,font=\small] at (4.5,-4) {1.71};

\node[anchor=east,font=\small] at (-0.12,-5) {\textbf{KernelMO-PS / R$\times$R}};
\node[draw=white,line width=0.35pt,fill=hmeight,minimum width=1.40cm,minimum height=0.56cm,text=white,font=\small] at (0.5,-5) {76.0};
\node[draw=white,line width=0.35pt,fill=hmsix,minimum width=1.40cm,minimum height=0.56cm,text=black,font=\small] at (1.5,-5) {5.25};
\node[draw=white,line width=0.35pt,fill=hmseven,minimum width=1.40cm,minimum height=0.56cm,text=white,font=\small] at (2.5,-5) {27.7};
\node[draw=white,line width=0.35pt,fill=hmseven,minimum width=1.40cm,minimum height=0.56cm,text=white,font=\small] at (3.5,-5) {16.7};
\node[draw=white,line width=0.35pt,fill=hmthree,minimum width=1.40cm,minimum height=0.56cm,text=black,font=\small] at (4.5,-5) {1.72};

\node[anchor=east,font=\small] at (-0.12,-6) {\textbf{KernelMO-PS (PCA) / M$\times$M}};
\node[draw=white,line width=0.35pt,fill=hmeight,minimum width=1.40cm,minimum height=0.56cm,text=white,font=\small] at (0.5,-6) {90.0};
\node[draw=white,line width=0.35pt,fill=hmfour,minimum width=1.40cm,minimum height=0.56cm,text=black,font=\small] at (1.5,-6) {2.11};
\node[draw=white,line width=0.35pt,fill=hmthree,minimum width=1.40cm,minimum height=0.56cm,text=black,font=\small] at (2.5,-6) {1.98};
\node[draw=white,line width=0.35pt,fill=hmseven,minimum width=1.40cm,minimum height=0.56cm,text=white,font=\small] at (3.5,-6) {19.8};
\node[draw=white,line width=0.35pt,fill=hmtwo,minimum width=1.40cm,minimum height=0.56cm,text=black,font=\small] at (4.5,-6) {1.29};

\node[anchor=east,font=\small] at (-0.12,-7) {\textbf{KernelMO-PS (PCA) / R$\times$R}};
\node[draw=white,line width=0.35pt,fill=hmeight,minimum width=1.40cm,minimum height=0.56cm,text=white,font=\small] at (0.5,-7) {120};
\node[draw=white,line width=0.35pt,fill=hmsix,minimum width=1.40cm,minimum height=0.56cm,text=black,font=\small] at (1.5,-7) {5.33};
\node[draw=white,line width=0.35pt,fill=hmseven,minimum width=1.40cm,minimum height=0.56cm,text=white,font=\small] at (2.5,-7) {27.7};
\node[draw=white,line width=0.35pt,fill=hmseven,minimum width=1.40cm,minimum height=0.56cm,text=white,font=\small] at (3.5,-7) {23.3};
\node[draw=white,line width=0.35pt,fill=hmone,minimum width=1.40cm,minimum height=0.56cm,text=black,font=\small] at (4.5,-7) {1.01};

\node[anchor=east,font=\small] at (-0.12,-8) {DeepONet-C};
\node[draw=white,line width=0.35pt,fill=hmeight,minimum width=1.40cm,minimum height=0.56cm,text=white,font=\small] at (0.5,-8) {202};
\node[draw=white,line width=0.35pt,fill=hmfive,minimum width=1.40cm,minimum height=0.56cm,text=black,font=\small] at (1.5,-8) {4.74};
\node[draw=white,line width=0.35pt,fill=hmeight,minimum width=1.40cm,minimum height=0.56cm,text=white,font=\small] at (2.5,-8) {31.7};
\node[draw=white,line width=0.35pt,fill=hmseven,minimum width=1.40cm,minimum height=0.56cm,text=white,font=\small] at (3.5,-8) {14.7};
\node[draw=white,line width=0.35pt,fill=hmsix,minimum width=1.40cm,minimum height=0.56cm,text=black,font=\small] at (4.5,-8) {6.66};

\node[anchor=east,font=\small] at (-0.12,-9) {MIONet};
\node[draw=white,line width=0.35pt,fill=hmeight,minimum width=1.40cm,minimum height=0.56cm,text=white,font=\small] at (0.5,-9) {65.3};
\node[draw=white,line width=0.35pt,fill=hmsix,minimum width=1.40cm,minimum height=0.56cm,text=black,font=\small] at (1.5,-9) {9.27};
\node[draw=white,line width=0.35pt,fill=hmseven,minimum width=1.40cm,minimum height=0.56cm,text=white,font=\small] at (2.5,-9) {23.5};
\node[draw=white,line width=0.35pt,fill=hmseven,minimum width=1.40cm,minimum height=0.56cm,text=white,font=\small] at (3.5,-9) {12.7};
\node[draw=white,line width=0.35pt,fill=hmsix,minimum width=1.40cm,minimum height=0.56cm,text=black,font=\small] at (4.5,-9) {7.49};

\node[anchor=east,font=\small] at (-0.12,-10) {MNO};
\node[draw=white,line width=0.35pt,fill=hmeight,minimum width=1.40cm,minimum height=0.56cm,text=white,font=\small] at (0.5,-10) {94.6};
\node[draw=white,line width=0.35pt,fill=hmfour,minimum width=1.40cm,minimum height=0.56cm,text=black,font=\small] at (1.5,-10) {2.52};
\node[draw=white,line width=0.35pt,fill=hmseven,minimum width=1.40cm,minimum height=0.56cm,text=white,font=\small] at (2.5,-10) {13.3};
\node[draw=white,line width=0.35pt,fill=hmsix,minimum width=1.40cm,minimum height=0.56cm,text=black,font=\small] at (3.5,-10) {9.19};
\node[draw=white,line width=0.35pt,fill=hmsix,minimum width=1.40cm,minimum height=0.56cm,text=black,font=\small] at (4.5,-10) {7.27};
\end{tikzpicture}%
}
\caption{Summary of normalized performance in the operator-valued-learning
experiments. For each method \(m\), PDE \(p\), and associated test dataset
\(d\), we compute the normalized error
\(
\mathrm{Error}_{m,p,d}/\min_{m'}\mathrm{Error}_{m',p,d},
\)
where \(\mathrm{Error}_{m,p,d}\) denotes the mean relative error. Consequently,
the best-performing method on each test dataset attains the value \(1\). For
each PDE, the displayed value is obtained by averaging the normalized errors
over all associated test datasets (one in-distribution and all
out-of-distribution datasets). Lighter colors indicate better performance.}
\label{fig:summary_heatmap_ov}
\end{figure}

\begin{figure}
\centering
\resizebox{\textwidth}{!}{%
\begin{tikzpicture}[x=1.42cm,y=0.58cm,font=\sffamily]
\node[rotate=35,anchor=west] at (0.5,1.05) {Conservation Law};
\node[rotate=35,anchor=west] at (1.5,1.05) {Diffusion-Reaction-Advection};
\node[rotate=35,anchor=west] at (2.5,1.05) {Nonlinear Klein--Gordon};
\node[rotate=35,anchor=west] at (3.5,1.05) {Parametric Diffusion-Reaction};
\node[rotate=35,anchor=west] at (4.5,1.05) {Parametric Wave};

\node[anchor=east,font=\small] at (-0.12,0) {KernelO / M};
\foreach \x/\value/\shade/\ink in {
  0.5/3.97/hmfive/black,
  1.5/5.89/hmsix/black,
  2.5/28.0/hmseven/white,
  3.5/4.90/hmfive/black,
  4.5/15.3/hmseven/white}
  \node[draw=white,line width=0.35pt,fill=\shade,minimum width=1.40cm,
        minimum height=0.56cm,text=\ink,font=\small] at (\x,0) {\value};

\node[anchor=east,font=\small] at (-0.12,-1) {KernelO / R};
\foreach \x/\value/\shade/\ink in {
  0.5/2.25/hmfour/black,
  1.5/139/hmeight/white,
  2.5/141/hmeight/white,
  3.5/3.96/hmfive/black,
  4.5/125/hmeight/white}
  \node[draw=white,line width=0.35pt,fill=\shade,minimum width=1.40cm,
        minimum height=0.56cm,text=\ink,font=\small] at (\x,-1) {\value};

\node[anchor=east,font=\small] at (-0.12,-2) {KernelO (PCA) / M};
\foreach \x/\value/\shade/\ink in {
  0.5/2.20/hmfour/black,
  1.5/4.39/hmfive/black,
  2.5/27.8/hmseven/white,
  3.5/4.91/hmfive/black,
  4.5/15.7/hmseven/white}
  \node[draw=white,line width=0.35pt,fill=\shade,minimum width=1.40cm,
        minimum height=0.56cm,text=\ink,font=\small] at (\x,-2) {\value};

\node[anchor=east,font=\small] at (-0.12,-3) {KernelO (PCA) / R};
\foreach \x/\value/\shade/\ink in {
  0.5/1.88/hmthree/black,
  1.5/26.7/hmseven/white,
  2.5/56.5/hmeight/white,
  3.5/3.81/hmfive/black,
  4.5/22.6/hmseven/white}
  \node[draw=white,line width=0.35pt,fill=\shade,minimum width=1.40cm,
        minimum height=0.56cm,text=\ink,font=\small] at (\x,-3) {\value};

\node[anchor=east,font=\small] at (-0.12,-4) {\textbf{KernelMO-PS / M$\times$M}};
\foreach \x/\value/\shade/\ink in {
  0.5/1.23/hmtwo/black,
  1.5/1.04/hmone/black,
  2.5/1.10/hmtwo/black,
  3.5/1.45/hmtwo/black,
  4.5/1.67/hmthree/black}
  \node[draw=white,line width=0.35pt,fill=\shade,minimum width=1.40cm,
        minimum height=0.56cm,text=\ink,font=\small] at (\x,-4) {\value};

\node[anchor=east,font=\small] at (-0.12,-5) {\textbf{KernelMO-PS / R$\times$R}};
\foreach \x/\value/\shade/\ink in {
  0.5/1.94/hmthree/black,
  1.5/1.27/hmtwo/black,
  2.5/3.04/hmfive/black,
  3.5/14.4/hmseven/white,
  4.5/1.79/hmthree/black}
  \node[draw=white,line width=0.35pt,fill=\shade,minimum width=1.40cm,
        minimum height=0.56cm,text=\ink,font=\small] at (\x,-5) {\value};

\node[anchor=east,font=\small] at (-0.12,-6)
  {\textbf{KernelMO-PS (PCA) / M$\times$M}};
\foreach \x/\value/\shade/\ink in {
  0.5/1.32/hmtwo/black,
  1.5/1.03/hmone/black,
  2.5/1.20/hmtwo/black,
  3.5/14.8/hmseven/white,
  4.5/1.13/hmtwo/black}
  \node[draw=white,line width=0.35pt,fill=\shade,minimum width=1.40cm,
        minimum height=0.56cm,text=\ink,font=\small] at (\x,-6) {\value};

\node[anchor=east,font=\small] at (-0.12,-7)
  {\textbf{KernelMO-PS (PCA) / R$\times$R}};
\foreach \x/\value/\shade/\ink in {
  0.5/1.82/hmthree/black,
  1.5/1.26/hmtwo/black,
  2.5/2.02/hmfour/black,
  3.5/21.8/hmseven/white,
  4.5/1.00/hmone/black}
  \node[draw=white,line width=0.35pt,fill=\shade,minimum width=1.40cm,
        minimum height=0.56cm,text=\ink,font=\small] at (\x,-7) {\value};

\node[anchor=east,font=\small] at (-0.12,-8) {DeepONet};
\foreach \x/\value/\shade/\ink in {
  0.5/2.10/hmfour/black,
  1.5/3.61/hmfive/black,
  2.5/20.6/hmseven/white,
  3.5/4.25/hmfive/black,
  4.5/10.7/hmseven/white}
  \node[draw=white,line width=0.35pt,fill=\shade,minimum width=1.40cm,
        minimum height=0.56cm,text=\ink,font=\small] at (\x,-8) {\value};

\node[anchor=east,font=\small] at (-0.12,-9) {DeepONet-C};
\foreach \x/\value/\shade/\ink in {
  0.5/1.36/hmtwo/black,
  1.5/1.24/hmtwo/black,
  2.5/8.13/hmsix/black,
  3.5/1.83/hmthree/black,
  4.5/4.08/hmfive/black}
  \node[draw=white,line width=0.35pt,fill=\shade,minimum width=1.40cm,
        minimum height=0.56cm,text=\ink,font=\small] at (\x,-9) {\value};

\node[anchor=east,font=\small] at (-0.12,-10) {MIONet};
\foreach \x/\value/\shade/\ink in {
  0.5/2.30/hmfour/black,
  1.5/3.34/hmfive/black,
  2.5/10.4/hmseven/white,
  3.5/3.11/hmfive/black,
  4.5/5.27/hmsix/black}
  \node[draw=white,line width=0.35pt,fill=\shade,minimum width=1.40cm,
        minimum height=0.56cm,text=\ink,font=\small] at (\x,-10) {\value};

\node[anchor=east,font=\small] at (-0.12,-11) {MNO};
\foreach \x/\value/\shade/\ink in {
  0.5/1.23/hmtwo/black,
  1.5/1.14/hmtwo/black,
  2.5/3.35/hmfive/black,
  3.5/1.46/hmtwo/black,
  4.5/5.11/hmsix/black}
  \node[draw=white,line width=0.35pt,fill=\shade,minimum width=1.40cm,
        minimum height=0.56cm,text=\ink,font=\small] at (\x,-11) {\value};
\end{tikzpicture}%
}
\caption{Summary of normalized performance in the product-space-learning
experiments. For each method \(m\), PDE \(p\), and associated test dataset
\(d\), we compute the normalized error
\(
\mathrm{Error}_{m,p,d}/\min_{m'}\mathrm{Error}_{m',p,d},
\)
where \(\mathrm{Error}_{m,p,d}\) denotes the mean relative error. Consequently,
the best-performing method on each test dataset attains the value \(1\). For
each PDE, the displayed value is obtained by averaging the normalized errors
over all six associated test datasets (one in-distribution and five
out-of-distribution datasets). Lighter colors indicate better performance.}
\label{fig:summary_heatmap_ps}
\end{figure}

\begin{figure}
    \centering
    \includegraphics[width=\textwidth]{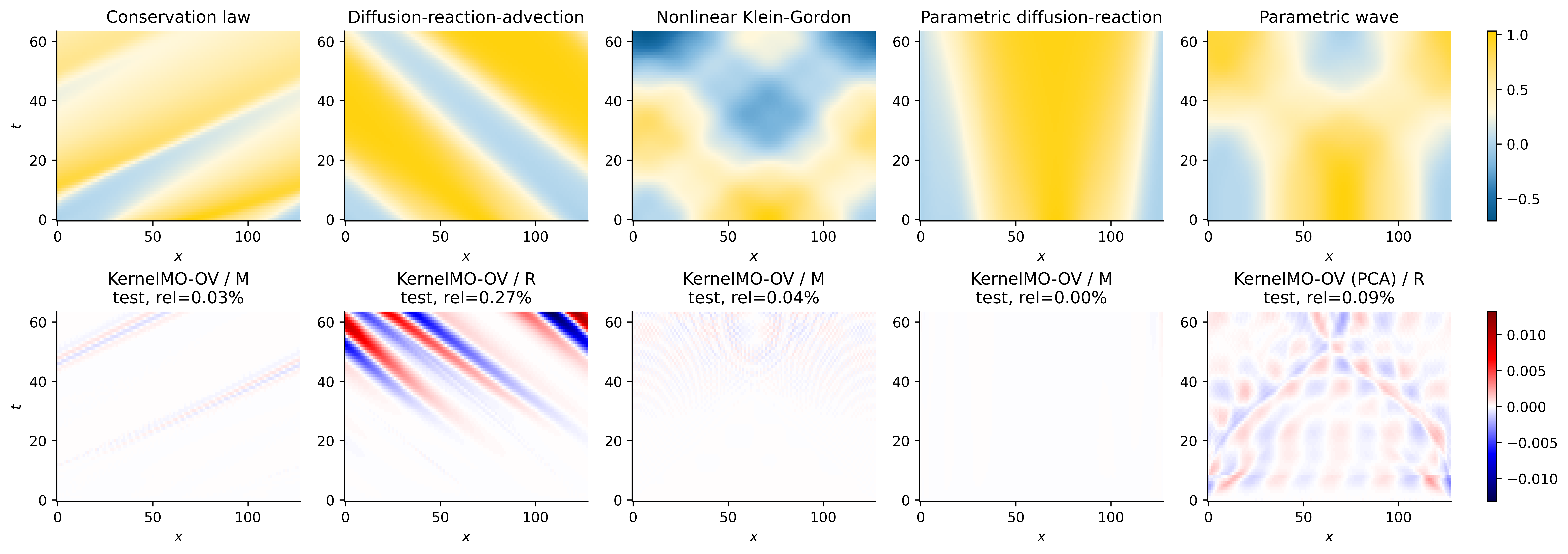}
    \caption{Qualitative summary across PDEs for operator-valued learning. Columns correspond to conservation law, diffusion-reaction-advection, nonlinear Klein-Gordon, parametric diffusion-reaction, parametric wave. The top row shows the reference solution for one selected trajectory from the test split; the bottom row shows the signed prediction error for the kernel method with the smallest mean relative error.}
    \label{fig:qualitative_grid_framework1_operator_test}
\end{figure}

\begin{figure}
    \centering
    \includegraphics[width=\textwidth]{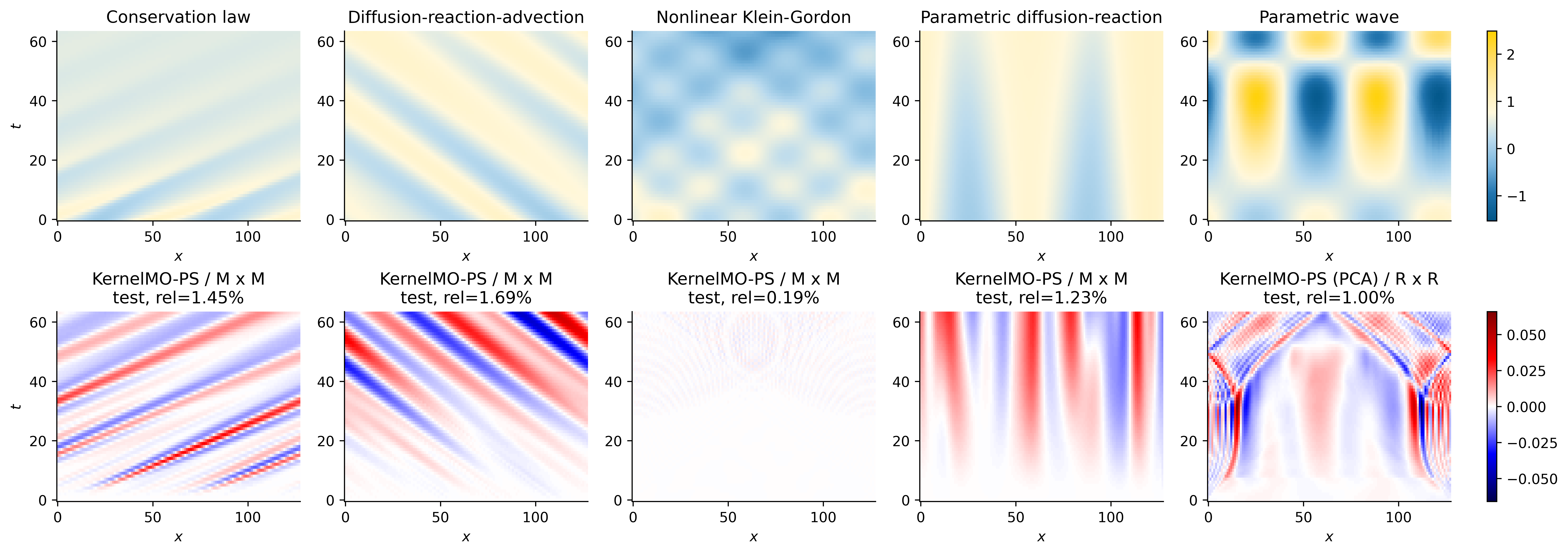}
    \caption{Qualitative summary across PDEs for product-space learning. Columns correspond to conservation law, diffusion-reaction-advection, nonlinear Klein-Gordon, parametric diffusion-reaction, parametric wave. The top row shows the reference solution for one selected trajectory from the test split; the bottom row shows the signed prediction error for the kernel method with the smallest mean relative error.}
    \label{fig:qualitative_grid_framework2_product_test}
\end{figure}

Across the five PDE benchmarks, the proposed kernel-based methods are consistently competitive with, and in most cases substantially outperform, the neural operator baselines, as illustrated by the predominance of light-colored cells for the proposed methods in Figures~\ref{fig:summary_heatmap_ov} and~\ref{fig:summary_heatmap_ps}. Their largest gains are generally observed on the in-distribution test sets, where the prediction error is often reduced by one or two orders of magnitude. For example, in operator-valued learning, on the conservation law the best kernel methods reduce the error from $1.23\%$ for MIONet to $0.01\%$, while on the nonlinear Klein--Gordon equation the error decreases from $4.63\%$ for MNO to $0.21\%$. The proposed methods also frequently attain the lowest out-of-distribution errors, although the relative improvement depends more strongly on the governing PDE and the type of distribution shift.

The operator-valued and product-space formulations exhibit complementary strengths. The \KernelMOOV methods are particularly well-suited to operator-valued learning, generally outperforming the \KernelMOPS variants and obtaining the lowest in-distribution errors on the conservation law ($0.01\%$), diffusion--reaction--advection equation ($0.38\%$), nonlinear Klein--Gordon equation ($0.21\%$), and parametric diffusion--reaction equation ($0.06\%$). Its PCA-based variants also frequently provide the best out-of-distribution performance. For product-space learning, \KernelMOPS consistently outperforms the corresponding single-kernel baseline KernelO (e.g. $2.13\%$ versus $59.28\%$ for the parametric wave equation) and remains competitive with, or superior to, the neural operator baselines across most benchmarks.

The comparison between the full and PCA-based variants reveals no universal ordering. PCA substantially reduces the dimensionality of the encoded inputs and outputs while often preserving most of the predictive accuracy and, in several experiments, improving out-of-distribution generalization. For example, on the nonlinear Klein--Gordon equation in operator-valued learning, the OOD error decreases from $6.72\%$ for \KernelMOOV\ / M to $4.11\%$ for \KernelMOOV\ (PCA) / M. Likewise, for the parametric wave equation in product-space learning, \KernelMOPS\ (PCA) / R $\times$ R substantially improves the most challenging distribution shifts compared with its non-PCA counterpart ($46.20\%$ versus $99.55\%$). On the other hand, PCA introduces a moderate loss of in-distribution accuracy on some benchmarks; for instance, in operator-valued learning, on the conservation law, the error increases from $0.01\%$ to $1.77\%$, while on the parametric diffusion--reaction equation it increases from $0.06\%$ to $3.04\%$. Thus, PCA should primarily be viewed as a computationally efficient representation that can additionally provide a regularization effect, rather than as a uniformly accuracy-improving transformation.

The choice of kernel is similarly problem dependent. Mat{\'e}rn kernels generally provide the strongest and most stable performance, particularly for the conservation law, nonlinear Klein--Gordon equation, and parametric diffusion--reaction equation. For example, in operator-valued learning, on the conservation law, \KernelMOOV\ / M achieves an OOD error of $0.77\%$, compared with $5.47\%$ for the RBF kernel. Likewise, on the parametric diffusion--reaction equation, \KernelMOOV\ / M attains OOD errors of $3.53\%$, $4.05\%$, and $0.12\%$ on the three distribution shifts, whereas \KernelMOOV\ / R reaches $56.22\%$, $83.90\%$, and $0.33\%$. RBF kernels can nevertheless attain extremely small in-distribution errors. The clearest example is the product-space parametric diffusion--reaction experiment, where \KernelMOPS\ / R $\times$ R achieves the lowest in-distribution error ($0.75\%$). 

Finally, the experiments reveal no universal out-of-distribution trend across all PDEs. Generalization depends strongly on both the governing equation and the nature of the distribution shift. Overall, OOD prediction is consistently more challenging than in-distribution prediction for all methods, highlighting the intrinsic difficulty of extrapolating beyond the training distribution. Changes in initial-condition amplitude are, for example, particularly challenging for the nonlinear Klein--Gordon equation in product-space learning, where all methods incur large OOD errors, although \KernelMOPS\ still achieves the lowest overall error. Across operator-valued and product-space learning, in the parametric wave equation, changing the Gaussian-process kernel (OOD\_matern) or its length scale (OOD\_scale) is also challenging. Across several benchmarks, we further note that PCA-based variants exhibit improved robustness under distribution shift, suggesting that the reduced latent representation can provide a useful regularization effect (e.g. from 3.37\% to 2.47\% for \KernelMOPS / M$\times$M in product-space learning for the nonlinear Klein Gordon equation).

\subsubsection{Conservation Laws}

We consider the following one-dimensional conservation law with periodic boundary conditions:
\begin{equation*}
\begin{aligned}
u_t + (\alpha_1u+\alpha_2u^2+\alpha_3u^3)_x & = \alpha_4u_{xx}, \quad (t,x)\in[0,2]\times[0,2] \\
u(0,x) &= u_0(x), \\
u(t,0) &= u(t,2),
\end{aligned}
\end{equation*}
where the components of parameter vector $\alpha = [\alpha_1,\alpha_2,\alpha_3,\alpha_4]^\top$ are sampled from the ranges $\alpha_i\in[0.9\alpha_i^c, \allowbreak 1.1\alpha_i^c]$, with the reference values given by $\alpha^c = [1,1,1,0.1]^\top$.

\paragraph{Operator-Valued Learning.}

The mean relative errors and standard deviations are summarized in Table~\ref{tab:framework1_cons}, while the implementation details and hyperparameter choices are summarized in Table~\ref{tab:conservationOV}. Overall, the proposed kernel-based methods achieve the lowest prediction errors on both the in-distribution and out-of-distribution datasets and the best performance is obtained with the Mat{\'e}rn kernel. Compared with the best-performing neural operator baselines, the proposed methods can improve the prediction accuracy by approximately two orders of magnitude on the in-distribution test set (from $1.23\%$ for MIONet to $0.01\%$ for \KernelMOOV / M and \KernelMOPS / M $\times$ M) and by a factor of approximately five on the out-of-distribution dataset (from $3.97\%$ for MNO to $0.77\%$ for \KernelMOOV / M and \KernelMOPS / M $\times$ M). Applying PCA substantially reduces the dimensionality of the learning problem while incurring a loss of accuracy, however, the PCA-based variants remain highly competitive. For example, \KernelMOOV (PCA) / M and \KernelMOPS (PCA) / M $\times$ M achieve prediction errors of $1.77\%$ on the in-distribution test set and $2.32\%$ on the out-of-distribution dataset, compared with $1.23\%$ for MIONet and $3.97\%$ for MNO, respectively.

\begin{table}[H]
\centering
\small
\begin{tabular}{lcc}
\hline
Method / kernel & Test & OOD \\
\hline
\textbf{\KernelMOOV / M} & {\bf 0.01\% (0.02\%)} & {\bf 0.77\% (1.54\%)} \\
\textbf{\KernelMOOV / R} & 0.02\% (0.02\%) & 5.47\% (6.40\%) \\
\textbf{\KernelMOOV (PCA)} / M & 1.77\% (0.43\%) & 2.32\% (1.21\%) \\
\textbf{\KernelMOOV (PCA)} / R & 1.77\% (0.43\%) & 2.82\% (1.65\%) \\
\textbf{\KernelMOPS / M x M} & {\bf 0.01\% (0.02\%)} & {\bf 0.77\% (1.54\%)} \\
\textbf{\KernelMOPS / R x R} & 1.46\% (1.09\%) & 4.63\% (5.02\%) \\
\textbf{\KernelMOPS (PCA) / M x M} & 1.77\% (0.43\%) & 2.32\% (1.21\%) \\
\textbf{\KernelMOPS (PCA) / R x R} & 2.34\% (0.90\%) & 5.15\% (4.74\%) \\\hline
DeepONet-C & 3.96\% (0.74\%) & 5.56\% (2.75\%) \\
MIONet & 1.23\% (0.28\%) & 5.86\% (7.64\%) \\
MNO & 1.84\% (0.34\%) & 3.97\% (4.09\%) \\\hline
\end{tabular}
\caption{Operator-valued learning: performance comparison on the conservation laws. We report mean relative errors with standard deviations in parentheses on in-distribution and out-of-distribution datasets. 
}
\label{tab:framework1_cons}
\end{table}

\begin{figure}[H]
    \centering
    \includegraphics[width=\textwidth]{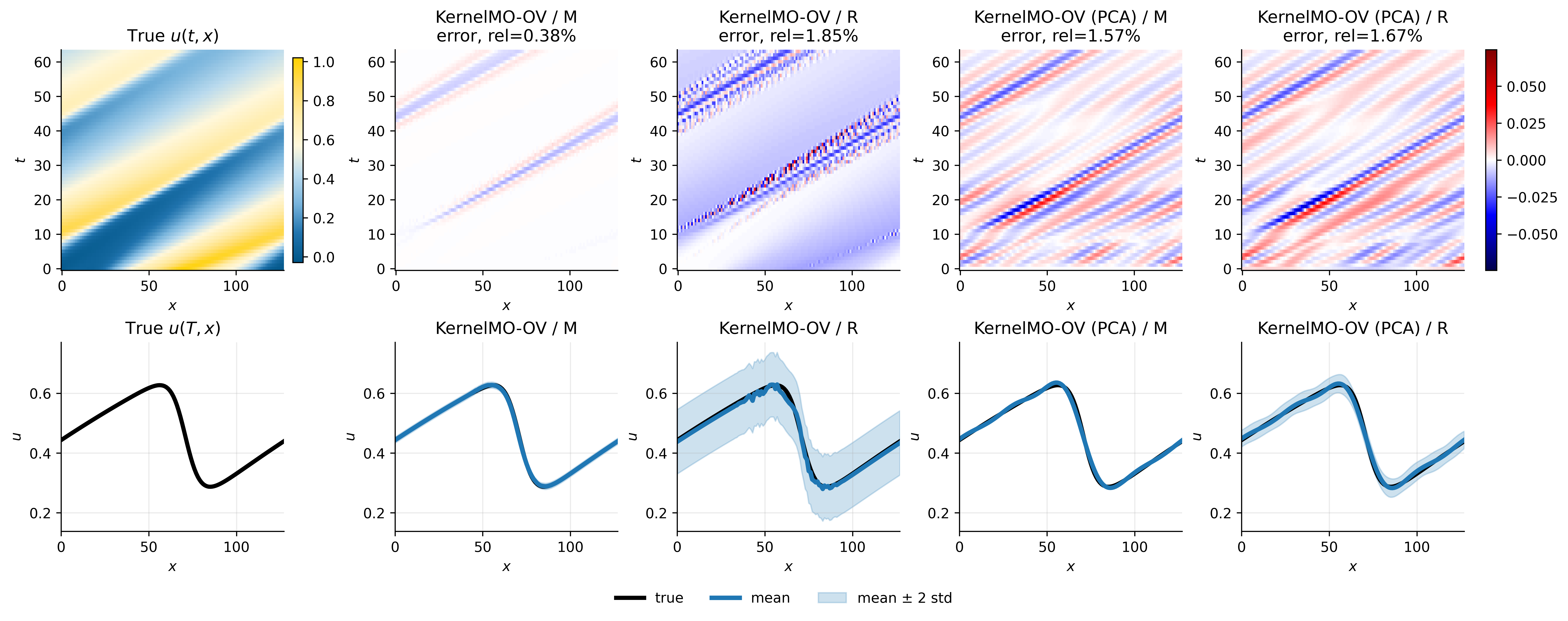}
    \caption{Operator-valued learning: qualitative prediction and uncertainty comparison for the conservation law on the OOD dataset. The first row shows a reference solution and signed prediction errors for four kernel variants; the relative error for the displayed trajectory is reported in each error-figure title. The second row shows the final-time solution, predictive mean, and an uncertainty band of $\pm 2$ predictive standard deviations.}
    \label{fig:qualitative_framework1_operator_conservation_ood_ood_h5}
\end{figure}

\paragraph{Product-Space Learning.}

The mean relative errors and standard deviations are summarized in Table~\ref{tab:framework2_cons}, while the implementation details and hyperparameter choices are summarized in Table~\ref{tab:conservationPS}. Overall, the proposed kernel-based methods achieve the lowest prediction errors on the in-distribution test set, with the best performance obtained by \KernelMOPS\ / M $\times$ M using the Mat{\'e}rn kernel. Compared with the best-performing neural operator baseline, the proposed method improves the prediction accuracy by approximately a factor of two on the in-distribution test set (from $3.29\%$ for MNO to $1.76\%$ for \KernelMOPS\ / M $\times$ M). Across the out-of-distribution datasets, different methods achieve the best performance under different distribution shifts. Nevertheless, the proposed product-kernel methods consistently rank among the top-performing approaches, demonstrating strong robustness and improved generalization ability. Applying PCA substantially reduces the dimensionality of the learning problem while incurring only a moderate loss of accuracy. For example, \KernelMOPS\ (PCA) / M $\times$ M achieves a prediction error of $2.96\%$ on the in-distribution test set while maintaining competitive performance across all out-of-distribution datasets.

\begin{table}[H]
\centering
\tiny
\begin{tabular}{lcccccc}
\hline
Method / kernel & Test & OOD\_par & OOD\_init\_GP & OOD\_init\_amp  & OOD\_par\_init\_GP & OOD\_par\_init\_amp \\
\hline
KernelO / M & 6.67\% (5.45\%) & {\bf 3.98\% (3.28\%)} & 14.18\% (8.47\%) & 46.93\% (20.19\%) & 17.41\% (9.52\%) & 47.18\% (19.88\%) \\
KernelO / R & 4.76\% (3.21\%) & 7.50\% (5.17\%) & 10.29\% (3.50\%) & 16.32\% (6.36\%) & 14.65\% (6.48\%) & 16.93\% (6.20\%) \\
KernelO (PCA) / M & 6.97\% (5.11\%) & 4.82\% (2.84\%) & 13.28\% (7.60\%) & 10.31\% (4.10\%) & 16.63\% (9.21\%) & 11.21\% (4.70\%) \\
KernelO (PCA) / R & 5.34\% (2.88\%) & 7.87\% (4.92\%) & 10.30\% (3.53\%) & 7.20\% (1.86\%) & 14.63\% (6.49\%) & 8.15\% (2.63\%) \\
\textbf{KernelMO-PS / M x M} & {\bf 1.76\% (1.52\%)} & 6.24\% (4.24\%) & 5.57\% (3.69\%) & {\bf 6.18\% (1.99\%)} & 10.80\% (7.57\%) & 9.05\% (3.49\%) \\
\textbf{KernelMO-PS / R x R} & 2.97\% (2.15\%) & 9.36\% (6.46\%) & 8.59\% (3.65\%) & 10.38\% (3.61\%) & 16.25\% (9.51\%) & 14.84\% (7.75\%) \\
\textbf{KernelMO-PS (PCA) / M x M} & 2.96\% (1.13\%) & 6.10\% (3.60\%) & 5.81\% (3.38\%) & 6.35\% (1.90\%) & 10.04\% (7.17\%) & 8.63\% (3.11\%) \\
\textbf{KernelMO-PS (PCA) / R x R} & 3.79\% (1.73\%) & 9.21\% (5.84\%) & 8.56\% (3.65\%) & 6.68\% (1.82\%) & 15.65\% (9.19\%) & 11.61\% (7.45\%) \\\hline
DeepONet & 6.77\% (2.80\%) & 8.49\% (4.16\%) & 10.55\% (3.62\%) & 8.36\% (2.29\%) & 14.70\% (6.46\%) & 9.23\% (3.03\%) \\
DeepONet-C & 4.42\% (0.92\%) & 4.88\% (1.22\%) & 6.22\% (1.68\%) & 7.66\% (2.11\%) & {\bf 6.92\% (2.61\%)} & 7.83\% (2.10\%) \\
MIONet & 4.32\% (1.31\%) & 8.90\% (9.79\%) & 6.37\% (2.83\%) & 21.98\% (11.22\%) & 10.74\% (11.55\%) & 21.09\% (10.83\%) \\  
MNO & 3.29\% (1.05\%) & 4.79\% (2.88\%) & {\bf 5.46\% (1.71\%)} & 7.04\% (2.11\%) & 8.25\% (4.96\%) & {\bf 7.49\% (2.30\%)} \\\hline
\end{tabular}
\caption{Product-space learning: performance comparison on the conservation law. We report mean relative errors with standard deviations in parentheses on in-distribution and out-of-distribution datasets.  
}
\label{tab:framework2_cons}
\end{table}

\begin{figure}[H]
    \centering
    \includegraphics[width=\textwidth]{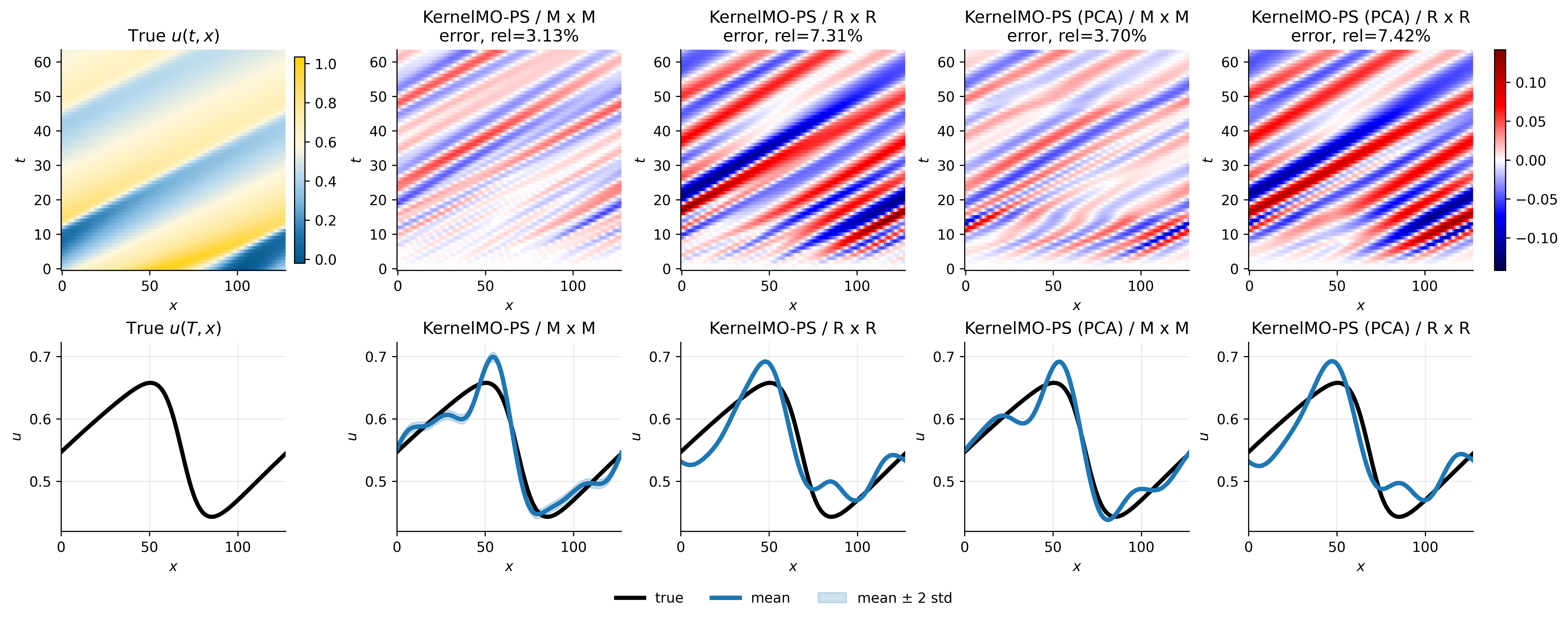}
    \caption{Product-space learning: qualitative prediction and uncertainty comparison for the conservation law on the OOD\_par dataset. The first row shows a reference solution and signed prediction errors for four kernel variants; the relative error for the displayed trajectory is reported in each error-figure title. The second row shows the final-time solution, predictive mean, and an uncertainty band of $\pm 2$ predictive standard deviations. }
    \label{fig:qualitative_framework2_product_conservation_ood_ood_par_h5}
\end{figure}

\subsubsection{Diffusion-Reaction-Advection}

We consider the following one-dimensional diffusion-reaction-advection equation:
\begin{equation*}
\begin{aligned}
u_t &= \alpha_1u_{xx} + \alpha_2u_x + \alpha_3u^{\alpha_4}(1-u^{\alpha_5}), \quad (t,x)\in[0,2]\times[0,2] \\
u(0,x) &= u_0(x), \\
u(t,0) &= u(t,2),
\end{aligned}
\end{equation*}
where the first three components of parameter vector $\alpha = [\alpha_1,\alpha_2,\alpha_3,\alpha_4, \alpha_5]^\top$ are sampled from the ranges $\alpha_i\in[0.9\alpha_i^c,1.1\alpha_i^c]$, with the reference values given by $\alpha^c = [0.01,1,1]^\top$, while $\alpha_4$ and $\alpha_5$ are drawn uniformly from $[1,3]$.

\paragraph{Operator-Valued Learning.}

The mean relative errors and standard deviations are summarized in Table~\ref{tab:framework1_DRA}, while the implementation details and hyperparameter choices are summarized in Table~\ref{tab:diffusionReactionAdvectionOV}. Overall, the proposed kernel-based methods achieve the lowest prediction errors on both the in-distribution and out-of-distribution datasets. The best in-distribution performance is obtained by \KernelMOOV\ / R, while the lowest out-of-distribution error is achieved by \KernelMOOV\ (PCA) / R. Compared with the best-performing neural operator baseline, the proposed methods improve the prediction accuracy by more than a factor of three on the in-distribution test set (from $1.21\%$ for MNO to $0.38\%$ for \KernelMOOV\ / R) and by nearly a factor of two on the out-of-distribution dataset (from $6.58\%$ for MNO to $3.53\%$ for \KernelMOOV\ (PCA) / R). Applying PCA substantially reduces the dimensionality of the learning problem while preserving the strong predictive performance of the proposed methods, and even slightly improves the out-of-distribution accuracy for the RBF kernel.

\begin{table}[H]
\centering
\small
\begin{tabular}{lcc}
\hline
Method / kernel & Test & OOD \\
\hline
\textbf{KernelMO-OV / M} & 0.60\% (0.51\%) & 5.18\% (4.80\%) \\
\textbf{KernelMO-OV / R} & {\bf 0.38\% (0.38\%)} & 3.72\% (3.94\%) \\
\textbf{KernelMO-OV (PCA) / M} & 0.88\% (0.46\%) & 5.17\% (4.63\%) \\
\textbf{KernelMO-OV (PCA) / R} & 0.74\% (0.34\%) & {\bf 3.53\% (3.50\%)} \\
\textbf{KernelMO-PS / M x M} & 0.79\% (0.66\%) & 5.33\% (4.59\%) \\
\textbf{KernelMO-PS / R x R} & 2.87\% (1.41\%) & 10.42\% (6.48\%) \\
\textbf{KernelMO-PS (PCA) / M x M} & 1.03\% (0.60\%) & 5.30\% (4.44\%) \\
\textbf{KernelMO-PS (PCA) / R x R} & 2.93\% (1.38\%) & 10.43\% (6.46\%) \\\hline
DeepONet-C & 2.89\% (0.63\%) & 6.61\% (3.59\%) \\
MIONet & 5.56\% (1.43\%) & 13.80\% (7.08\%) \\
MNO & 1.21\% (0.34\%) & 6.58\% (5.85\%) \\\hline
\end{tabular}
\caption{Operator-valued learning: performance comparison on the diffusion-reaction-advection equation. We report mean relative errors with standard deviations in parentheses on in-distribution and out-of-distribution datasets.
}
\label{tab:framework1_DRA}
\end{table}

\begin{figure}[H]
    \centering
    \includegraphics[width=\textwidth]{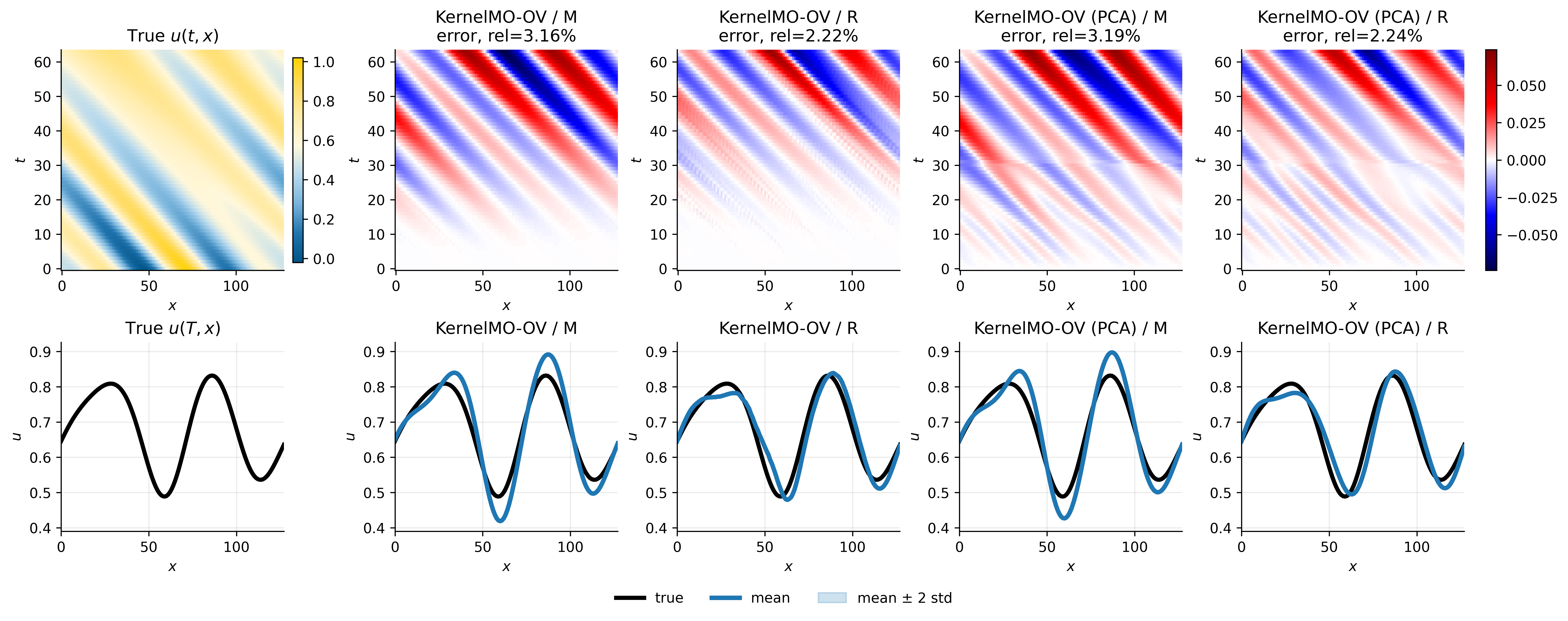}
    \caption{Operator-valued learning: qualitative prediction and uncertainty comparison for the diffusion-reaction-advection equation on the OOD dataset. The first row shows a reference solution and signed prediction errors for four kernel variants; the relative error for the displayed trajectory is reported in each error-figure title. The second row shows the final-time solution, predictive mean, and an uncertainty band of $\pm 2$ predictive standard deviations. }
    \label{fig:qualitative_framework1_operator_diffreacadv_ood_ood_h5}
\end{figure}

\paragraph{Product-Space Learning.}

The mean relative errors and standard deviations are summarized in Table~\ref{tab:framework2_DRA}, while the implementation details and hyperparameter choices are summarized in Table~\ref{tab:diffusionReactionAdvectionPS}. On the in-distribution test set, the proposed product-kernel methods achieve prediction errors comparable to the best-performing neural operator baseline, with \KernelMOPS\ / M $\times$ M attaining a mean relative error of $2.55\%$, compared with $2.25\%$ for MNO. On the out-of-distribution datasets, however, the proposed methods consistently outperform the neural operator baselines. In particular, \KernelMOPS\ (PCA) / M $\times$ M achieves the lowest prediction errors on four of the five out-of-distribution datasets, while \KernelMOPS\ / M $\times$ M achieves the best performance on the remaining one. Applying PCA substantially reduces the dimensionality of the learning problem while preserving the strong predictive performance of the proposed methods. Moreover, the proposed product-kernel methods consistently outperform the corresponding vanilla kernel methods, highlighting the advantage of explicitly modeling the product-space structure inherent in multiple operator learning problems.

\begin{table}[H]
\centering
\tiny
\begin{tabular}{lcccccc}
\hline
Method / kernel & Test & OOD\_par & OOD\_init\_GP & OOD\_init\_amp & OOD\_par\_init\_GP & OOD\_par\_init\_amp \\
\hline
KernelO / M & 18.81\% (10.19\%) & 8.27\% (4.83\%) & 23.00\% (14.71\%) & 48.20\% (19.41\%) &  28.09\% (16.11\%) & 48.77\% (18.81\%) \\
KernelO / R & 12.63\% (5.46\%) & 15.83\% (7.18\%) & 233.85\% (153.11\%) & 2589.43\% (1249.39\%) &  233.85\% (155.23\%) & 2590.93\% (1250.92\%) \\
KernelO (PCA) / M & 18.80\% (10.23\%) & 8.30\% (4.80\%) & 22.46\% (15.23\%) & 15.59\% (6.05\%) & 27.65\% (16.62\%) & 17.48\% (6.71\%) \\
KernelO (PCA) / R & 12.64\% (5.45\%) & 15.84\% (7.17\%) & 115.56\% (122.83\%) & 347.07\% (255.06\%) & 116.17\% (113.67\%) & 347.40\% (255.15\%) \\
\textbf{KernelMO-PS / M x M} & 2.55\% (1.54\%) & {\bf 4.57\% (3.74\%)} & 3.16\% (1.50\%) & 7.29\% (2.74\%) & 6.16\% (3.76\%) & 8.16\% (2.84\%) \\
\textbf{KernelMO-PS / R x R} & 2.77\% (2.54\%) & 5.92\% (5.13\%) & 3.88\% (2.07\%) & 7.28\% (2.69\%) & 9.37\% (6.11\%) & 9.92\% (5.75\%) \\
\textbf{KernelMO-PS (PCA) / M x M} & 2.61\% (1.50\%) & 4.61\% (3.67\%) & {\bf 3.15\% (1.50\%)} & {\bf 6.86\% (2.44\%)} & {\bf 6.15\% (3.76\%)} & {\bf 7.77\% (2.70\%)} \\
\textbf{KernelMO-PS (PCA) / R x R} & 2.83\% (2.44\%) & 5.88\% (5.02\%) & 3.85\% (2.05\%) & 7.15\% (2.64\%) & 9.25\% (6.03\%) & 9.74\% (5.67\%) \\\hline
DeepONet & 13.14\% (5.11\%) & 15.98\% (6.83\%) & 15.33\% (6.18\%) & 13.76\% (4.52\%) & 21.37\% (9.38\%) & 15.40\% (5.30\%) \\
DeepONet-C & 3.46\% (0.73\%) & 5.91\% (3.32\%) & 3.65\% (0.96\%) & 8.37\% (2.85\%) & 6.34\% (3.97\%) & 9.10\% (2.90\%) \\
MIONet & 5.71\% (1.35\%) & 16.22\% (13.97\%) & 9.66\% (3.88\%) & 30.24\% (14.11\%) & 17.47\% (13.65\%) & 28.18\% (13.38\%) \\
MNO & {\bf 2.25\% (0.74\%)} & 5.62\% (4.48\%) & 3.33\% (1.32\%) & 8.16\% (2.90\%) & 7.15\% (5.14\%) & 9.27\% (3.23\%) \\\hline
\end{tabular}
\caption{Product-space learning: performance comparison on the diffusion-reaction-advection equation. We report mean relative errors with standard deviations in parentheses on in-distribution and out-of-distribution datasets.
}
\label{tab:framework2_DRA}
\end{table}

\begin{figure}[H]
    \centering
    \includegraphics[width=\textwidth]{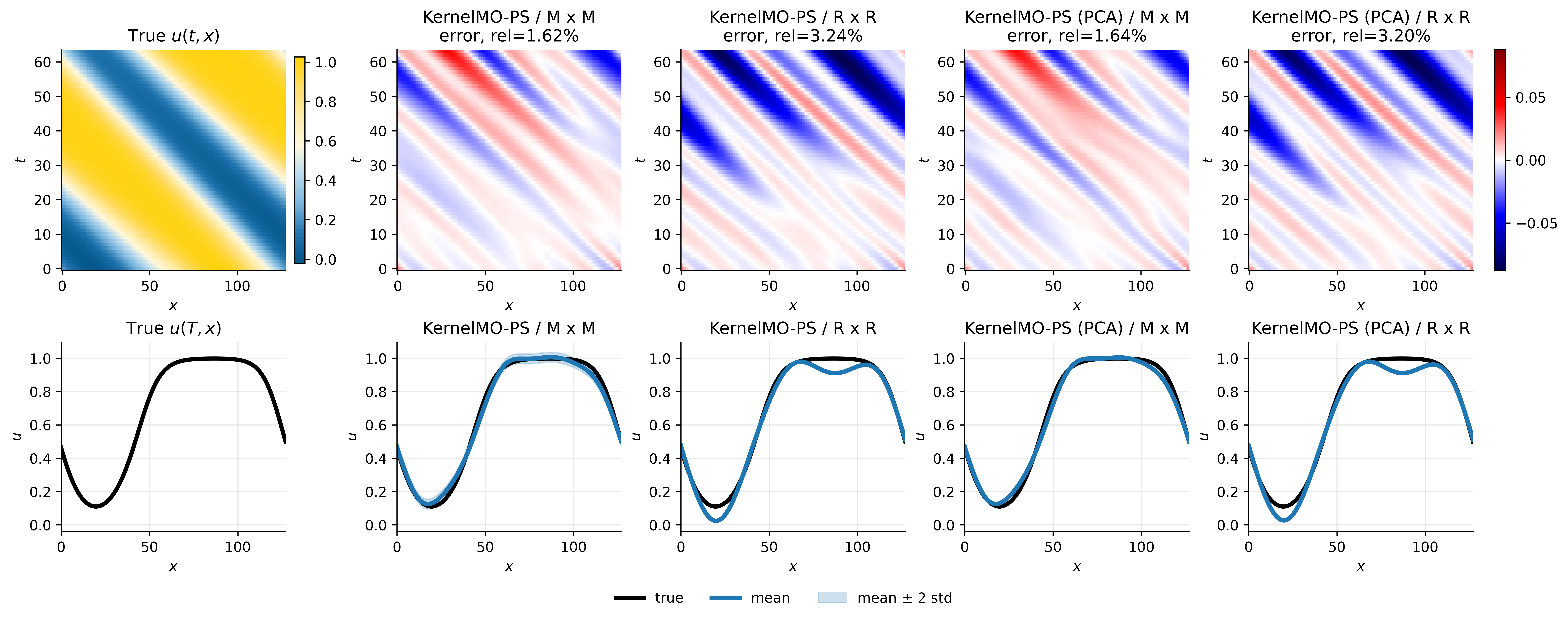}
    \caption{Product-space learning: qualitative prediction and uncertainty comparison for the diffusion-reaction-advection equation on the OOD\_init\_GP dataset. The first row shows a reference solution and signed prediction errors for four kernel variants; the relative error for the displayed trajectory is reported in each error-figure title. The second row shows the final-time solution, predictive mean, and an uncertainty band of $\pm 2$ predictive standard deviations. }
    \label{fig:qualitative_framework2_product_diffreacadv_ood_ood_init_gp_h5}
\end{figure}

\subsubsection{Nonlinear Klein-Gordon}

We consider the following nonlinear Klein--Gordon equation:
\begin{equation*}
\begin{aligned}
u_{tt} &= \alpha_1^2u_{xx} - \alpha_2^2\alpha_1^4u - \alpha_3u^3, \quad (t,x)\in[0,2]\times[0,2] \\
u(0,x) &= u_0(x), \\
u_t(0,x) &= 0, \\
u(t,0) &= u(t,2).
\end{aligned}
\end{equation*}
The components of the parameter vector $\alpha = [\alpha_1,\alpha_2,\alpha_3]^\top$ are sampled from the ranges $\alpha_i\in[0.9\alpha_i^c,1.1\alpha_i^c]$ with reference values $\alpha^c=[1,1,1]^\top$.

\paragraph{Operator-Valued Learning.}

The mean relative errors and standard deviations are summarized in Table~\ref{tab:framework1_KG}, while the implementation details and hyperparameter choices are summarized in Table~\ref{tab:kleinGordonOV}. Overall, the proposed kernel-based methods achieve the lowest prediction errors on both the in-distribution and out-of-distribution datasets. The best in-distribution performance is obtained by \KernelMOOV\ / M and \KernelMOPS\ / M $\times$ M, while the lowest out-of-distribution error is achieved by their PCA-based variants. Compared with the best-performing neural operator baseline, the proposed methods improve the prediction accuracy by more than one order of magnitude on the in-distribution test set (from $4.63\%$ for MNO to $0.21\%$ for \KernelMOOV\ / M and \KernelMOPS\ / M $\times$ M) and by more than a factor of four on the out-of-distribution dataset (from $18.50\%$ for MNO to $4.11\%$ for \KernelMOOV\ (PCA) / M and \KernelMOPS\ (PCA) / M $\times$ M). Applying PCA substantially reduces the dimensionality of the learning problem while further improving the out-of-distribution performance. 

\begin{table}[H]
\centering
\small
\begin{tabular}{lcc}
\hline
Method / kernel & Test & OOD \\
\hline
\textbf{KernelMO-OV / M} & {\bf 0.21\% (0.19\%)} & 6.72\% (9.13\%) \\
\textbf{KernelMO-OV / R} & 0.37\% (0.24\%) & 11.17\% (24.33\%) \\
\textbf{KernelMO-OV (PCA) / M} & 0.62\% (0.28\%) & {\bf 4.11\% (6.96\%)} \\
\textbf{KernelMO-OV (PCA) / R} & 0.68\% (0.33\%) & 10.39\% (23.80\%) \\
\textbf{KernelMO-PS / M x M} & {\bf 0.21\% (0.19\%)} & 6.72\% (9.13\%) \\
\textbf{KernelMO-PS / R x R} & 10.22\% (6.85\%) & 27.57\% (27.58\%) \\
\textbf{KernelMO-PS (PCA) / M x M} & 0.62\% (0.28\%) & {\bf 4.11\% (6.96\%)} \\
\textbf{KernelMO-PS (PCA) / R x R} & 10.23\% (6.84\%) & 27.57\% (27.58\%) \\\hline
DeepONet-C & 12.17\% (3.64\%) & 22.06\% (14.84\%) \\
MIONet & 8.41\% (2.83\%) & 28.46\% (31.09\%) \\
MNO & 4.63\% (1.46\%) & 18.50\% (19.36\%) \\\hline
\end{tabular}
\caption{Operator-valued learning: performance comparison on the nonlinear Klein-Gordon equation. We report mean relative errors with standard deviations in parentheses on in-distribution and out-of-distribution datasets.
}
\label{tab:framework1_KG}
\end{table}

\begin{figure}[H]
    \centering
    \includegraphics[width=\textwidth]{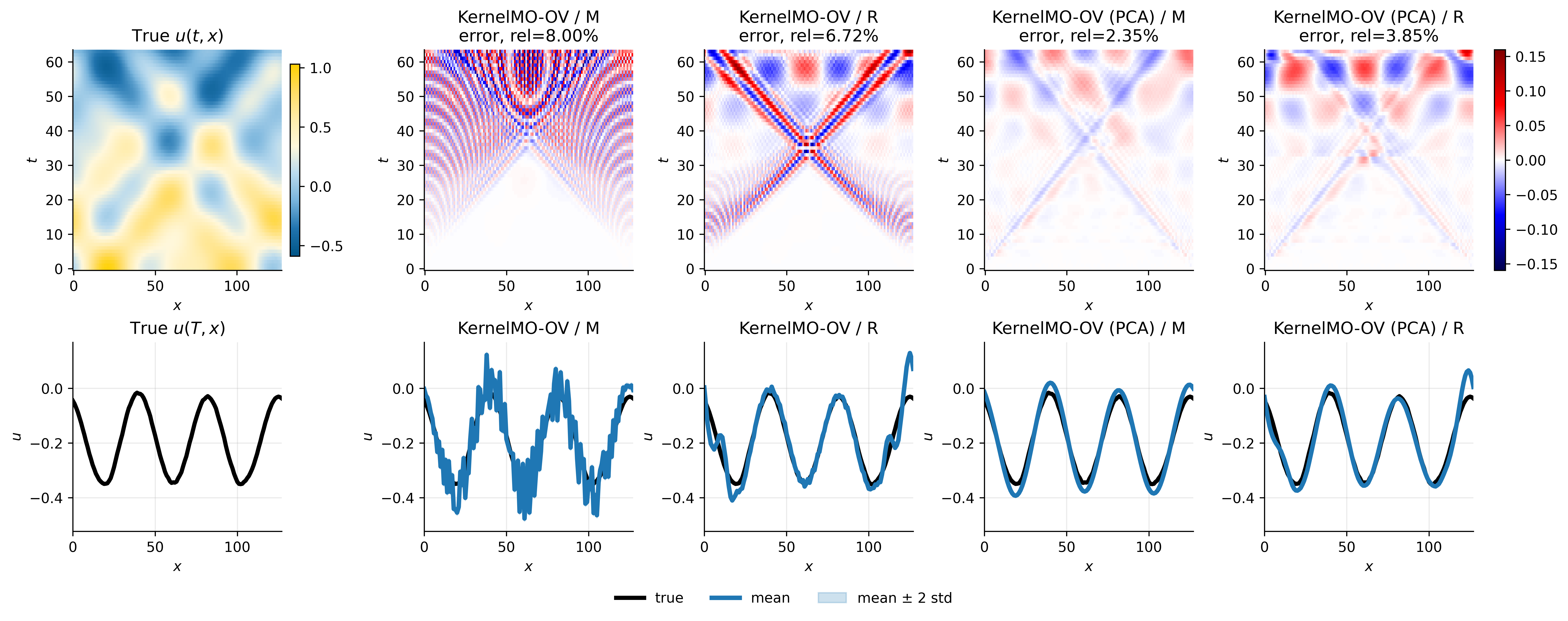}
    \caption{Operator-valued learning: qualitative prediction and uncertainty comparison for the Nonlinear Klein-Gordon equation on the OOD dataset. The first row shows a reference solution and signed prediction errors for four kernel variants; the relative error for the displayed trajectory is reported in each error-figure title. The second row shows the final-time solution, predictive mean, and an uncertainty band of $\pm 2$ predictive standard deviations. }
    \label{fig:qualitative_framework1_operator_klein_gordon_ood_ood_h5}
\end{figure}

\paragraph{Product-Space Learning.}

The mean relative errors and standard deviations are summarized in Table~\ref{tab:framework2_KG}, while the implementation details and hyperparameter choices are summarized in Table~\ref{tab:kleinGordonPS}. Overall, the proposed product-kernel methods achieve the lowest prediction errors on the in-distribution test set and consistently outperform the neural operator baselines across the out-of-distribution datasets. In particular, \KernelMOPS\ / M $\times$ M attains the lowest in-distribution test error of $0.31\%$, compared with $2.84\%$ for the best-performing neural operator baseline MNO, corresponding to an improvement of approximately one order of magnitude. Across the out-of-distribution datasets, \KernelMOPS\ (PCA) / M $\times$ M achieves the best performance on four of the five datasets, while \KernelMOPS\ / M $\times$ M performs best on the remaining one. Applying PCA substantially reduces the dimensionality of the learning problem while preserving the strong predictive performance of the proposed methods. Finally, we observe that the datasets involving out-of-distribution initial-condition amplitudes (OOD\_init\_amp and OOD\_par\_init\_amp) are substantially more challenging for all methods, as the unseen initial conditions contain higher-frequency modes that are absent from the training data, leading to more oscillatory solution dynamics and larger prediction errors.

\begin{table}[H]
\centering
\tiny
\begin{tabular}{lcccccc}
\hline
Method / kernel & Test & OOD\_par & OOD\_init\_GP & OOD\_init\_amp & OOD\_par\_init\_GP & OOD\_par\_init\_amp \\
\hline
KernelO / M & 41.28\% (27.25\%) & 28.12\% (15.75\%) & 24.95\% (14.78\%) & 59.59\% (16.29\%) & 35.95\% (20.22\%) & 65.80\% (14.20\%) \\
KernelO / R & 28.94\% (14.46\%) & 42.95\% (19.50\%) & 434.20\% (297.77\%) & 5811.74\% (2849.27\%) & 431.33\% (289.06\%) & 5672.18\% (2787.12\%) \\
KernelO (PCA) / M & 41.25\% (27.26\%) & 28.12\% (15.74\%) & 24.64\% (14.83\%) & 46.81\% (12.42\%) & 35.78\% (20.42\%) & 57.36\% (14.91\%) \\
KernelO (PCA) / R & 28.94\% (14.46\%) & 42.95\% (19.50\%) & 214.07\% (218.27\%) & 766.86\% (609.60\%) & 213.33\% (202.18\%) & 750.18\% (594.53\%) \\
\textbf{KernelMO-PS / M x M} & {\bf 0.31\% (0.24\%)} & 8.36\% (12.27\%) & 2.10\% (1.33\%) & {\bf 31.79\% (12.65\%)} & 3.37\% (2.18\%) & 32.90\% (11.56\%) \\
\textbf{KernelMO-PS / R x R} & 0.90\% (0.66\%) & 15.81\% (25.30\%) & 7.79\% (5.67\%) & 90.45\% (40.93\%) & 9.13\% (5.77\%) & 90.30\% (38.01\%) \\
\textbf{KernelMO-PS (PCA) / M x M} & 0.68\% (0.35\%) & {\bf 7.18\% (11.64\%)} & {\bf 2.02\% (1.27\%)} & 31.94\% (12.74\%) & {\bf 2.47\% (1.68\%)} & {\bf 32.75\% (11.78\%)} \\
\textbf{KernelMO-PS (PCA) / R x R} & 1.01\% (0.66\%) & 14.86\% (24.83\%) & 4.22\% (4.19\%) & 36.08\% (14.83\%) & 5.86\% (4.84\%) & 40.08\% (16.84\%) \\\hline
DeepONet & 28.58\% (12.44\%) & 45.35\% (19.29\%) & 18.55\% (8.38\%) & 42.43\% (9.25\%) & 31.65\% (16.22\%) & 53.09\% (12.07\%) \\
DeepONet-C & 10.25\% (3.58\%) & 23.42\% (17.12\%) & 8.23\% (2.08\%) & 36.52\% (12.09\%) & 14.55\% (10.46\%) & 43.00\% (12.97\%) \\
MIONet & 11.19\% (2.60\%) & 35.97\% (35.71\%) & 13.26\% (9.19\%) & 72.59\% (32.71\%) & 26.00\% (26.60\%) & 66.93\% (29.43\%) \\
MNO & 2.84\% (0.70\%) & 18.94\% (20.61\%) & 3.98\% (1.69\%) & 32.50\% (12.60\%) & 10.29\% (9.00\%) & 37.93\% (12.55\%) \\\hline
\end{tabular}
\caption{Product-space learning: performance comparison on the nonlinear Klein-Gordon equation. We report mean relative errors with standard deviations in parentheses on in-distribution and out-of-distribution datasets.
}
\label{tab:framework2_KG}
\end{table}

\begin{figure}[H]
    \centering
    \includegraphics[width=\textwidth]{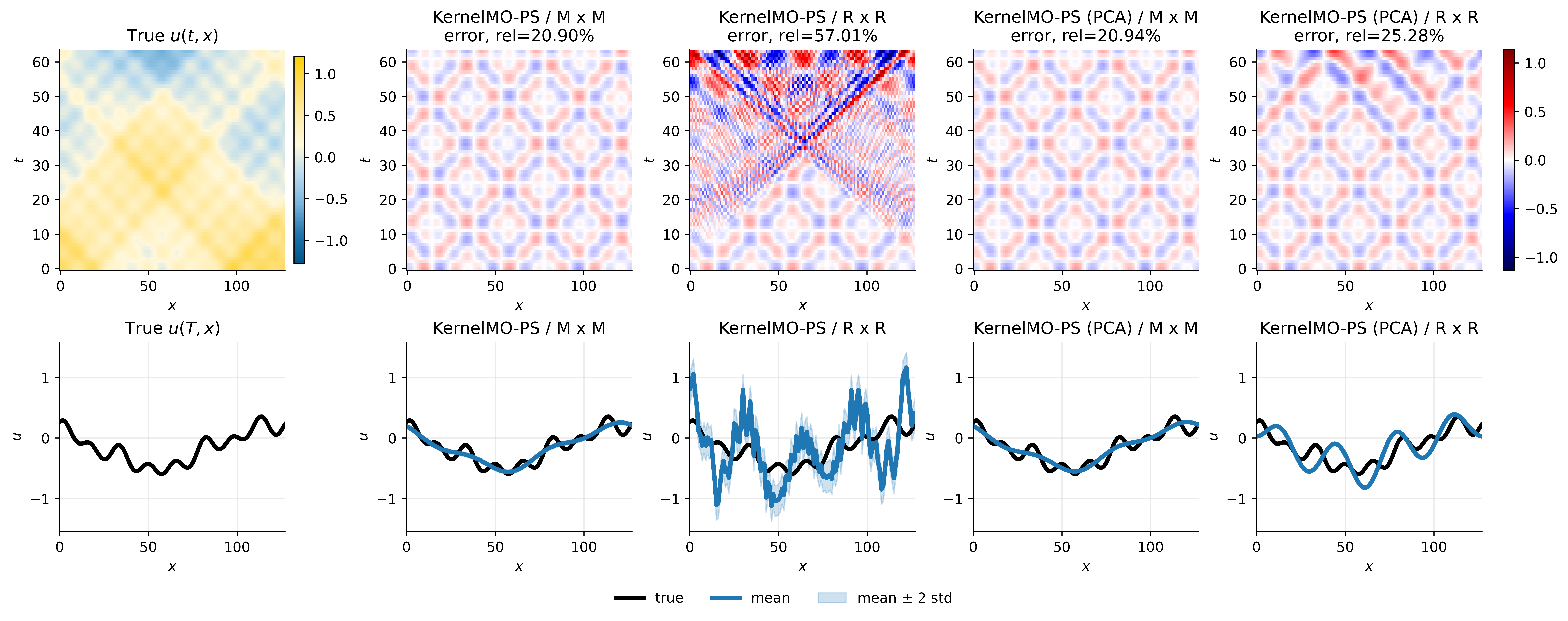}
    \caption{Product-space learning: qualitative prediction and uncertainty comparison for the nonlinear Klein-Gordon equation on the OOD\_par\_init\_amp dataset. The first row shows a reference solution and signed prediction errors for four kernel variants; the relative error for the displayed trajectory is reported in each error-figure title. The second row shows the final-time solution, predictive mean, and an uncertainty band of $\pm 2$ predictive standard deviations. }
    \label{fig:qualitative_framework2_product_klein_gordon_ood_ood_par_init_amp_h5}
\end{figure}

\subsubsection{Parametric Diffusion-Reaction Equation}

We consider the following parametric diffusion-reaction equation:
\begin{equation*}
\begin{aligned}
u_t &= (\alpha(x)u_x)_x + u(1-u), \quad (t,x)\in[0,2]\times[0,2] \\
u(0,x) &= u_0(x), \\
u(t,0) &= u(t,2),
\end{aligned}
\end{equation*}
where the spatially varying diffusivity $\alpha(x)$ is sampled from a Gaussian random process with RBF kernel (length scale=1) and with variance $0.01^2$.
The parametric function $\alpha(x)$ is evaluated at 129 sensor points corresponding to the boundaries of uniformly spaced cells, $\{x_i^b\}_{i=1}^{129}$, and the resulting values $\alpha(x_i^b)$ are encoded within the operator-valued learning and product-space learning methods. 

\paragraph{Operator-Valued Learning.}

The mean relative errors and standard deviations are summarized in Table~\ref{tab:framework1_parDR}, while the implementation details and hyperparameter choices are summarized in Table~\ref{tab:parametricDiffusionReactionOV}. Overall, the proposed kernel-based methods achieve the lowest prediction errors on both the in-distribution and out-of-distribution datasets. The best in-distribution performance is jointly attained by \KernelMOOV\ / M and \KernelMOPS\ / M $\times$ M, while \KernelMOOV\ / M achieves the lowest prediction errors on the OOD\_matern and OOD\_var datasets and \KernelMOOV\ (PCA) / M achieves the best performance on OOD\_scale. Compared with the best-performing neural operator baseline, the proposed methods improve the prediction accuracy by more than one order of magnitude on the in-distribution test set (from $1.34\%$ for MNO to $0.06\%$ for \KernelMOOV\ / M and \KernelMOOV\ / M $\times$ M). Applying PCA substantially reduces the dimensionality of the learning problem while preserving competitive predictive performance, particularly under distribution shifts in the coefficient functions. Overall, the operator-valued formulation demonstrates stronger out-of-distribution generalization than the product-space formulation on this benchmark.

\begin{table}[H]
\centering
\small
\begin{tabular}{lcccc}
\hline
Method / kernel & Test & OOD\_matern & OOD\_scale & OOD\_var \\
\hline
\textbf{KernelMO-OV / M} & {\bf 0.06\% (0.15\%)} & {\bf 3.53\% (1.56\%)} & 4.13\% (1.67\%) & {\bf 0.12\% (0.37\%)} \\
\textbf{KernelMO-OV / R} & 0.12\% (0.34\%) & 56.22\% (48.91\%) & 83.90\% (72.76\%) & 0.33\% (1.08\%) \\
\textbf{KernelMO-OV (PCA) / M} & 3.04\% (1.90\%) & 3.77\% (1.67\%) & {\bf 4.05\% (1.59\%)} & 3.19\% (1.92\%) \\
\textbf{KernelMO-OV (PCA) / R} & 3.04\% (1.90\%) & 7.25\% (6.29\%) & 10.34\% (10.19\%) & 3.19\% (1.92\%) \\
\textbf{KernelMO-PS / M x M} & {\bf 0.06\% (0.15\%)} & 10.12\% (8.30\%) & 14.80\% (11.56\%) & 0.14\% (0.42\%) \\
\textbf{KernelMO-PS / R x R} & 2.48\% (1.07\%) & 6.48\% (3.22\%) & 8.19\% (3.63\%) & 2.60\% (1.12\%) \\
\textbf{KernelMO-PS (PCA) / M x M} & 3.04\% (1.90\%) & 3.81\% (1.69\%) & 4.12\% (1.62\%) & 3.19\% (1.92\%) \\
\textbf{KernelMO-PS (PCA) / R x R} & 3.58\% (1.82\%) & 4.44\% (1.60\%) & 4.91\% (1.52\%) & 3.73\% (1.86\%) \\\hline
DeepONet-C & 2.22\% (0.88\%) & 4.19\% (1.41\%) & 4.68\% (1.44\%) & 2.34\% (0.99\%) \\
MIONet & 1.88\% (0.78\%) & 4.98\% (1.69\%) & 5.42\% (1.63\%) & 2.02\% (0.89\%) \\
MNO & 1.34\% (0.50\%) & 4.18\% (1.62\%) & 4.68\% (1.52\%) & 1.45\% (0.62\%) \\\hline
\end{tabular}
\caption{Operator-valued learning: performance comparison on the parametric diffusion-reaction equation. We report mean relative errors with standard deviations in parentheses on in-distribution and out-of-distribution datasets.
}
\label{tab:framework1_parDR}
\end{table}

\begin{figure}[H]
    \centering
    \includegraphics[width=\textwidth]{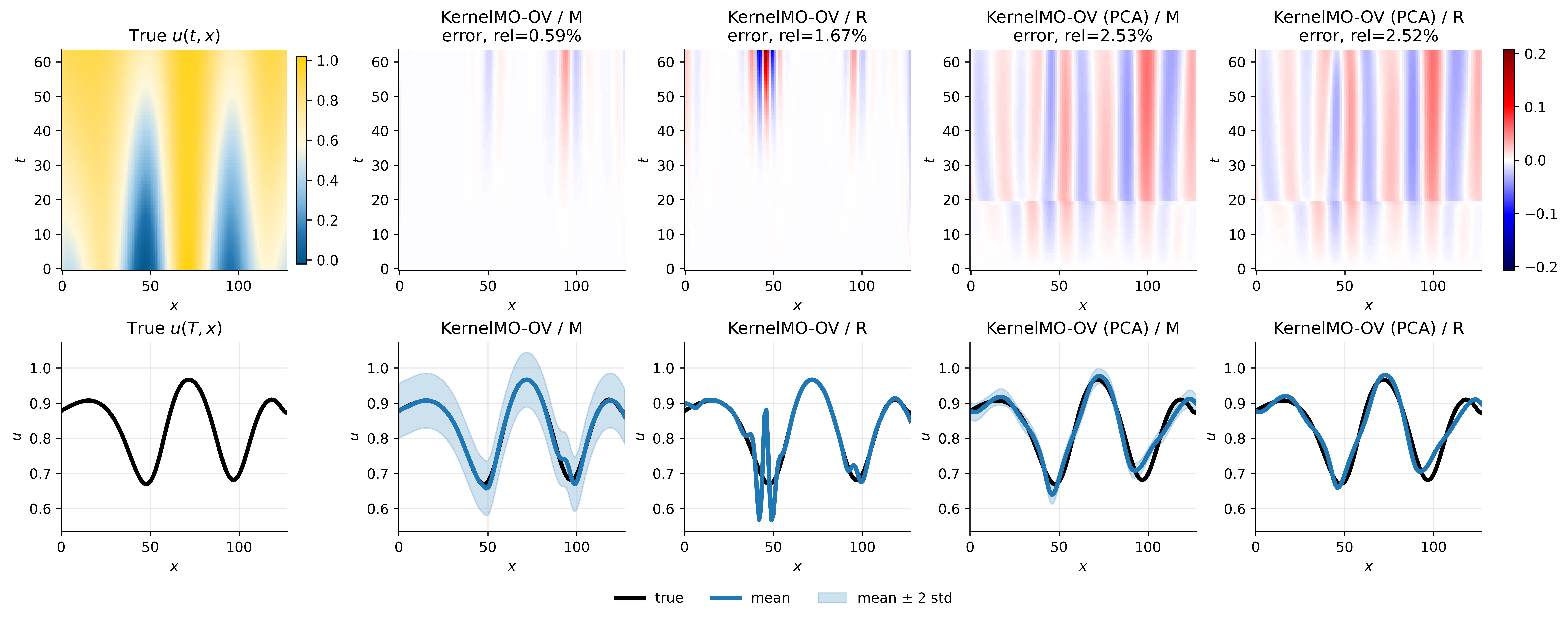}
    \caption{Operator-valued learning: qualitative prediction and uncertainty comparison for the parametric diffusion-reaction equation on the OOD\_var dataset. The first row shows a reference solution and signed prediction errors for four kernel variants; the relative error for the displayed trajectory is reported in each error-figure title. The second row shows the final-time solution, predictive mean, and an uncertainty band of $\pm 2$ predictive standard deviations. }
    \label{fig:qualitative_framework1_operator_param_diffreac_ood_ood_var_h5}
\end{figure}

\paragraph{Product-Space Learning.}

The mean relative errors and standard deviations are summarized in Table~\ref{tab:framework2_parDR}, while the implementation details and hyperparameter choices are summarized in Table~\ref{tab:parametricDiffusionReactionPS}. Overall, the proposed product-kernel methods achieve prediction errors comparable to those of the best-performing neural operator baseline across both the in-distribution and out-of-distribution datasets. In particular, \KernelMOPS\ / R $\times$ R attains the lowest in-distribution test error of $0.75\%$, compared with $1.77\%$ for MNO. However, this variant exhibits poor generalization under most distribution shifts, with substantially larger errors on the OOD\_init, OOD\_par\_init, OOD\_par\_kernel, and OOD\_par\_scale datasets. By contrast, \KernelMOPS\ / M $\times$ M achieves consistently strong performance across all evaluation settings, obtaining prediction errors comparable to MNO on every out-of-distribution dataset and outperforming it on the OOD\_par\_var dataset ($1.30\%$ versus $1.73\%$). These results highlight the superior robustness of the Mat{\'e}rn product-kernel method under distribution shifts, even though the RBF variant achieves the lowest in-distribution error.

\begin{table}[H]
\centering
\tiny
\begin{tabular}{lcccccc}
\hline
Method / kernel & Test & OOD\_init & OOD\_par\_init & OOD\_par\_kernel & OOD\_par\_scale & OOD\_par\_var \\
\hline
KernelO / M & 9.21\% (5.27\%) & 12.48\% (3.52\%) & 12.51\% (3.47\%) & 8.20\% (3.50\%) & 8.01\% (3.47\%) & 8.64\% (5.11\%) \\
KernelO / R & 6.76\% (2.54\%) & 18.96\% (7.00\%) & 19.02\% (6.95\%) & 4.98\% (1.78\%) & 5.18\% (1.67\%) & 6.87\% (2.64\%) \\
KernelO (PCA) / M & 8.96\% (4.72\%) & 15.70\% (5.20\%) & 15.74\% (5.14\%) & 7.77\% (3.12\%) & 7.62\% (3.08\%) & 8.62\% (4.53\%) \\
KernelO (PCA) / R & 6.97\% (2.74\%) & {\bf 10.76\% (3.62\%)} & {\bf 10.78\% (3.57\%)} & \textbf{5.11\% (1.91\%)} & \textbf{5.30\% (1.83\%)} & 7.09\% (2.84\%) \\
\textbf{KernelMO-PS / M x M} & 1.50\% (0.75\%) & 11.88\% (3.53\%) & 12.73\% (3.43\%) & 5.40\% (1.98\%) & 7.35\% (2.57\%) & 1.30\% (0.70\%) \\
\textbf{KernelMO-PS / R x R} & {\bf 0.75\% (0.44\%)} & 127.42\% (47.76\%) & 188.16\% (88.56\%) & 105.53\% (50.74\%) & 157.30\% (77.91\%) & {\bf 0.76\% (0.45\%)} \\
\textbf{KernelMO-PS (PCA) / M x M} & 3.94\% (2.01\%) & 104.63\% (68.48\%) & 150.65\% (81.94\%) & 104.76\% (68.56\%) & 154.17\% (91.24\%) & 4.04\% (2.04\%) \\
\textbf{KernelMO-PS (PCA) / R x R} & 4.27\% (1.96\%) & 156.18\% (102.47\%) & 243.60\% (144.65\%) & 155.33\% (109.85\%) & 235.54\% (162.35\%) & 4.39\% (1.96\%) \\\hline
DeepONet & 7.75\% (2.76\%) & 11.56\% (3.38\%) & 11.56\% (3.36\%) & 6.27\% (1.95\%) & 6.37\% (1.85\%) & 7.86\% (2.91\%) \\
DeepONet-C & 2.48\% (0.78\%) & 11.03\% (3.73\%) & 11.66\% (3.67\%) & 5.02\% (1.51\%) & 5.52\% (1.58\%) & 2.52\% (0.76\%) \\
MIONet & 3.49\% (1.00\%) & 36.36\% (16.85\%) & 37.00\% (16.77\%) & 5.43\% (1.74\%) & 6.60\% (2.50\%) & 3.51\% (0.95\%) \\
MNO & 1.77\% (0.66\%) & 11.26\% (3.77\%) & 11.43\% (3.66\%) & {\bf 4.42\% (1.82\%)} & {\bf 5.00\% (1.82\%)} & 1.73\% (0.61\%) \\\hline
\end{tabular}
\caption{Product-space learning: performance comparison on parametric diffusion-reaction equation. We report mean relative errors with standard deviations in parentheses on in-distribution and out-of-distribution datasets.
}
\label{tab:framework2_parDR}
\end{table}

\begin{figure}[H]
    \centering
    \includegraphics[width=\textwidth]{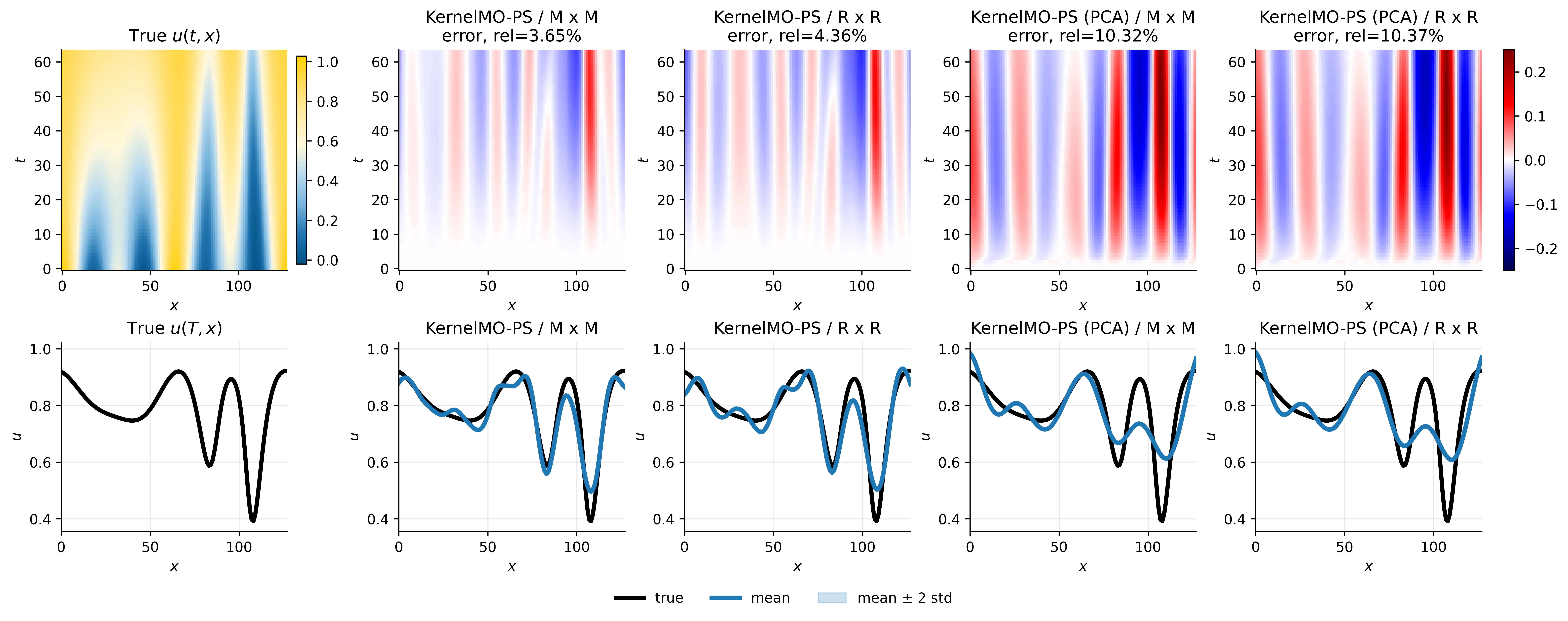}
    \caption{Product-space learning: qualitative prediction and uncertainty comparison for the parametric diffusion-reaction equation on the OOD\_par\_var dataset. The first row shows a reference solution and signed prediction errors for four kernel variants; the relative error for the displayed trajectory is reported in each error-figure title. The second row shows the final-time solution, predictive mean, and an uncertainty band of $\pm 2$ predictive standard deviations. }
    \label{fig:qualitative_framework2_product_param_diffreac_ood_ood_par_var_h5}
\end{figure}

\begin{remark}[Hyperparameter selection and out-of-distribution performance]
\label{rem:hyperparameter-ood}
The out-of-distribution performance of the proposed kernel methods can be
sensitive to the choice of kernel hyperparameters. In all experiments, we
select these hyperparameters using only in-distribution validation data and
do not tune them separately for the OOD test cases. Consequently, a
hyperparameter choice that yields strong in-distribution performance need
not be optimal under distribution shift. Table~\ref{tab:parameter_DR}
illustrates this effect for \KernelMOPS with an RBF kernel on the parametric
diffusion--reaction problem. In particular, the length scale selected based
on in-distribution performance can lead to substantially larger errors for
some OOD shifts, whereas other length scales yield considerably better OOD
performance despite being less accurate in distribution.
\end{remark}

\begin{table}[!htbp]
\centering
\tiny
\begin{tabular}{lcccccc}
Length-scale & Test & OOD\_init & OOD\_par\_init & OOD\_par\_kernel & OOD\_par\_scale & OOD\_par\_var \\\hline
 1 & 5.43\% (3.40\%) & 78.10\% (23.07\%) & 78.11\% (23.05\%) & 4.44\% (1.88\%) & 4.84\% (1.87\%) & 3.81\% (2.21\%) \\
 10 & 0.75\% (0.44\%) & 127.42\% (47.76\%) & 188.16\% (88.56\%) & 105.53\% (50.74\%) & 157.30\% (77.91\%) & 0.76\% (0.45\%) \\
 100 & 1.51\% (0.68\%) & 251.19\% (158.00\%) & 458.95\% (284.88\%) & 276.53\% (146.44\%) & 408.37\% (227.92\%) & 1.56\% (0.70\%) \\
 1000 & 3.16\% (1.21\%) & 516.14\% (344.66\%) & 937.77\% (641.70\%) & 575.36\% (332.08\%) & 811.18\% (501.40\%) & 3.29\% (1.23\%)\\\hline
\end{tabular}
\caption{Product-space learning: performance of \KernelMOPS with an RBF kernel on the parametric diffusion--reaction equation for different length scales. The first column reports the length scale of the kernel \(k_U\) acting on the initial conditions, while the remaining columns report the corresponding in-distribution and OOD prediction errors.}
\label{tab:parameter_DR}
\end{table}

\subsubsection{Parametric Wave Equation}

We consider the following parametric wave equation:
\begin{equation*}
\begin{aligned}
u_{tt} &= \alpha^2(t) u_{xx}, \quad (t,x)\in[0,2]\times[0,2] \\
u(0,x) &= u_0(x), \\
u_t(0,x) &= 0, \\
u(t,0) &= u(t,2),
\end{aligned}
\end{equation*}
where the time-dependent parametric function $\alpha(t)$ is drawn from a Gaussian random process with RBF kernel (length scale = 1) and with variance 1. 
The parametric function $\alpha(t)$ is evaluated at 64 sensor points corresponding to the boundaries of spaced cells, $\{t_i^b\}_{i=1}^{64}$, and the resulting values $\alpha(t_i^b)$ are encoded within the operator-valued learning and product-space learning methods.

\paragraph{Operator-Valued Learning.}

The mean relative errors and standard deviations are summarized in Table~\ref{tab:framework1_par_wave}, while the implementation details and hyperparameter choices are summarized in Table~\ref{tab:parametricWaveOV}. Overall, the proposed kernel-based methods achieve the lowest prediction errors on the in-distribution test set, with the best performance obtained by \KernelMOOV\ (PCA) / R. Compared with the best-performing neural operator baseline, the proposed method improves the prediction accuracy by more than one order of magnitude on the in-distribution test set (from $17.83\%$ for DeepONet-C to $1.59\%$ for \KernelMOOV\ (PCA) / R). On the OOD\_var dataset, the proposed methods also substantially outperform all neural operator baselines, achieving the lowest prediction error of $1.29\%$ compared with $17.11\%$ for DeepONet-C. By contrast, all methods experience a significant degradation in performance on the OOD\_matern and OOD\_scale datasets, indicating that these distribution shifts are particularly challenging for this PDE. This behavior differs from that observed for the parametric diffusion--reaction equation (Table~\ref{tab:framework1_parDR}), demonstrating that out-of-distribution generalization depends strongly on the underlying PDE. Applying PCA substantially reduces the dimensionality of the learning problem while consistently improving the out-of-distribution performance of the proposed kernel-based methods.

\begin{table}[H]
\centering
\small
\begin{tabular}{lcccc}
\hline
Method / kernel & Test & OOD\_matern & OOD\_scale & OOD\_var \\
\hline
\textbf{KernelMO-OV / M} & 2.63\% (5.86\%) & 56.69\% (17.61\%) & 71.93\% (13.38\%) & 2.05\% (3.22\%) \\
\textbf{KernelMO-OV / R} & 9.37\% (25.42\%) & 209.27\% (247.77\%) & 117.14\% (69.91\%) & 6.65\% (10.34\%) \\
\textbf{KernelMO-OV (PCA) / M} & 2.50\% (4.42\%) & 32.90\% (16.17\%) & 43.27\% (20.56\%) & 1.98\% (2.33\%) \\
\textbf{KernelMO-OV (PCA) / R} & {\bf 1.59\% (2.34\%)} & 31.57\% (15.71\%) & 43.60\% (20.86\%) & {\bf 1.29\% (1.04\%)} \\
\textbf{KernelMO-PS / M x M} & 2.63\% (5.86\%) & 56.69\% (17.61\%) & 71.93\% (13.38\%) & 2.05\% (3.22\%) \\
\textbf{KernelMO-PS / R x R} & 2.37\% (5.13\%) & 66.05\% (17.44\%) & 79.51\% (13.45\%) & 1.68\% (2.29\%) \\
\textbf{KernelMO-PS (PCA) / M x M} & 2.50\% (4.42\%) & 32.27\% (16.31\%) & 42.40\% (21.19\%) & 1.98\% (2.33\%) \\
\textbf{KernelMO-PS (PCA) / R x R} & 1.64\% (2.46\%) & {\bf 30.42\% (15.96\%)} & 42.36\% (21.68\%) & 1.31\% (1.05\%) \\\hline
DeepONet-C & 17.83\% (7.29\%) & 35.91\% (17.51\%) & {\bf 41.90\% (23.91\%)} & 17.11\% (5.98\%) \\
MIONet & 19.08\% (8.44\%) & 67.88\% (20.60\%) & 70.41\% (20.11\%) & 18.13\% (6.75\%) \\
MNO & 19.46\% (8.38\%) & 40.09\% (19.75\%) & 46.28\% (24.27\%) & 18.60\% (7.19\%) \\\hline
\end{tabular}
\caption{Operator-valued learning: performance comparison on the parametric wave equation. We report mean relative errors with standard deviations in parentheses on in-distribution and out-of-distribution datasets.
}
\label{tab:framework1_par_wave}
\end{table}

\begin{figure}[H]
    \centering
    \includegraphics[width=\textwidth]{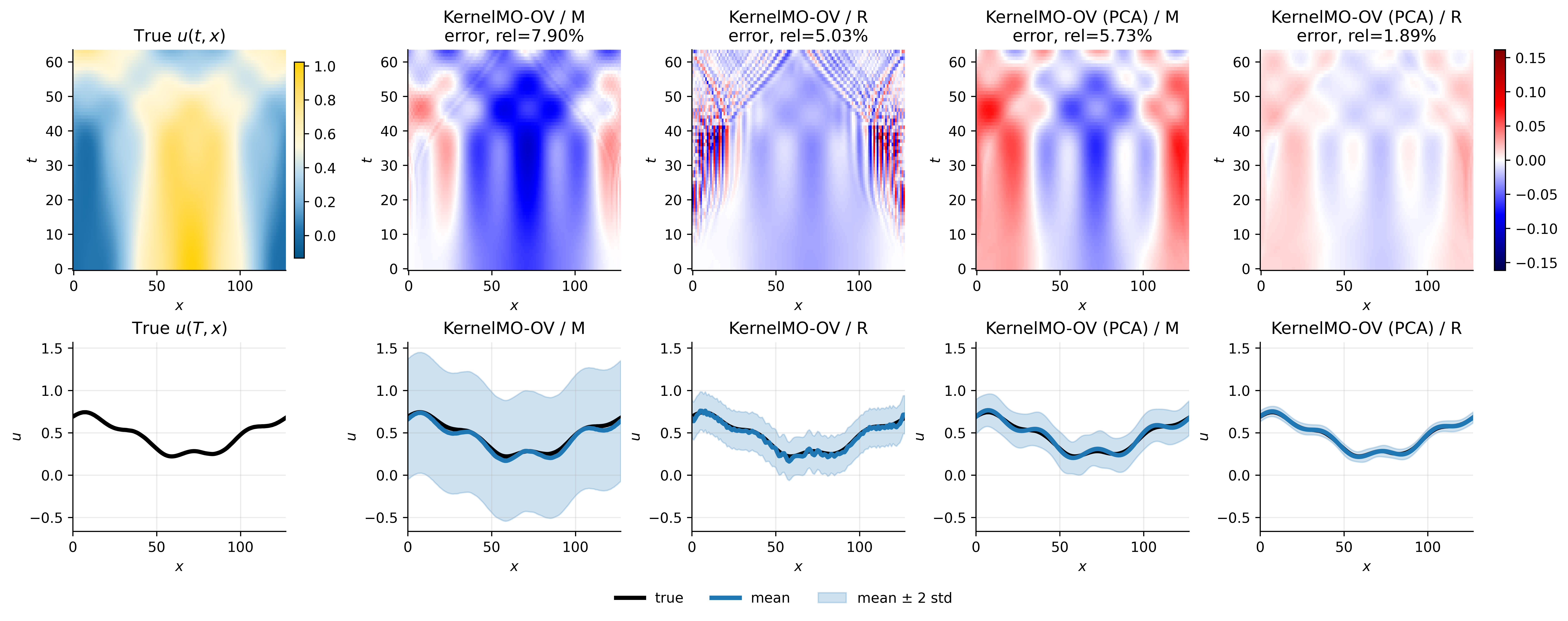}
    \caption{Operator-valued learning: qualitative prediction and uncertainty comparison for the parametric wave equation on the OOD\_var dataset. The first row shows a reference solution and signed prediction errors for four kernel variants; the relative error for the displayed trajectory is reported in each error-figure title. The second row shows the final-time solution, predictive mean, and an uncertainty band of $\pm 2$ predictive standard deviations. }
    \label{fig:qualitative_framework1_operator_param_wave_ood_ood_var_h5}
\end{figure}

\paragraph{Product-Space Learning.}

The mean relative errors and standard deviations are summarized in Table~\ref{tab:framework2_par_wave}, while the implementation details and hyperparameter choices are summarized in Table~\ref{tab:parametricWavePS}. Overall, the proposed PCA-based product-kernel methods achieve the lowest prediction errors across both the in-distribution and out-of-distribution datasets. In particular, \KernelMOPS\ (PCA) / R $\times$ R attains the lowest in-distribution test error of $2.13\%$, compared with $18.20\%$ for the best-performing neural operator baseline DeepONet-C, corresponding to an improvement of nearly one order of magnitude. Moreover, \KernelMOPS\ (PCA) / R $\times$ R achieves the best performance on four of the five out-of-distribution datasets, while \KernelMOPS\ (PCA) / M $\times$ M performs best on the remaining one. By contrast, although the original product-kernel methods achieve competitive in-distribution prediction errors, their performance deteriorates substantially under several distribution shifts, particularly on the OOD\_par\_init, OOD\_par\_kernel, and OOD\_par\_scale datasets. These results demonstrate that combining the product-kernel formulation with PCA substantially improves the robustness and out-of-distribution generalization of the proposed approach.

\begin{table}[H]
\centering
\tiny
\begin{tabular}{lcccccc}
\hline
Method / kernel & Test & OOD\_init & OOD\_par\_init & OOD\_par\_kernel & OOD\_par\_scale & OOD\_par\_var \\
\hline
KernelO / M & 83.16\% (40.26\%) & 80.02\% (26.03\%) & 84.30\% (28.92\%) & 74.34\% (38.37\%) & 72.95\% (37.89\%) & 75.26\% (39.70\%) \\
KernelO / R & 59.29\% (21.32\%) & 13473.96\% (6245.11\%) & 15134.50\% (6361.41\%) & 50.63\% (19.00\%) & 52.13\% (20.95\%) & 55.35\% (20.30\%) \\
KernelO (PCA) / M & 83.14\% (40.25\%) & 124.30\% (39.58\%) & 135.55\% (41.91\%) & 74.32\% (38.36\%) & 72.93\% (37.89\%) & 75.24\% (39.69\%) \\
KernelO (PCA) / R & 59.28\% (21.32\%) & 1415.03\% (1126.69\%) & 1585.64\% (1200.96\%) & 50.62\% (19.00\%) & 52.13\% (20.95\%) & 55.35\% (20.30\%) \\
\textbf{KernelMO-PS / M x M} & 3.33\% (6.75\%) & 38.15\% (22.19\%) & 97.65\% (0.34\%) & 60.48\% (17.09\%) & 76.32\% (11.83\%) & 2.67\% (4.98\%) \\
\textbf{KernelMO-PS / R x R} & 3.52\% (6.25\%) & 37.89\% (22.21\%) & 99.55\% (0.21\%) & 68.31\% (19.36\%) & 83.82\% (12.27\%) & 2.88\% (5.35\%) \\
\textbf{KernelMO-PS (PCA) / M x M} & 2.84\% (4.51\%) & 37.72\% (22.24\%) & 46.29\% (5.58\%) & 32.06\% (17.45\%) & {\bf 41.43\% (19.02\%)} & 2.37\% (3.53\%) \\
\textbf{KernelMO-PS (PCA) / R x R} & {\bf 2.13\% (3.38\%)} & {\bf 37.47\% (22.24\%)} & {\bf 46.20\% (5.55\%)} & {\bf 31.46\% (17.10\%)} & 41.73\% (18.62\%) & {\bf 1.69\% (2.00\%)} \\\hline
DeepONet & 56.93\% (21.72\%) & 54.88\% (19.86\%) & 55.52\% (14.22\%) & 48.72\% (19.18\%) & 50.37\% (20.99\%) & 53.70\% (21.09\%) \\
DeepONet-C & 18.20\% (7.00\%) & 42.10\% (20.72\%) & 108.17\% (14.57\%) & 35.87\% (17.08\%) & 42.41\% (19.92\%) & 17.43\% (6.70\%) \\
MIONet & 20.54\% (10.85\%) & 45.40\% (20.06\%) & 170.46\% (16.33\%) & 97.25\% (28.43\%) & 110.74\% (28.18\%) & 19.07\% (10.51\%) \\
MNO & 25.04\% (11.75\%) & 43.36\% (20.27\%) & 73.84\% (19.56\%) & 46.96\% (27.01\%) & 53.72\% (28.81\%) & 22.58\% (10.79\%) \\\hline
\end{tabular}
\caption{Product-space learning: performance comparison on the parametric wave equation. We report mean relative errors with standard deviations in parentheses on in-distribution and out-of-distribution datasets.
}
\label{tab:framework2_par_wave}
\end{table}

\begin{figure}[H]
    \centering
    \includegraphics[width=\textwidth]{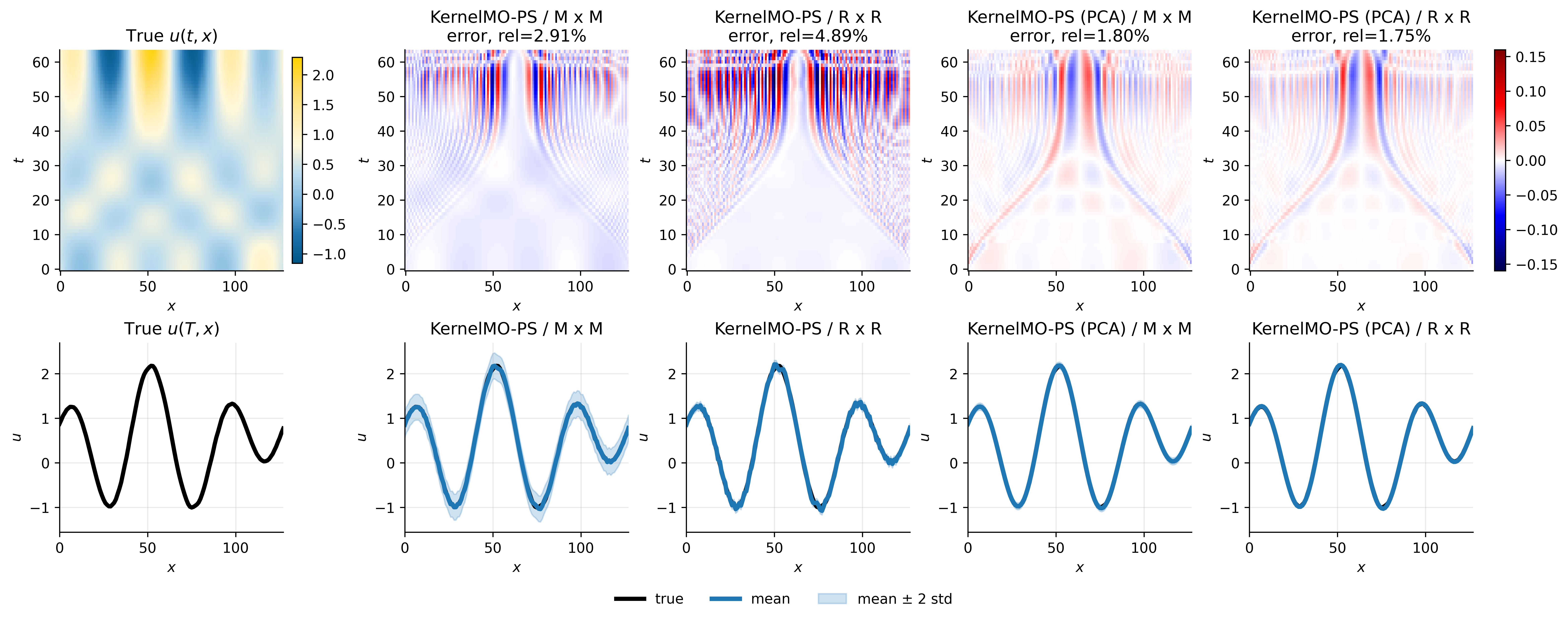}
    \caption{Product-space learning: qualitative prediction and uncertainty comparison for the parametric wave equation on the OOD\_par\_var dataset. The first row shows a reference solution and signed prediction errors for four kernel variants; the relative error for the displayed trajectory is reported in each error-figure title. The second row shows the final-time solution, predictive mean, and an uncertainty band of $\pm 2$ predictive standard deviations. }
    \label{fig:qualitative_framework2_product_param_wave_ood_ood_par_var_h5}
\end{figure}

\subsection{Efficiency}

In this section, we compare the computational efficiency of the proposed kernel-based methods with the neural network-based baselines in terms of training time and prediction time. 
For brevity, we report the results for the conservation law problem only, as similar trends are observed across the other PDE models. 

Table~\ref{tab:timing-f1} summarizes the computational cost of the operator-valued learning problem shown in Fig.~\ref{fig:cd-operator}. Although both \KernelMOOV and \KernelMOPS rely on kernel regression, \KernelMOOV is substantially more efficient during training. As explained in Section \ref{sec:multiple} (see also Table \ref{tab:mol-comparison}), the reason is that its kernel matrix has dimensions $n_\alpha \times n_\alpha$ and therefore contains $320^2 = 102400$ entries, whereas the kernel matrix of \KernelMOPS\ has dimensions $(n_\alpha n_u)\times(n_\alpha n_u)$ and contains $320^2 20^2 = 40960000$ entries. Thus, without PCA, \KernelMOOV\ trains approximately
$13$--$14\times$ faster than \KernelMOPS. For the PCA-based variants,
the corresponding speedup is approximately $48$--$55\times$.

Compared with the neural operator baselines, the kernel methods are also considerably more efficient. The fastest \KernelMOOV\ variants require less than $0.5$ seconds for training, compared with $158$--$250$ seconds for the neural operators, corresponding to speedups of more than two orders of magnitude. At inference time, \KernelMOOV\ also provides the fastest predictions, requiring approximately $0.04$\,ms per sample. This corresponds to prediction times that are approximately $4$--$9\times$ faster than the neural operator baselines. With PCA, the prediction time is further reduced to approximately $0.004$\,ms per sample, yielding speedups of roughly $40$--$80\times$ over the neural operators. Although \KernelMOPS\ is slower due to its larger kernel matrix, it remains competitive with the neural operator baselines while providing comparable or better predictive accuracy.

Applying PCA yields an additional reduction in computational cost. For \KernelMOOV, PCA decreases the training time from approximately $0.47$ s to $0.036$ s and reduces the prediction time by almost an order of magnitude. Similar improvements are observed for \KernelMOPS. These computational savings are obtained with only a modest reduction in predictive accuracy, indicating that PCA provides an attractive trade-off between computational efficiency and accuracy.

\begin{table}[H]
\centering
\scriptsize
\begin{tabular}{lccc}
\hline
Method / kernel & Train time (s) & Test pred. (ms/sample) \\
\hline
KernelMO-OV / M & 0.472 & 0.0375 \\
KernelMO-OV / R & 0.464 & 0.0377 \\
KernelMO-OV(PCA) / M & 0.0357 & 0.00407 \\
KernelMO-OV (PCA) / R & 0.0361 & 0.00412 \\
KernelMO-PS / M x M & 6.56 & 0.538 \\
KernelMO-PS / R x R & 6.24 & 0.497 \\
KernelMO-PS (PCA) / M x M & 1.96 & 0.24 \\
KernelMO-PS (PCA) / R x R & 1.72 & 0.19 \\
DeepONet-C & 250.44 & 0.334 \\
MIONet & 157.95 & 0.17 \\
MNO & 180.54 & 0.16 \\
\hline
\end{tabular}
\caption{Operator-valued learning: Training time and prediction time per evaluated sample.}
\label{tab:timing-f1}
\end{table}

These observations are also illustrated in Fig.~\ref{fig:paretoOV}, which jointly displays predictive accuracy and computational cost. The figure highlights that \KernelMOOV\ consistently combines the lowest prediction errors with the shortest training and inference times among all methods considered. While the PCA variants incur a modest increase in prediction error, they achieve substantial additional reductions in computational cost. In contrast, the neural operator baselines require significantly longer training times while providing lower predictive accuracy, whereas \KernelMOPS\ offers a competitive compromise between computational efficiency and predictive performance despite its larger kernel matrix.

\begin{figure}[H]
\centering
\begin{subfigure}{0.49\textwidth}
\centering
\begin{tikzpicture}
\begin{semilogxaxis}[
width=\linewidth,height=6.2cm,
xlabel={Training time (s)},ylabel={Mean relative error (\%)},
xmin=0.02,xmax=400,ymin=0,ymax=5,
grid=major,clip=false,mark size=2.7pt,
every axis plot/.append style={thick},
legend columns=4,legend to name=paretolegendOV,
legend style={font=\scriptsize,draw=none,/tikz/every even column/.append style={column sep=0.8em},row sep=0.2em}]

\addplot+[only marks,mark=*,blue] coordinates{(0.472,0.01) (0.464,0.02)};
\addlegendentry{KernelMO-OV}
\node[fill=white,rounded corners=1pt,inner sep=1.2pt,font=\scriptsize,text=blue]
at (axis cs:0.472,0.01) [anchor=west,xshift=4pt,yshift=-1pt] {M};
\node[fill=white,rounded corners=1pt,inner sep=1.2pt,font=\scriptsize,text=blue]
at (axis cs:0.464,0.02) [anchor=west,xshift=4pt,yshift=5pt] {R};

\addplot+[only marks,mark=square*,blue!60] coordinates{(0.0357,1.77) (0.0361,1.77)};
\addlegendentry{KernelMO-OV (PCA)}
\node[fill=white,rounded corners=1pt,inner sep=1.2pt,font=\scriptsize,text=blue!70]
at (axis cs:0.0357,1.77) [anchor=east,xshift=-4pt,yshift=-5pt] {M};
\node[fill=white,rounded corners=1pt,inner sep=1.2pt,font=\scriptsize,text=blue!70]
at (axis cs:0.0361,1.77) [anchor=west,xshift=4pt,yshift=5pt] {R};

\addplot+[only marks,mark=triangle*,red] coordinates{(6.56,0.01) (6.24,1.46)};
\addlegendentry{KernelMO-PS}
\node[fill=white,rounded corners=1pt,inner sep=1.2pt,font=\scriptsize,text=red]
at (axis cs:6.56,0.01) [anchor=west,xshift=4pt,yshift=3pt] {$M\times M$};
\node[fill=white,rounded corners=1pt,inner sep=1.2pt,font=\scriptsize,text=red]
at (axis cs:6.24,1.46) [anchor=west,xshift=4pt] {$R\times R$};

\addplot+[only marks,mark=diamond*,red!70] coordinates{(1.96,1.77) (1.72,2.34)};
\addlegendentry{KernelMO-PS (PCA)}
\node[fill=white,rounded corners=1pt,inner sep=1.2pt,font=\scriptsize,text=red!70]
at (axis cs:1.96,1.77) [anchor=west,xshift=4pt] {$M\times M$};
\node[fill=white,rounded corners=1pt,inner sep=1.2pt,font=\scriptsize,text=red!70]
at (axis cs:1.72,2.34) [anchor=west,xshift=4pt] {$R\times R$};

\addplot+[only marks,mark=o,black] coordinates{(250.44,3.96)};
\addlegendentry{DeepONet-C}
\addplot+[only marks,mark=triangle,black] coordinates{(157.95,1.23)};
\addlegendentry{MIONet}
\addplot+[only marks,mark=square,black] coordinates{(180.54,1.84)};
\addlegendentry{MNO}
\end{semilogxaxis}
\end{tikzpicture}
\caption{Training time.}
\end{subfigure}
\hfill
\begin{subfigure}{0.49\textwidth}
\centering
\begin{tikzpicture}
\begin{semilogxaxis}[
width=\linewidth,height=6.2cm,
xlabel={Inference time (ms/sample)},ylabel={Mean relative error (\%)},
xmin=0.003,xmax=1,ymin=0,ymax=5,
grid=major,clip=false,mark size=2.7pt,
every axis plot/.append style={thick}]

\addplot+[only marks,mark=*,blue] coordinates{(0.0375,0.01) (0.0377,0.02)};
\node[fill=white,rounded corners=1pt,inner sep=1.2pt,font=\scriptsize,text=blue]
at (axis cs:0.0375,0.01) [anchor=west,xshift=4pt,yshift=-1pt] {M};
\node[fill=white,rounded corners=1pt,inner sep=1.2pt,font=\scriptsize,text=blue]
at (axis cs:0.0377,0.02) [anchor=west,xshift=4pt,yshift=5pt] {R};

\addplot+[only marks,mark=square*,blue!60] coordinates{(0.00407,1.77) (0.00412,1.77)};
\node[fill=white,rounded corners=1pt,inner sep=1.2pt,font=\scriptsize,text=blue!70]
at (axis cs:0.00407,1.77) [anchor=east,xshift=-4pt,yshift=-5pt] {M};
\node[fill=white,rounded corners=1pt,inner sep=1.2pt,font=\scriptsize,text=blue!70]
at (axis cs:0.00412,1.77) [anchor=west,xshift=4pt,yshift=5pt] {R};

\addplot+[only marks,mark=triangle*,red] coordinates{(0.538,0.01) (0.497,1.46)};
\node[fill=white,rounded corners=1pt,inner sep=1.2pt,font=\scriptsize,text=red]
at (axis cs:0.538,0.01) [anchor=east,xshift=-4pt,yshift=3pt] {$M\times M$};
\node[fill=white,rounded corners=1pt,inner sep=1.2pt,font=\scriptsize,text=red]
at (axis cs:0.497,1.46) [anchor=east,xshift=-4pt] {$R\times R$};

\addplot+[only marks,mark=diamond*,red!70] coordinates{(0.24,1.77) (0.19,2.34)};
\node[fill=white,rounded corners=1pt,inner sep=1.2pt,font=\scriptsize,text=red!70]
at (axis cs:0.24,1.77) [anchor=west,xshift=4pt] {$M\times M$};
\node[fill=white,rounded corners=1pt,inner sep=1.2pt,font=\scriptsize,text=red!70]
at (axis cs:0.19,2.34) [anchor=west,xshift=4pt] {$R\times R$};

\addplot+[only marks,mark=o,black] coordinates{(0.334,3.96)};
\addplot+[only marks,mark=triangle,black] coordinates{(0.17,1.23)};
\addplot+[only marks,mark=square,black] coordinates{(0.16,1.84)};
\end{semilogxaxis}
\end{tikzpicture}
\caption{Inference time.}
\end{subfigure}

\vspace{0.8ex}
\pgfplotslegendfromname{paretolegendOV}

\caption{Operator-valued learning: trade-off between predictive accuracy and computational cost on the conservation-law benchmark. Labels beside kernel points indicate the kernel choice: M denotes Mat\'ern and R denotes RBF.}
\label{fig:paretoOV}
\end{figure}
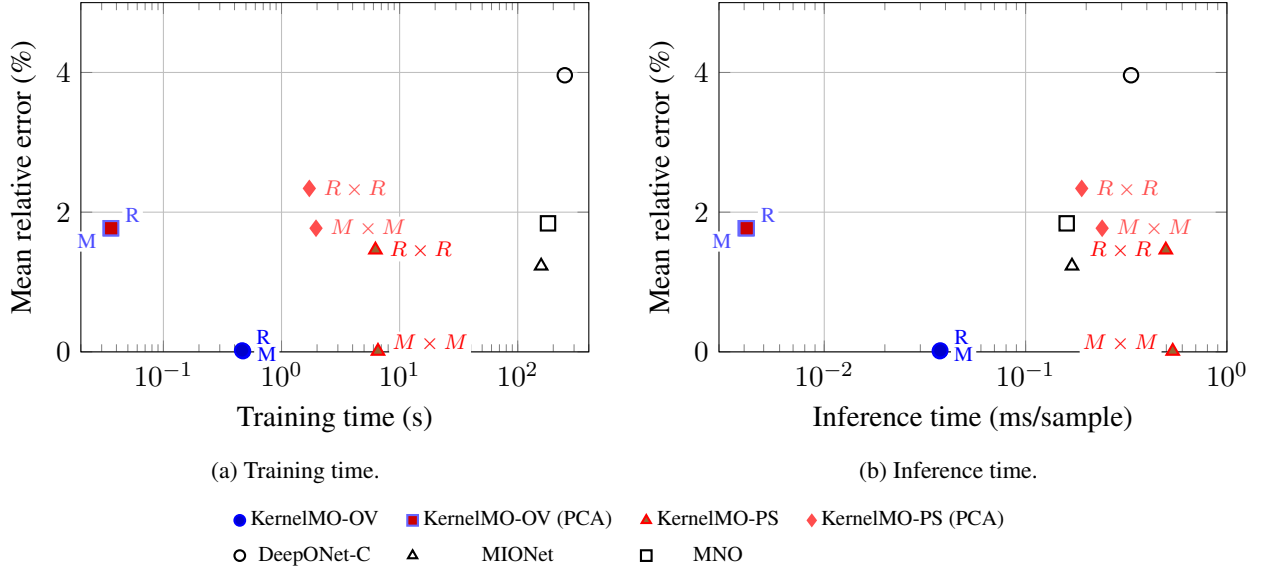

Table~\ref{tab:timing-f2} summarizes the computational cost of the product-space learning formulation shown in Fig.~\ref{fig:cd-product}. As in the operator-valued setting, the proposed kernel-based methods are substantially more computationally efficient than the neural operator baselines. The full-dimensional kernel methods require only $15$--$16$ seconds for training, compared with $227$--$397$ seconds for the neural operators, corresponding to speedups of approximately $15$--$25\times$. Applying PCA further reduces the training time to approximately $4$--$5$ seconds, yielding an additional factor of about three. At inference time, the prediction cost of the kernel methods is comparable to that of the neural operator baselines. The full-dimensional methods require approximately $0.8$ ms per sample, while the PCA variants reduce this to approximately $0.35$--$0.42$ ms per sample, which is comparable to MIONet and MNO and substantially faster than DeepONet-C. The proposed kernel methods are particularly attractive in applications where both training and inference efficiency are important, such as settings requiring frequent retraining or repeated deployment on new operator families.

One observation is that \KernelMOPS\ and the corresponding single-operator baseline KernelO exhibit nearly identical training and prediction times, both with and without PCA. This is expected, since both methods solve kernel regression problems of the same size and therefore have essentially identical computational complexity. However, as demonstrated throughout the experimental Section \ref{sec:performance}, \KernelMOPS\ consistently achieves substantially higher predictive accuracy than KernelO across all benchmarks. These results demonstrate that the proposed product-kernel framework provides a favorable trade-off between computational efficiency and predictive performance, improving accuracy in multiple operator learning without increasing computational cost.

\begin{table}[H]
\centering
\scriptsize
\begin{tabular}{lcc}
\hline
Method / kernel & Train time (s) & Test pred. (ms/sample) \\
\hline
KernelMO-PS (PCA) / M x M & 5.19 & 0.415  \\
KernelMO-PS (PCA) / R x R & 4.7 & 0.365 \\
KernelMO-PS / M x M & 15.5 & 0.874  \\
KernelMO-PS / R x R & 14.8 & 0.829  \\
KernelO (PCA) / M & 4.06 & 0.339  \\
KernelO (PCA) / R & 3.9 & 0.323  \\
KernelO / M & 14 & 0.815  \\
KernelO / R & 14 & 0.778  \\
DeepONet & 226.79 & 0.23  \\
DeepONet-C & 397.46 & 0.83  \\
MIONet & 245.41 & 0.41  \\
MNO & 277.59 & 0.38  \\\hline
\end{tabular}
\caption{Product-space learning: Training time and prediction time per evaluated sample.}
\label{tab:timing-f2}
\end{table}

These observations are further illustrated in Fig.~\ref{fig:paretoPS}, which jointly displays predictive accuracy and computational cost. The figure shows that the proposed product-kernel methods consistently achieve a favorable balance between accuracy and efficiency. In particular, \KernelMOPS\ occupies a similar computational regime to the corresponding single-operator baseline KernelO while consistently attaining substantially lower prediction errors. The PCA variants further reduce both training and inference times while maintaining competitive predictive performance. Overall, the figure highlights that the proposed multiple-operator kernel construction improves predictive accuracy without increasing computational cost, while remaining significantly more efficient to train than the neural operator baselines.

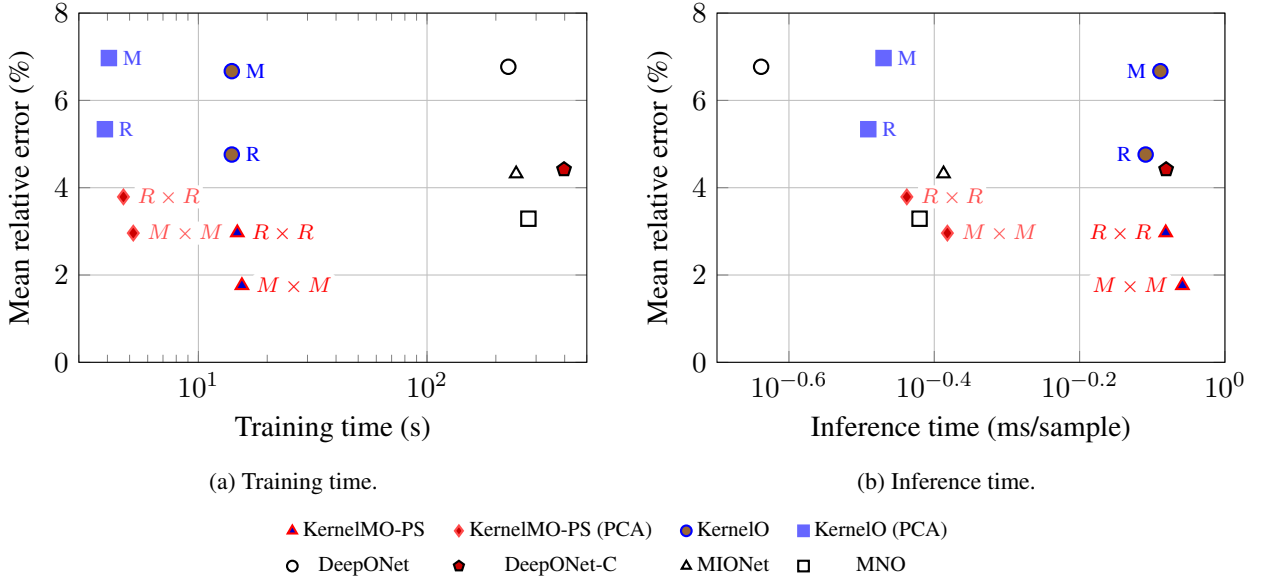
\begin{figure}[H]
\centering
\begin{subfigure}{0.49\textwidth}
\centering
\begin{tikzpicture}
\begin{semilogxaxis}[
width=\linewidth,height=6.2cm,
xlabel={Training time (s)},ylabel={Mean relative error (\%)},
xmin=3,xmax=500,ymin=0,ymax=8,
grid=major,clip=false,mark size=2.7pt,
every axis plot/.append style={thick},
legend columns=4,legend to name=paretolegendPS,
legend style={font=\scriptsize,draw=none,/tikz/every even column/.append style={column sep=0.8em},row sep=0.2em}]

\addplot+[only marks,mark=triangle*,red] coordinates{(15.5,1.76) (14.8,2.97)};
\addlegendentry{KernelMO-PS}
\node[fill=white,rounded corners=1pt,inner sep=1.2pt,font=\scriptsize,text=red]
at (axis cs:15.5,1.76) [anchor=west,xshift=4pt] {$M\times M$};
\node[fill=white,rounded corners=1pt,inner sep=1.2pt,font=\scriptsize,text=red]
at (axis cs:14.8,2.97) [anchor=west,xshift=4pt] {$R\times R$};

\addplot+[only marks,mark=diamond*,red!70] coordinates{(5.19,2.96) (4.70,3.79)};
\addlegendentry{KernelMO-PS (PCA)}
\node[fill=white,rounded corners=1pt,inner sep=1.2pt,font=\scriptsize,text=red!70]
at (axis cs:5.19,2.96) [anchor=west,xshift=4pt] {$M\times M$};
\node[fill=white,rounded corners=1pt,inner sep=1.2pt,font=\scriptsize,text=red!70]
at (axis cs:4.70,3.79) [anchor=west,xshift=4pt] {$R\times R$};

\addplot+[only marks,mark=*,blue] coordinates{(14.0,6.67) (14.0,4.76)};
\addlegendentry{KernelO}
\node[fill=white,rounded corners=1pt,inner sep=1.2pt,font=\scriptsize,text=blue]
at (axis cs:14.0,6.67) [anchor=west,xshift=4pt] {M};
\node[fill=white,rounded corners=1pt,inner sep=1.2pt,font=\scriptsize,text=blue]
at (axis cs:14.0,4.76) [anchor=west,xshift=4pt] {R};

\addplot+[only marks,mark=square*,blue!60] coordinates{(4.06,6.97) (3.90,5.34)};
\addlegendentry{KernelO (PCA)}
\node[fill=white,rounded corners=1pt,inner sep=1.2pt,font=\scriptsize,text=blue!70]
at (axis cs:4.06,6.97) [anchor=west,xshift=4pt] {M};
\node[fill=white,rounded corners=1pt,inner sep=1.2pt,font=\scriptsize,text=blue!70]
at (axis cs:3.90,5.34) [anchor=west,xshift=4pt] {R};

\addplot+[only marks,mark=o,black] coordinates{(226.79,6.77)};
\addlegendentry{DeepONet}
\addplot+[only marks,mark=pentagon*,black] coordinates{(397.46,4.42)};
\addlegendentry{DeepONet-C}
\addplot+[only marks,mark=triangle,black] coordinates{(245.41,4.32)};
\addlegendentry{MIONet}
\addplot+[only marks,mark=square,black] coordinates{(277.59,3.29)};
\addlegendentry{MNO}
\end{semilogxaxis}
\end{tikzpicture}
\caption{Training time.}
\end{subfigure}
\hfill
\begin{subfigure}{0.49\textwidth}
\centering
\begin{tikzpicture}
\begin{semilogxaxis}[
width=\linewidth,height=6.2cm,
xlabel={Inference time (ms/sample)},ylabel={Mean relative error (\%)},
xmin=0.2,xmax=1,ymin=0,ymax=8,
grid=major,clip=false,mark size=2.7pt,
every axis plot/.append style={thick}]

\addplot+[only marks,mark=triangle*,red] coordinates{(0.874,1.76) (0.829,2.97)};
\node[fill=white,rounded corners=1pt,inner sep=1.2pt,font=\scriptsize,text=red]
at (axis cs:0.874,1.76) [anchor=east,xshift=-4pt] {$M\times M$};
\node[fill=white,rounded corners=1pt,inner sep=1.2pt,font=\scriptsize,text=red]
at (axis cs:0.829,2.97) [anchor=east,xshift=-4pt] {$R\times R$};

\addplot+[only marks,mark=diamond*,red!70] coordinates{(0.415,2.96) (0.365,3.79)};
\node[fill=white,rounded corners=1pt,inner sep=1.2pt,font=\scriptsize,text=red!70]
at (axis cs:0.415,2.96) [anchor=west,xshift=4pt] {$M\times M$};
\node[fill=white,rounded corners=1pt,inner sep=1.2pt,font=\scriptsize,text=red!70]
at (axis cs:0.365,3.79) [anchor=west,xshift=4pt] {$R\times R$};

\addplot+[only marks,mark=*,blue] coordinates{(0.815,6.67) (0.778,4.76)};
\node[fill=white,rounded corners=1pt,inner sep=1.2pt,font=\scriptsize,text=blue]
at (axis cs:0.815,6.67) [anchor=east,xshift=-4pt] {M};
\node[fill=white,rounded corners=1pt,inner sep=1.2pt,font=\scriptsize,text=blue]
at (axis cs:0.778,4.76) [anchor=east,xshift=-4pt] {R};

\addplot+[only marks,mark=square*,blue!60] coordinates{(0.339,6.97) (0.323,5.34)};
\node[fill=white,rounded corners=1pt,inner sep=1.2pt,font=\scriptsize,text=blue!70]
at (axis cs:0.339,6.97) [anchor=west,xshift=4pt] {M};
\node[fill=white,rounded corners=1pt,inner sep=1.2pt,font=\scriptsize,text=blue!70]
at (axis cs:0.323,5.34) [anchor=west,xshift=4pt] {R};

\addplot+[only marks,mark=o,black] coordinates{(0.23,6.77)};
\addplot+[only marks,mark=pentagon*,black] coordinates{(0.83,4.42)};
\addplot+[only marks,mark=triangle,black] coordinates{(0.41,4.32)};
\addplot+[only marks,mark=square,black] coordinates{(0.38,3.29)};
\end{semilogxaxis}
\end{tikzpicture}
\caption{Inference time.}
\end{subfigure}

\vspace{0.8ex}
\pgfplotslegendfromname{paretolegendPS}

\caption{Product-space learning: trade-off between predictive accuracy and computational cost on the conservation-law benchmark. Labels beside kernel points indicate the kernel choice: M denotes Mat\'ern and R denotes RBF.}
\label{fig:paretoPS}
\end{figure}

\section{Conclusion}

We have introduced a general kernel-based framework for learning maps between Hilbert spaces within an encoder--decoder architecture. By showing that kernels defined on the latent surrogate space induce corresponding kernels on the original input and output spaces, we established a rigorous theoretical connection between latent-space learning and learning in the original function spaces. This perspective leads to an encoder–decoder error decomposition and a general approximation theory for learning maps between products of function spaces, yielding approximation guarantees that scale favorably with the numbers of input and output tasks. As an application of the proposed framework, we developed two kernel formulations for multiple operator learning. The resulting methods inherit the theoretical guarantees of the general framework while exhibiting complementary computational characteristics. Our numerical experiments demonstrate that these approaches achieve competitive predictive performance while substantially reducing training and inference times compared with state-of-the-art neural operator architectures. 

More broadly, this work suggests that kernel methods constitute a viable alternative to deep neural surrogates for scientific machine learning. Importantly, the kernel learning approaches presented in this work have mathematical guarantees, are computational efficiency, and leverage limited data sampling. At the same time, the proposed encoder-decoder framework is not restricted to multiple operator learning and provides a unified perspective for kernel-based learning of general maps between function spaces.

Several directions for future work naturally arise. From a theoretical perspective, it would be of considerable interest to develop an approximation theory for learning maps whose inputs and outputs are themselves spaces of operators, rather than function spaces, and to establish corresponding guarantees when the encoders and decoders are learned from data. Another important direction is to extend the framework to settings in which the observation spaces are themselves infinite-dimensional, requiring both new theoretical foundations and practical learning algorithms. From a practical perspective, replacing the prescribed encoders and decoders by adaptive learned representations is a natural extension of the present framework. Another promising direction is to investigate the transferability of learned surrogates across different observation operators and to leverage the close connection between kernel methods and Gaussian processes for principled uncertainty quantification.

\section*{Acknowledgments}
This work was supported by NSF 2514157. The authors would like to thank Bamdad Hosseini for helpful discussions. 
\vspace{0.5em}

\bibliographystyle{plain}
\bibliography{references}{}

\newpage

\appendix 
\section{Proofs} \label{sec:proofs}

\subsection{Proofs of the Background Section} \label{sec:proof:back}

\begin{proof}[Proof of Theorem \ref{thm:minimumNorm}]
We first note that \(L^\ast\) exists because \(L:X\to Z\) is a bounded linear operator between Hilbert spaces \cite[Theorem 2.2]{conway2007course}.

\begin{enumerate}
    \item Since \(L\) is surjective, the affine
constraint set
\(
    \mathcal A_S:=\{x\in X:Lx=S\}
\)
is nonempty. If \(x_0\in \mathcal A_S\), then
\begin{equation} \label{eq:minimumNormRecoverty:eq1}
        \mathcal A_S=x_0+\ker L.
\end{equation}
If \(h\in\ker L\), then \(L(x_0+h)=S\), and conversely, if
\(x\in\mathcal A_S\), then \(L(x-x_0)=0\), so \(x-x_0\in\ker L\). Let $x \in \cA_S$ and decompose it as
\[
    x=x_\perp+x_{\ker},
    \qquad
    x_\perp\in(\ker L)^\perp,\quad x_{\ker}\in\ker L.
\]
Then, \(Lx_\perp=Lx=S\), so \(x_\perp \in \cA_S\), and, by \eqref{eq:minimumNormRecoverty:eq1}, every other feasible
point has the form \(x_\perp+h\) with \(h\in\ker L\). Since
\(x_\perp\perp h\),
\(
    \|x_\perp+h\|_X^2
    =
    \|x_\perp\|_X^2+\|h\|_X^2.
\)
Thus, the minimum-norm element of \(\mathcal A_S\), is uniquely defined by having $h = 0$ i.e. $\overline{x}(S)$ is orthogonal to \(\ker L\).

Since \(L\) is
surjective, the restricted operator
\[
    \widetilde L:(\ker L)^\perp\to Z,
    \qquad
    \widetilde Lx:=Lx,
\]
is a bounded linear bijection. It is injective because if
\(x\in(\ker L)^\perp\) and \(\widetilde Lx=0\), then
\(x\in \ker L\cap(\ker L)^\perp = \{0\}\). It is surjective because, for
any \(z\in Z\), surjectivity of \(L\) gives some \(x\in X\) with \(Lx=z\). If
\[
    x=x_\perp+x_{\ker},
    \qquad
    x_\perp\in(\ker L)^\perp,\quad x_{\ker }\in\ker L,
\]
then
\(
    Lx_\perp=Lx=z.
\)
Thus \(\widetilde L\) is bijective. 

By the bounded inverse theorem \cite[Corollary 2.7]{brezis2010functional}, we deduce that \(\widetilde L^{-1}:Z\to(\ker L)^\perp\) is
bounded. Hence there exists a constant \(C>0\) such that, for every \(z\in Z\),
there exists \(x_z\in(\ker L)^\perp\) satisfying
\[
    Lx_z=z,
    \qquad
    \|x_z\|_X\le C\|z\|_Z.
\]
Using this \(x_z\), we obtain
\[
    \|z\|_Z^2
    =
    \langle z,Lx_z\rangle_Z =
    \langle L^\ast z,x_z\rangle_X \le
    \|L^\ast z\|_X\|x_z\|_X \le C\|L^\ast z\|_X\|z\|_Z.
\]
Therefore, for $z \in Z$,
\(
    \|L^\ast z\|_X
    \ge
    C^{-1}\|z\|_Z,
\)
and it follows that
\[
    \langle LL^\ast z,z\rangle_Z =
    \langle L^\ast z,L^\ast z\rangle_X =
    \|L^\ast z\|_X^2 \ge
    C^{-2}\|z\|_Z^2.
\]
Thus, the bilinear form
\[
    a:Z\times Z \mapsto \bbR \qquad \text{given by} \qquad a(z,w):=\langle LL^\ast z,w\rangle_Z
\]
is coercive and bounded, since
\(
    |a(z,w)|
    =
    |\langle L^\ast z,L^\ast w\rangle_X|
    \le
    \|L^\ast\|^2_{\textrm{op}}\|z\|_Z\|w\|_Z = \|L\|^2_{\textrm{op}}\|z\|_Z\|w\|_Z.
\)

We apply the Lax--Milgram theorem \cite[Corollary 5.8]{brezis2010functional} and obtain that for any $F \in Z^*$, there exists a unique element $z \in Z$ such that \(
a(z,w) = F(w)
\)
for all $w \in Z$. By picking $F_{\overline{z}}(w) = \langle \overline{z}, w \rangle_Z$, this implies that \[
a(z,w) = \langle LL^\ast z,w\rangle_Z = \langle \overline{z}, w \rangle_Z
\]
for all $w \in Z$, or equivalently $LL^\ast z = \overline{z}$. 
This shows that \(LL^\ast:Z\to Z\) is a linear bijection and, by \cite[Corollary 2.7]{brezis2010functional}, therefore boundedly
invertible.

Finally, we set
\(
    \overline{x}(S)=L^\ast(LL^\ast)^{-1}S.
\)
Then, 
\(
    L\overline{x}(S)
    =
    LL^\ast(LL^\ast)^{-1}S
    =
    S,
\)
so \(\overline{x}(S) \in \cA_S \). Moreover, \(\overline{x}(S) \in \operatorname{ran}(L^\ast)\), and
\(\operatorname{ran}(L^\ast)\subseteq(\ker L)^\perp\), because for
\(h\in\ker L\) and \(z\in Z\),
\[
    \langle L^\ast z,h\rangle_X
    =
    \langle z,Lh\rangle_Z
    =
    0.
\]
This implies that \(\overline{x}(S)\) is feasible and orthogonal to \(\ker L\), and therefore it is the
unique minimum-norm solution. 
\item We start by proving \begin{equation}
    \label{eq:normal-equation-regularized-minimum-norm-formula}
    \overline{x}_\gamma(S)
    =
    (L^\ast L+\gamma I_X)^{-1}L^\ast S.
\end{equation} 
Define the functional
\[
    J_\gamma:X\to\bbR,
    \qquad
    J_\gamma(x)
    :=
    \|x\|_X^2+\gamma^{-1}\|Lx-S\|_Z^2.
\]
The functional \(J_\gamma\) is strictly convex, since \(x\mapsto\|x\|_X^2\) is
strictly convex. Hence \(J_\gamma\) has at most one minimizer.

We compute the first-order optimality condition. Let \(h\in X\). For
\(\varepsilon\in\bbR\),
\[
\begin{aligned}
    J_\gamma(x+\varepsilon h)
    &=
    \|x+\varepsilon h\|_X^2
    +
    \gamma^{-1}\|Lx+\varepsilon Lh-S\|_Z^2.
\end{aligned}
\]
Differentiating at \(\varepsilon=0\) gives
\[
    \frac{d}{d\varepsilon}J_\gamma(x+\varepsilon h)\bigg|_{\varepsilon=0}
    =
    2\langle x,h\rangle_X
    +
    2\gamma^{-1}\langle Lx-S,Lh\rangle_Z.
\]
Using the definition of the adjoint,
\(
    \langle Lx-S,Lh\rangle_Z
    =
    \langle L^\ast(Lx-S),h\rangle_X
\)
and therefore the first variation is
\(
    2\left\langle
        x+\gamma^{-1}L^\ast(Lx-S),
        h
    \right\rangle_X.
\)
The minimizer hence satisfies
\(
    x+\gamma^{-1}L^\ast(Lx-S)=0
\)
or equivalently,
\[
    (L^\ast L+\gamma I_X)x=L^\ast S.
\]

We now show that \(L^\ast L+\gamma I_X\) is boundedly invertible. For
\(x\in X\), we have
\[
    \langle (L^\ast L+\gamma I_X)x,x\rangle_X
    =
    \langle Lx,Lx\rangle_Z+\gamma\|x\|_X^2 =
    \|Lx\|_Z^2+\gamma\|x\|_X^2 \ge
    \gamma\|x\|_X^2,
\]
and hence the bilinear form
\[
   a:X\times X \mapsto \bbR \qquad \text{given by} \qquad a_X(x,h)
    :=
    \langle (L^\ast L+\gamma I_X)x,h\rangle_X
\]
is coercive. It is also bounded, since
\(
    |a_X(x,h)|
    \le
    (\|L\|_{\mathrm{op}}^2+\gamma)\|x\|_X\|h\|_X.
\)
By the Lax-Milgram theorem \cite[Corollary 5.8]{brezis2010functional} and the same argument as part 1 of the proof, \(L^\ast L+\gamma I_X:X\to X\) is boundedly
invertible and \eqref{eq:normal-equation-regularized-minimum-norm-formula} therefore has the unique
solution
\[
    \overline{x}_\gamma(S)
    =
    (L^\ast L+\gamma I_X)^{-1}L^\ast S.
\]

It remains to show the equivalent measurement-space formula
\eqref{eq:master-regularized-minimum-norm-formula}. First observe that
\(LL^\ast+\gamma I_Z:Z\to Z\) is also boundedly invertible. For
\(z\in Z\),
\[
    \langle (LL^\ast+\gamma I_Z)z,z\rangle_Z
    =
    \langle L^\ast z,L^\ast z\rangle_X+\gamma\|z\|_Z^2 =
    \|L^\ast z\|_X^2+\gamma\|z\|_Z^2 \ge
    \gamma\|z\|_Z^2.
\]
Thus \(LL^\ast+\gamma I_Z\) is coercive and bounded, and hence boundedly
invertible. Next, set
\(
    x_S:=L^\ast(LL^\ast+\gamma I_Z)^{-1}S.
\)
and 
\(
    c:=(LL^\ast+\gamma I_Z)^{-1}S.
\)
Then, 
\(
    (LL^\ast+\gamma I_Z)c=S
\)
and therefore,
\(
    L^\ast LL^\ast c+\gamma L^\ast c=L^\ast S.
\)
Since \(x_S=L^\ast c\), this becomes
\(
    (L^\ast L+\gamma I_X)x_S=L^\ast S
\)
and thus \(x_S\) solves the normal equation
\eqref{eq:normal-equation-regularized-minimum-norm-formula}. By uniqueness of the solution, we conclude that
\[
    \overline{x}_\gamma(S)
    =
    x_S
    =
    L^\ast(LL^\ast+\gamma I_Z)^{-1}S.
\]
\end{enumerate}
\end{proof}

\subsection{Proofs of the Main Results}

\subsubsection{Proofs for the General Learning Framework} \label{sec:proofs:general}

\begin{proof}[Proof of Theorem \ref{thm:induced-kernel-original-spaces}]
We prove the result by applying two standard operations for
operator-valued reproducing kernels: first a pullback in the input variable,
and then a pushforward in the output variable.

We start with the pullback construction. The pullback theorem \cite[Proposition 7]{Carmeli} with $\Psi = E_X: X \mapsto Z_X$ gives the kernel
\[
    \Gamma_{E_X}:X\times X\to \mathcal L(Z_Y),
    \qquad
    \Gamma_{E_X}(x,x')
    :=
    \Gamma(E_Xx,E_Xx')
\]
with associated RKHS
\begin{equation} \label{eq:kernelRepresentation:RKHSpullback}
        \mathcal H_{\Gamma_{E_X}} =
    \{f\circ E_X:f\in\mathcal H_\Gamma\}
\end{equation}
and minimal-representative norm
\begin{equation} \label{eq:kernelRepresentation:RKHSpullbackNorm}
    \|F_0\|_{\mathcal H_{\Gamma_{E_X}}}
    =
    \inf\left\{
        \|f\|_{\mathcal H_\Gamma}
        :
        F_0=f\circ E_X
    \right\}.
\end{equation}
We note that we only use the Hilbert-space part of the pullback theorem. The locally
compact, second countable assumptions appearing in some formulations of this
result are needed for additional topological conclusions \cite[Page 4]{Carmeli}.

We next apply the output transformation. Since
\(
    D_Y:Z_Y\to Y
\)
is bounded and linear, the output pushforward theorem \cite[Proposition 7]{Carmeli} applied to \(\Gamma_{E_X}\) with \(w=D_Y\) implies that
\[
    K:X\times X\to\mathcal L(Y),
    \qquad
    K(x,x')
    :=
    D_Y\Gamma_{E_X}(x,x')D_Y^\ast = D_Y\Gamma(E_Xx,E_Xx')D_Y^\ast.
\]
is a \(Y\)-valued reproducing kernel on \(X\). The output pushforward theorem also identifies the associated RKHS as
\[
    \mathcal H_K
    =
    \{D_Y\circ F_0:F_0\in\mathcal H_{\Gamma_{E_X}}\},
\]
with norm
\begin{equation} \label{eq:kernelRepresentation:RKHSpullforwardNorm}
    \|F\|_{\mathcal H_K}
    =
    \inf\left\{
        \|F_0\|_{\mathcal H_{\Gamma_{E_X}}}
        :
        F=D_Y\circ F_0
    \right\}.
\end{equation}
Using \eqref{eq:kernelRepresentation:RKHSpullback}, every
\(F_0\in\mathcal H_{\Gamma_{E_X}}\) has the form
\(
    F_0=f\circ E_X
\)
for some $f \in \cH_{\Gamma}$. 
Therefore
\[
    \mathcal H_K
    =
    \{D_Y\circ f\circ E_X:f\in\mathcal H_\Gamma\} 
\]
and, inserting \eqref{eq:kernelRepresentation:RKHSpullbackNorm} into \eqref{eq:kernelRepresentation:RKHSpullforwardNorm}, we obtain
\begin{equation} \label{eq:kernelRepresentation:NormK}
    \|F\|_{\mathcal H_K}
    =
    \inf\left\{
        \|f\|_{\mathcal H_\Gamma}
        :
        F=D_Y\circ f\circ E_X
    \right\}.
\end{equation}

It remains to prove the simplified statement under exact minimum-norm recovery. By Theorem \ref{thm:minimumNorm}, we have
\[
    D_X=E_X^\ast(E_XE_X^\ast)^{-1},
    \qquad
    D_Y=E_Y^\ast(E_YE_Y^\ast)^{-1}.
\]
Next, we assume that 
\[
    D_Y\circ f_1\circ E_X
    =
    D_Y\circ f_2\circ E_X.
\]
Then, applying \(E_Y\) on the left and \(D_X\) on the right gives
\(
    f_1=f_2.
\)
Hence the infimum in the minimal-representative norm in \eqref{eq:kernelRepresentation:NormK} is attained by the unique
representative \(f\), and therefore
\[
    \|D_Y\circ f\circ E_X\|_{\mathcal H_K}
    =
    \|f\|_{\mathcal H_\Gamma}.
\]
\end{proof}

\begin{proof}[Proof of Corollary \ref{cor:induced-measurement-space-equivalence}]
By Theorem~\ref{thm:induced-kernel-original-spaces}, every
\(F\in\mathcal H_K\) can be represented as
\(
    F=D_Y\circ f\circ E_X
\)
for some \(f\in\mathcal H_\Gamma\). 
For any such representative,
\(
    E_YF(x_i)
    =
    E_YD_Yf(E_Xx_i)
    =
    Af(U_i).
\)
Hence the constraints
\(
    E_YF(x_i)=S_i
\)
and
\(
    Af(U_i)=S_i
\)
are equivalent, under the representation \(F=D_Y\circ f \circ E_X\).
Since the norm in \(\mathcal H_K\), i.e. 
\[
\|F\|_{\mathcal H_K}
    =
    \inf\left\{
        \|f\|_{\mathcal H_\Gamma}
        :
        F=D_Y\circ f\circ E_X
    \right\},
\]
is the minimal
\(\mathcal H_\Gamma\)-norm over all representatives inducing the same \(F\),
minimizing over \(F\in\mathcal H_K\) is equivalent to minimizing over
representatives \(f\in\mathcal H_\Gamma\). This proves the equivalence of the
interpolation problems.

We now compute the interpolation formula. The observation map is
\[
    L_{\Gamma,A}:\mathcal H_\Gamma\to Z_Y^N,
    \qquad
    L_{\Gamma,A}f
    =
    (Af(U_1),\ldots,Af(U_N)).
\]
This map is bounded and linear because point evaluation is bounded in the RKHS
\(\mathcal H_\Gamma\) and \(A=E_YD_Y\) is bounded and linear by assumption on $D_Y$. For
\(c=(c_1,\ldots,c_N)\in Z_Y^N\), using the product inner product on \(Z_Y^N\), we compute as follows:
\begin{align}
    \langle L_{\Gamma,A} f,c\rangle_{Z_Y^N}
    &=
    \sum_{j=1}^N
    \langle Af(U_j),c_j\rangle_{Z_Y} \notag \\
    &=
    \sum_{j=1}^N
    \langle f(U_j),A^\ast c_j\rangle_{Z_Y} \notag\\
    &=
    \sum_{j=1}^N
    \left\langle
        f,
        \Gamma(\cdot,U_j)A^\ast c_j
    \right\rangle_{\mathcal H_\Gamma} \label{eq:cor:equivalenceRKHS}\\
    &=
    \left\langle
        f,
        \sum_{j=1}^N
        \Gamma(\cdot,U_j)A^\ast c_j
    \right\rangle_{\mathcal H_\Gamma} \notag
\end{align}
where \eqref{eq:cor:equivalenceRKHS} follows from the reproducing property on
\(\mathcal H_\Gamma\).

Therefore
\[
    L_{\Gamma,A}^\ast c
    =
    \sum_{j=1}^N
    \Gamma(\cdot,U_j)A^\ast c_j
\]
and
\[
    L_{\Gamma,A}L_{\Gamma,A}^\ast
    =
    \Gamma_A(\mathbf U,\mathbf U).
\]
If \(L_{\Gamma,A}\) is surjective, equivalently if
\(\Gamma_A(\mathbf U,\mathbf U)\) is boundedly invertible, the minimum-norm
formula from Theorem \ref{thm:minimumNorm} gives
\[
    \overline f
    =
    L_{\Gamma,A}^\ast
    (L_{\Gamma,A}L_{\Gamma,A}^\ast)^{-1}
    \mathbf S.
\]
Evaluating at \(U\in Z_X\) yields
\[
    \overline f(U)
    =
    \Gamma_A(U,\mathbf U)
    \Gamma_A(\mathbf U,\mathbf U)^{-1}
    \mathbf S.
\]
For ridge regression, Theorem \ref{thm:minimumNorm} yields
\[
    \overline f_\lambda(U)
    =
    \Gamma_A(U,\mathbf U)
    \bigl(\Gamma_A(\mathbf U,\mathbf U)+\lambda \Id_{Z_Y^N}\bigr)^{-1}
    \mathbf S.
\]
We note that these formulas correspond to standard kernel interpolation/ridge regression with kernel $A\Gamma(\cdot,\cdot)A^\ast$. The formulas for \(\overline F_\lambda\) follow by composing with \(D_Y\) and
\(E_X\).

Finally, assume that \(D_X\) and
\(D_Y\) are the exact minimum-norm recovery maps. Then,
Theorem~\ref{thm:minimumNorm} gives
\(
    A = E_YD_Y=I_{Z_Y}
\)
and the encoded residuals reduce to
\(
    Af(U_i)=f(U_i).
\)
The simplified interpolation and ridge-regression formulas follow by setting
\(A=I_{Z_Y}\) in the general formulas. 
\end{proof}

\begin{proof}[Proof of Theorem \ref{thm:encoder-decoder-learning-error}] 

For the first statement, observe that
\[
\begin{aligned}
    G(x)-\overline G(x)
    &=
    G(x)-D_Y\widehat G(E_Xx) \\
    &=
    \bigl(G(x)-D_YG_{\mathrm{enc}}(E_Xx)\bigr)
    +
    D_Y\bigl(G_{\mathrm{enc}}-\widehat G\bigr)(E_Xx).
\end{aligned}
\]
Taking norms and applying the triangle inequality gives
\[
    \|G(x)-\overline G(x)\|_Y
    \le
    \|G(x)-D_YG_{\mathrm{enc}}(E_Xx)\|_Y 
    +
    \|D_Y(G_{\mathrm{enc}}-\widehat G)(E_Xx)\|_Y.
\]
Using the definition of \(G_{\mathrm{enc}}\) gives the first estimate. 

For the second statement, writing
\(
    \mathbf G_{\mathrm{enc}}
    :=
    \bigl(G_{\mathrm{enc}}(U_1),\ldots,G_{\mathrm{enc}}(U_N)\bigr)
    \in Z_Y^N,
\)
we have (analogously to the derivation in the proof of Corollary \ref{cor:induced-measurement-space-equivalence})
\[
    \widehat G_{\mathrm{enc},\lambda}(z)
    =
    \Gamma(z,\mathbf U)
    \l\Gamma(\mathbf U,\mathbf U) + \lambda \Id_{Z_Y^N}\r^{-1}
    \mathbf G_{\mathrm{enc}}
\]
and
\[
    \widehat G_\lambda(z)
    =
    \Gamma(z,\mathbf U)
    \l\Gamma(\mathbf U,\mathbf U) + \lambda \Id_{Z_Y^N}\r^{-1}
    \mathbf S,
\]
where
\(
    \mathbf S=(S_1,\ldots,S_N)\in Z_Y^N.
\)
By definition of the data-consistency residuals, we have
\[
    \mathbf S
    =
    \mathbf G_{\mathrm{enc}}+\eta,
    \qquad
    \eta=(\eta_1,\ldots,\eta_N)\in Z_Y^N.
\]
Therefore, for $z \in Z_X$,
\begin{align*}
     \widehat G_\lambda(z)-\widehat G_{\mathrm{enc}}(z)
    &=
    \Gamma(z,\mathbf U)
    \l \Gamma(\mathbf U,\mathbf U) + \lambda\Id_{Z_Y^N} \r^{-1}
    \bigl(\mathbf S-\mathbf G_{\mathrm{enc}}\bigr) \\
    &=
    \Gamma(z,\mathbf U)
    \l \Gamma(\mathbf U,\mathbf U) + \lambda\Id_{Z_Y^N} \r^{-1}
    \eta \\
    &=
    (\mathcal R_{\mathbf U,\lambda}\eta)(z).
\end{align*}
Therefore,
\(
    G_{\mathrm{enc}}-\widehat G_\lambda
    =
    \bigl(G_{\mathrm{enc}}-\widehat G_{\mathrm{enc},\lambda}\bigr)
    -
    \mathcal R_{\mathbf U,\lambda}\eta
\)
and 
\[
    \|(G_{\mathrm{enc}}-\widehat G_\lambda)(E_Xx)\|_{Z_Y}
    \le
    \|(G_{\mathrm{enc}}-\widehat G_{\mathrm{enc},\lambda})(E_Xx)\|_{Z_Y} 
    +
    \|\mathcal R_{\mathbf U,\lambda}\eta(E_Xx)\|_{Z_Y}.
\]
Plugging this into the first statement proves (a). 

Finally, applying \cite[Theorem 5.4]{OWHADI2023133592} yields
\(
    \|(G_{\mathrm{enc}}-\widehat G_{\mathrm{enc,\lambda}})(z)\|_{Z_Y}
    \le
    Q_{N,\lambda}(z) \|G_{\mathrm{enc}}\|_{\mathcal H_{\Gamma,\lambda}}.
\)
Inserting this with $z = E_X(x)$ into (a) yields (b).
\end{proof}

\subsubsection{Proofs for Kernel Learning for Multi-Input, Multi-Output Operators} \label{sec:proofs:multi}

\begin{proof}[Proof of Proposition \ref{prop:encoder-decoder-reconstruction-error}]
Let \(x\in B_R(X)\). We estimate as follows:
\begin{align}
    \|G(x)-D_YE_YG(D_XE_Xx)\|_Y
&\le
\|G(x)-G(D_XE_Xx)\|_Y +
\|G(D_XE_Xx)-D_YE_YG(D_XE_Xx)\|_Y \notag \\
&\leq \omega\left(
    \|x-D_XE_Xx\|_X
\right) + \delta_Y\left(
    G(D_XE_Xx)
\right) \label{eq:prop:eq1} \\
&\leq \omega\left(
    \delta_X(x)
\right) + \delta_Y\left(
    G(D_XE_Xx)
\right) \label{eq:prop:eq2}
\end{align}
where we used Assumption \ref{assumption:regularityG} for \eqref{eq:prop:eq1} and \eqref{eq:prop:reconstructionError} for \eqref{eq:prop:eq1} and \eqref{eq:prop:eq2}. The proof of the second identity is analogous.
\end{proof}

\begin{proof}[Proof of Lemma \ref{lem:sobolev}]
Throughout the proof, \(C>0\) denotes a generic constant independent of
the sampling fill distances and of the functions being reconstructed.
Its value may change from line to line.

We first establish the reconstruction estimate
for a fixed \(j\)-th component. For \(u_j\in\mathcal H_j\), define
\(
e:=u_j-D_j(E_ju_j).
\)
By Theorem \ref{thm:minimumNorm} and definition of $D_j$, we have
\(
E_j(D_j(E_ju_j))
=
E_ju_j
\)
and therefore
\(
E_je=0,
\)
or equivalently,
\begin{equation} \label{eq:lem:sobolev:restriction}
    e|_{A_j}=0.
\end{equation}

Using \eqref{eq:lem:sobolev:restriction} in \cite[Theorem 4.1]{arcangeli2007} with
\(
r=s_j
\)
and each integer derivative order
\(
\ell=0,\ldots,t_j,
\)
yields
\begin{equation}\label{eq:lem:sobolev:seminorm}
|e|_{\Wkp{\ell}{q_j}(\Omega_j)}
\le
C
h_j^{
s_j-\ell
-
n_j\left(\frac1{p_j}-\frac1{q_j}\right)_+
}
|e|_{\Wkp{s_j}{p_j}(\Omega_j)}
\end{equation}
as soon as $h_j \leq h_j^{0}$. Define
\(
\alpha_j
:=
s_j-t_j
-
n_j
\left(
\frac1{p_j}-\frac1{q_j}
\right)_+
\)
so that, for every \(\ell=0,\ldots,t_j\), 
\(
s_j-\ell
-
n_j
\left(
\frac1{p_j}-\frac1{q_j}
\right)_+
=
\alpha_j+t_j-\ell
\)
and
\[
h_j^{
s_j-\ell
-
n_j
\left(
\frac1{p_j}-\frac1{q_j}
\right)_+
}
=
h_j^{\alpha_j}
h_j^{t_j-\ell} \leq h_j^{\alpha_j} \max\{1,h_j^{0}\}^{t_j-\ell} \leq h_j^{\alpha_j} \max\{1,h_j^{0}\}^{l_{0,j}} \leq C h_j^{\alpha_j}
\]
where we used the fact that \(h_j\le h_j^0\) for the first inequality. By definition of the  Sobolev norms (with the usual modification when
\(q_j=\infty\)) and inserting the latter in \eqref{eq:lem:sobolev:seminorm}, we conclude that
\[
\|e\|_{\Wkp{t_j}{q_j}(\Omega_j)}
\le
C
h_j^{\alpha_j}
\|e\|_{\Wkp{s_j}{p_j}(\Omega_j)} \leq C h_j^{\alpha_j}
\|e\|_{\cH_j}
\]
where we used the fact that $\mathcal H_j
\hookrightarrow
\Wkp{s_j}{p_j}(\Omega_j)$ for the last inequality. Moreover,
\(
e=u_j-D_j(E_ju_j)\in\ker(E_j),
\)
while
\(
D_j(E_ju_j)\in\ker(E_j)^\perp,
\)
as was shown in the proof of Theorem \ref{thm:minimumNorm}.
Therefore,
\(
u_j=D_j(E_ju_j)+e
\)
is an orthogonal decomposition in \(\mathcal H_j\), and
\[
\|u_j\|_{\mathcal H_j}^2
=
\|D_j(E_ju_j)\|_{\mathcal H_j}^2
+
\|e\|_{\mathcal H_j}^2 \geq \|e\|_{\mathcal H_j}^2
\]
and, combining the preceding estimates, we conclude that
\[
\|u_j-D_j(E_ju_j)\|_{\Wkp{t_j}{q_j}(\Omega_j)}
\le
C_j
h_j^{\alpha_j}
\|u_j\|_{\mathcal H_j}.
\]

Since the above argument holds for every
\(j=1,\ldots,J\), the definition of the product norm yields
\[
\begin{aligned}
\|u-DEu\|_X^2
&=
\sum_{j=1}^J
\|u_j-D_j(E_ju_j)\|_{\Wkp{t_j}{q_j}(\Omega_j)}^2 \le
\sum_{j=1}^J
C_j^2
h_j^{2\alpha_j}
\|u_j\|_{\mathcal H_j}^2.
\end{aligned}
\]
\end{proof}

\begin{proof}[Proof of Lemma \ref{lem:finite-dimensional-learning-error}]
Throughout the proof, \(C>0\) denotes a generic constant independent of
the sampling fill distances and of the functions being reconstructed.
Its value may change from line to line.

Since \(f\in\mathcal H_\Gamma\), it is an admissible
competitor in the minimization problem defining
\(\widehat f_\lambda\). Therefore,
\begin{equation} \label{eq:prop:learningTerm:eq1}
    \sum_{i=1}^N
\|\widehat f_\lambda(U_i)-f(U_i)\|_2^2
+
\lambda
\|\widehat f_\lambda\|_{\mathcal H_\Gamma}^2
\le
\lambda
\|f\|_{\mathcal H_\Gamma}^2.
\end{equation}
Since
\(
e_\lambda(U_i)
=
f(U_i)-\widehat f_\lambda(U_i),
\)
it follows that
\(
\sum_{i=1}^N
\|e_\lambda(U_i)\|_2^2
\le
\lambda
\|f\|_{\mathcal H_\Gamma}^2,
\)
or, equivalently,
\begin{equation}
\|e_\lambda\|_{\ell^2(\mathbf U_N;\mathbb R^{d_Y})}
\le
\sqrt{\lambda}
\|f\|_{\mathcal H_\Gamma}.
\label{eq:KRR-discrete-residual}
\end{equation}
Similarly, \eqref{eq:prop:learningTerm:eq1} also gives
\(
\|\widehat f_\lambda\|_{\mathcal H_\Gamma}
\le
\|f\|_{\mathcal H_\Gamma}
\)
and therefore,
\begin{equation} \label{eq:prop:learningTerm:norm}
    \|e_\lambda\|_{\mathcal H_\Gamma} =
\|f-\widehat f_\lambda\|_{\mathcal H_\Gamma}
\le
\|f\|_{\mathcal H_\Gamma}
+
\|\widehat f_\lambda\|_{\mathcal H_\Gamma}
\le
2
\|f\|_{\mathcal H_\Gamma}.
\end{equation}

We now apply the vector-valued sampling inequality of \cite[Theorem 17]{LeGia2026} to
\(
e_\lambda \in \cH_\Gamma
\)
with input dimension \(d_X\), output dimension \(d_Y\), and discrete
exponent \(p=2\). Since
\(
\gamma=\max\{2,2,q\}=\max\{2,q\},
\)
we obtain:
\begin{align}
    |e_\lambda|_{\Wkp{s}{q}(\Upsilon;\mathbb R^{d_Y})} & \le
C\Bigg( h_{\mathrm{tr}}^{
\tau-s-d_X\left(\frac12-\frac1q\right)_+
}
|e_\lambda|_{\Hk{\tau}(\Upsilon;\mathbb R^{d_Y})}
+
h_{\mathrm{tr}}^{d_X/\gamma-s}
\|e_\lambda\|_{\ell^2(\mathbf U_N;\mathbb R^{d_Y})}
\Bigg) \notag \\
&\leq C\Bigg( h_{\mathrm{tr}}^{
\tau-s-d_X\left(\frac12-\frac1q\right)_+
}
\|e_\lambda\|_{\cH_\Gamma}
+
h_{\mathrm{tr}}^{d_X/\gamma-s}
\|e_\lambda\|_{\ell^2(\mathbf U_N;\mathbb R^{d_Y})}
\Bigg)  \label{eq:prop:learningTerm:eq2} \\
&\leq C\Bigg( h_{\mathrm{tr}}^{
\tau-s-d_X\left(\frac12-\frac1q\right)_+
}
\|f\|_{\mathcal H_\Gamma}
+
h_{\mathrm{tr}}^{d_X/\gamma-s}
\sqrt{\lambda}
\|f\|_{\mathcal H_\Gamma}
\Bigg)  \label{eq:prop:learningTerm:eq3}
\end{align}
where we used the continuous embedding
\(
\mathcal H_\Gamma
\hookrightarrow
\Hk{\tau}(\Upsilon;\mathbb R^{d_Y})
\)
for \eqref{eq:prop:learningTerm:eq2} and \eqref{eq:prop:learningTerm:norm} as well as \eqref{eq:KRR-discrete-residual} for \eqref{eq:prop:learningTerm:eq3}.

Taking
\(
q=\infty
\)
and
\(
s=0,
\)
we have
\[
\left(\frac12-\frac1q\right)_+
=
\frac12,
\qquad
\gamma=\infty,
\qquad
\frac{d_X}{\gamma}=0
\]
which implies
\begin{equation} \notag \label{eq:prop:special}
   \|e_\lambda\|_{\Lp{\infty}(\Upsilon;\mathbb R^{d_Y})}
\le
C
\left(
h_{\mathrm{tr}}^{\tau-d_X/2}
+
\sqrt{\lambda}
\right)
\|f\|_{\mathcal H_\Gamma}. 
\end{equation}
Using the latter, we conclude that, for every
\(
x\in X
\)
such that
\(
E_Xx\in\Upsilon,
\)
\begin{align*}
    \|e_\lambda(E_Xx)\|_2 &\leq \|e_\lambda\|_{\Lp{\infty}(\Upsilon;\mathbb R^{d_Y})} \leq C
\left(
h_{\mathrm{tr}}^{\tau-d_X/2}
+
\sqrt{\lambda}
\right)
\|f\|_{\mathcal H_\Gamma}. 
\end{align*}

\end{proof}

\begin{proof}[Proof of Theorem \ref{thm:total-product-sobolev-error}]
For \(j=1,\ldots,J_X\) and \(k=1,\ldots,J_Y\), define
\[
\alpha_{X,j}
:=
s_{X,j}-t_{X,j}
-
n_{X,j}
\left(
\frac{1}{p_{X,j}}-\frac{1}{q_{X,j}}
\right)_+
\]
and
\[
\alpha_{Y,k}
:=
s_{Y,k}-t_{Y,k}
-
n_{Y,k}
\left(
\frac{1}{p_{Y,k}}-\frac{1}{q_{Y,k}}
\right)_+.
\]
Fix
\(
x\in B_R(\mathcal H_X) 
\)
and for brevity, write
\(
E_X:=E_{\mathcal H_X},
\)
\(
D_X:=D_{\mathcal H_X},
\)
\(
E_Y:=E_{\mathcal H_Y},
\)
\(
D_Y:=D_{\mathcal H_Y}.
\)
By Theorem \ref{thm:encoder-decoder-learning-error}, we have: 
\begin{align}
    &\|G(x)-\overline{G}(x)\|_Y \notag \\
    &\le
    \|G(x)-D_YE_YG(D_XE_Xx)\|_Y +
    \|D_Y\|_{\mathrm{op}}\,
    \|(G_{\mathrm{enc}}-\widehat f_\lambda)(E_Xx)\|_{2} + \|D_Y\|_{\mathrm{op}}\,
    \|\mathcal R_{\mathbf U,\lambda}\eta(E_Xx)\|_{2} \notag \\
    &\leq \omega\left(\delta_X(x)\right) + \delta_Y\left(G(D_XE_Xx)\right) +
    \|D_Y\|_{\mathrm{op}}\,
    \|(G_{\mathrm{enc}}-\widehat G_{\mathrm{enc},\lambda})(E_Xx)\|_{2} + \|D_Y\|_{\mathrm{op}}\,
    \|\mathcal R_{\mathbf U,\lambda}\eta(E_Xx)\|_{2} \label{eq:thm:multitaskKernel:eq1} \\
    &\leq \omega \l \left[
\sum_{j=1}^{J_X}
C_{X,j}^2
h_{X,j}^{2\alpha_{X,j}}
\|x_j\|_{\mathcal H_{X,j}}^2
\right]^{1/2} \r +  \left[
\sum_{k=1}^{J_Y}
C_{Y,k}^2
h_{Y,k}^{2\alpha_{Y,k}}
\|(G(D_XE_Xx))_k\|_{\mathcal H_{Y,k}}^2
\right]^{1/2} \notag \\
&+ \|D_Y\|_{\mathrm{op}}\,
    \|(G_{\mathrm{enc}}-\widehat G_{\mathrm{enc},\lambda})(E_Xx)\|_{2} + \|D_Y\|_{\mathrm{op}}\,
    \|\mathcal R_{\mathbf U,\lambda}\eta(E_Xx)\|_{2} \label{eq:thm:multitaskKernel:eq2} \\
&\leq \omega \l  \l \max_{1 \leq j \leq J_X} C_{X,j}
h_{X,j}^{\alpha_{X,j}}  \r \left[
\sum_{j=1}^{J_X}
\|x_j\|_{\mathcal H_{X,j}}^2
\right]^{1/2} \r +  \l \max_{1 \leq k \leq J_Y} C_{Y,k} h_{Y,k}^{\alpha_{Y,k}} \r \left[
\sum_{k=1}^{J_Y}
\|(G(D_XE_Xx))_k\|_{\mathcal H_{Y,k}}^2
\right]^{1/2} \notag \\
&+ \|D_Y\|_{\mathrm{op}}\,
    C
\left(
h_{\mathrm{tr}}^{\tau-d_X/2}
+
\sqrt{\lambda}
\right)
\|G_{\mathrm{enc}}\|_{\mathcal H_\Gamma} + \|D_Y\|_{\mathrm{op}}\,
    \|\mathcal R_{\mathbf U,\lambda}\eta(E_Xx)\|_{2} \label{eq:thm:multitaskKernel:eq3} \\
&= \omega \l \l \max_{1 \leq j \leq J_X} C_{X,j}
h_{X,j}^{\alpha_{X,j}}  \r \Vert x \Vert_{\cH_X} \r +  \l \max_{1 \leq k \leq J_Y} C_{Y,k} h_{Y,k}^{\alpha_{Y,k}} \r 
\|G(D_XE_Xx)\|_{\mathcal H_{Y}}\notag \\
&+ \|D_Y\|_{\mathrm{op}}\,
    C
\left(
h_{\mathrm{tr}}^{\tau-d_X/2}
+
\sqrt{\lambda}
\right)
\|G_{\mathrm{enc}}\|_{\mathcal H_\Gamma} + \|D_Y\|_{\mathrm{op}}\,
    \|\mathcal R_{\mathbf U,\lambda}\eta(E_Xx)\|_{2} \notag \\
&\leq \omega \l R \max_{1 \leq j \leq J_X} C_{X,j}
h_{X,j}^{\alpha_{X,j}}   \r +  M_R  \max_{1 \leq k \leq J_Y} C_{Y,k} h_{Y,k}^{\alpha_{Y,k}} 
+ \|D_Y\|_{\mathrm{op}}\,
    C
\left(
h_{\mathrm{tr}}^{\tau-d_X/2} +
\sqrt{\lambda}
\right)
\|G_{\mathrm{enc}}\|_{\mathcal H_\Gamma}\\
&+ \|D_Y\|_{\mathrm{op}}\,
    \|\mathcal R_{\mathbf U,\lambda}\eta(E_Xx)\|_{2} \label{eq:thm:multitaskKernel:eq4}
\end{align}
where we used Proposition \ref{prop:encoder-decoder-reconstruction-error} for \eqref{eq:thm:multitaskKernel:eq1}, Lemma \ref{lem:sobolev} for \eqref{eq:thm:multitaskKernel:eq2}, Lemma \ref{lem:finite-dimensional-learning-error} for \eqref{eq:thm:multitaskKernel:eq3} and Assumption \ref{assumption:regularityG2} for \eqref{eq:thm:multitaskKernel:eq4}.

\end{proof}

\section{Experiment Setup}

\subsection{Out-of-distribution datasets}
\label{Appendix:ood}

In this section, we summarize the details of out-of-distribution datasets for five PDE benchmarks for both operator-valued and product-space learning frameworks.

We first consider the operator-valued learning framework.
For PDE examples (conservation law, diffusion-reaction-advection equation, and nonlinear Klein Gordon equation) where we only have finite-dimensional parameter vectors, we take the following out-of-distribution parameter vectors:
\begin{itemize}
\item {\bf Conservation law:} The components of parameter vector $\alpha=[\alpha_1,\alpha_2,\alpha_3,\alpha_4]^\top$ are sampled from the ranges $\alpha_i\in[0.8\alpha_i^c, 1.2\alpha_i^c]$ with reference values given by $\alpha^c=[1,1,1,0.1]^\top$ for $i=1,2,3,4$. 
\item {\bf Diffusion-reaction-advection:} The first three components of parameter vector $\alpha=[\alpha_1,\alpha_2,\alpha_3,\alpha_4, \alpha_5]^\top$ are sampled from the ranges $\alpha_i\in[0.8\alpha_i^c, 1.2\alpha_i^c]$ with reference values given by $\alpha^c=[0.01,1,1]^\top$ for $i=1,2,3$. 
\item {\bf Nonlinear Klein Gordon:} The components of parameter vector $\alpha=[\alpha_1,\alpha_2,\alpha_3]^\top$ are sampled from the ranges $\alpha_i\in[0.8\alpha_i^c, 1.2\alpha_i^c]$ with reference values given by $\alpha^c=[1,1,1]^\top$ for $i=1,2,3$. 
\end{itemize}
For the parametric diffusion--reaction equation and the parametric wave equation, we generate out-of-distribution test samples by randomly drawing the parametric function from Gaussian processes whose distributions differ from the one used for training. In particular, we consider the following three out-of-distribution datasets:
\begin{itemize}
\item {\bf OOD\_matern}: Gaussian process with a Mat{\'e}rn kernel (length scale parameter $\gamma=1$ and smoothness parameter $\nu=0.5$) and variance $1$.
\item {\bf OOD\_scale}: Gaussian process with an RBF kernel (length scale parameter $\gamma=0.1$) and variance $1$.
\item {\bf OOD\_var}: Gaussian process with an RBF kernel (length scale parameter $\gamma=1$) and variance $0.1^2$.
\end{itemize}

Next, we consider the product-space learning framework.
For the conservation law, diffusion--reaction--advection equation, and nonlinear Klein--Gordon equation, we consider the following out-of-distribution datasets:
\begin{itemize}
\item {\bf OOD\_par}: the same out-of-distribution of parameter vectors as above.
\item {\bf OOD\_init\_GP}: the initial conditions are randomly drawn from a Gaussian process with an RBF kernel (length scale parameter $\gamma=1$) and variance $1$.
\item {\bf OOD\_init\_amp}: the initial conditions follow the sinusoidal formulation with four sine functions, $n_i$ uniformly sampled integers in $[1,8]$, and amplitudes $A_i$ sampled uniformly from $[-2,2]$.
\item {\bf OOD\_par\_init\_GP}: the combination of {\bf OOD\_par} and {\bf OOD\_init\_GP}.
\item {\bf OOD\_par\_init\_amp}: the combination of {\bf OOD\_par} and {\bf OOD\_init\_amp}.
\end{itemize}
For the parametric diffusion--reaction equation and the parametric wave equation, we consider the following out-of-distribution datasets:
\begin{itemize}
\item {\bf OOD\_par\_kernel}: the same as {\bf OOD\_matern} in the operator-valued learning framework.
\item {\bf OOD\_par\_scale}: the same as {\bf OOD\_scale} in the operator-valued learning framework.
\item {\bf OOD\_par\_var}: the same as {\bf OOD\_var} in the operator-valued learning framework.
\item {\bf OOD\_init}: the same as {\bf OOD\_init\_amp} above.
\item {\bf OOD\_par\_init}: the combination of {\bf OOD\_par\_scale} and {\bf OOD\_init}.
\end{itemize}

\subsection{Hyperparameter choices} \label{sec:hyperparameters}

\begin{table}[H]
\centering
\small
\begin{tabular}{lccccc}
\toprule
Method & $k$ & $k_W$ & $k_U$ & PCA & Configuration \\
\midrule
\KernelMOOV
& \makecell[l]{R ($\gamma=0.1$)\\M ($\gamma=1,\nu=5/2$)}
& --
& --
& $v$ (10 PCs)
& -- \\

\KernelMOPS
& --
& \makecell[l]{R ($\gamma=100$)\\M ($\gamma=1,\nu=5/2$)}
& \makecell[l]{R ($\gamma=100$)\\M ($\gamma=1,\nu=5/2$)}
& $u,v$ (10 PCs)
& -- \\

DeepONet-C
& --
& --
& --
& --
& Small \\

MIONet
& --
& --
& --
& --
& Medium \\

MNO
& --
& --
& --
& --
& Large \\
\bottomrule
\end{tabular}
\caption{Hyperparameter choices and implementation details for the operator-valued learning experiments on the parametric conservation law.}
\label{tab:conservationOV}
\end{table}

\begin{table}[H]
\centering
\small
\begin{tabular}{lccccc}
\toprule
Method & $k$ & $k_W$ & $k_U$ & PCA & Configuration \\
\midrule
\KernelO
& \makecell[l]{R ($\gamma=100$)\\M ($\gamma=1,\nu=5/2$)}
& --
& --
& $u,v$ (10 PCs)
& -- \\

\KernelMOPS
& --
& \makecell[l]{R ($\gamma=1$)\\M ($\gamma=1,\nu=5/2$)}
& \makecell[l]{R ($\gamma=100$)\\M ($\gamma=100,\nu=3/2$)}
& $u,v$ (10 PCs)
& -- \\

DeepONet
& --
& --
& --
& --
& Small \\

DeepONet-C
& --
& --
& --
& --
& Small \\

MIONet
& --
& --
& --
& --
& Medium \\

MNO
& --
& --
& --
& --
& Large \\
\bottomrule
\end{tabular}
\caption{Hyperparameter choices and implementation details for the product-space learning experiments on the parametric conservation law.}
\label{tab:conservationPS}
\end{table}

\begin{table}[H]
\centering
\small
\begin{tabular}{lccccc}
\toprule
Method & $k$ & $k_W$ & $k_U$ & PCA & Configuration \\
\midrule
\KernelMOOV
& \makecell[l]{R ($\gamma=1$)\\M ($\gamma=10,\nu=7/2$)}
& --
& --
& $v$ (10 PCs)
& -- \\

\KernelMOPS
& --
& \makecell[l]{R ($\gamma=100$)\\M ($\gamma=10,\nu=5/2$)}
& \makecell[l]{R ($\gamma=100$)\\M ($\gamma=10,\nu=5/2$)}
& $u,v$ (10 PCs)
& -- \\

DeepONet-C
& --
& --
& --
& --
& Small \\

MIONet
& --
& --
& --
& --
& Small \\

MNO
& --
& --
& --
& --
& Large \\
\bottomrule
\end{tabular}
\caption{Hyperparameter choices and implementation details for the operator-valued learning experiments on the parametric diffusion--reaction--advection PDE.}
\label{tab:diffusionReactionAdvectionOV}
\end{table}

\begin{table}[H]
\centering
\small
\begin{tabular}{lccccc}
\toprule
Method & $k$ & $k_W$ & $k_U$ & PCA & Configuration \\
\midrule
KernelO
& \makecell[l]{R ($\gamma=10$)\\M ($\gamma=1,\nu=5/2$)}
& --
& --
& $u,v$ (10 PCs)
& -- \\

\KernelMOPS
& --
& \makecell[l]{R ($\gamma=1$)\\M ($\gamma=1,\nu=5/2$)}
& \makecell[l]{R ($\gamma=100$)\\M ($\gamma=100,\nu=3/2$)}
& $u,v$ (10 PCs)
& -- \\

DeepONet
& --
& --
& --
& --
& Small \\

DeepONet-C
& --
& --
& --
& --
& Small \\

MIONet
& --
& --
& --
& --
& Small \\

MNO
& --
& --
& --
& --
& Large \\
\bottomrule
\end{tabular}
\caption{Hyperparameter choices and implementation details for the product-space learning experiments on the diffusion--reaction--advection equation.}
\label{tab:diffusionReactionAdvectionPS}
\end{table}

\begin{table}[H]
\centering
\small
\begin{tabular}{lccccc}
\toprule
Method & $k$ & $k_W$ & $k_U$ & PCA & Configuration \\
\midrule
\KernelMOOV
& \makecell[l]{R ($\gamma=1$)\\M ($\gamma=1,\nu=5/2$)}
& --
& --
& $v$ (10 PCs)
& -- \\

\KernelMOPS
& --
& \makecell[l]{R ($\gamma=100$)\\M ($\gamma=1,\nu=5/2$)}
& \makecell[l]{R ($\gamma=100$)\\M ($\gamma=1,\nu=7/2$)}
& $u,v$ (10 PCs)
& -- \\

DeepONet-C
& --
& --
& --
& --
& Small \\

MIONet
& --
& --
& --
& --
& Medium \\

MNO
& --
& --
& --
& --
& Medium \\
\bottomrule
\end{tabular}
\caption{Hyperparameter choices and implementation details for the operator-valued learning experiments on the nonlinear Klein--Gordon equation.}
\label{tab:kleinGordonOV}
\end{table}

\begin{table}[H]
\centering
\small
\begin{tabular}{lccccc}
\toprule
Method & $k$ & $k_W$ & $k_U$ & PCA & Configuration \\
\midrule
KernelO
& \makecell[l]{R ($\gamma=10$)\\M ($\gamma=1,\nu=5/2$)}
& --
& --
& $u,v$ (10 PCs)
& -- \\

\KernelMOPS
& --
& \makecell[l]{R ($\gamma=1$)\\M ($\gamma=1,\nu=5/2$)}
& \makecell[l]{R ($\gamma=10$)\\M ($\gamma=100,\nu=5/2$)}
& $u,v$ (10 PCs)
& -- \\

DeepONet
& --
& --
& --
& --
& Large \\

DeepONet-C
& --
& --
& --
& --
& Medium \\

MIONet
& --
& --
& --
& --
& Medium \\

MNO
& --
& --
& --
& --
& Large \\
\bottomrule
\end{tabular}
\caption{Hyperparameter choices and implementation details for the product-space learning experiments on the nonlinear Klein--Gordon equation.}
\label{tab:kleinGordonPS}
\end{table}

\begin{table}[H]
\centering
\small
\begin{tabular}{lccccc}
\toprule
Method & $k$ & $k_W$ & $k_U$ & PCA & Configuration \\
\midrule
\KernelMOOV
& \makecell[l]{R ($\gamma=0.1$)\\M ($\gamma=0.1,\nu=5/2$)}
& --
& --
& $\alpha,v$ (10 PCs)
& -- \\

\KernelMOPS
& --
& \makecell[l]{R ($\gamma=100$)\\M ($\gamma=1,\nu=5/2$)}
& \makecell[l]{R ($\gamma=100$)\\M ($\gamma=0.1,\nu=5/2$)}
& $\alpha,u,v$ (10 PCs)
& -- \\

DeepONet-C
& --
& --
& --
& --
& Medium \\

MIONet
& --
& --
& --
& --
& Medium \\

MNO
& --
& --
& --
& --
& Large \\
\bottomrule
\end{tabular}
\caption{Hyperparameter choices and implementation details for the operator-valued learning experiments on the parametric diffusion--reaction equation.}
\label{tab:parametricDiffusionReactionOV}
\end{table}

\begin{table}[H]
\centering
\small
\begin{tabular}{lccccc}
\toprule
Method & $k$ & $k_W$ & $k_U$ & PCA & Configuration \\
\midrule
KernelO
& \makecell[l]{R ($\gamma=100$)\\M ($\gamma=100,\nu=3/2$)}
& --
& --
& $u,v$ (10 PCs)
& -- \\

\KernelMOPS
& --
& \makecell[l]{R ($\gamma=1$)\\M ($\gamma=1,\nu=5/2$)}
& \makecell[l]{R ($\gamma=1000$)\\M ($\gamma=1000,\nu=5/2$)}
& $\alpha,u,v$ (10 PCs)
& -- \\

DeepONet
& --
& --
& --
& --
& Small \\

DeepONet-C
& --
& --
& --
& --
& Small \\

MIONet
& --
& --
& --
& --
& Medium \\

MNO
& --
& --
& --
& --
& Large \\
\bottomrule
\end{tabular}
\caption{Hyperparameter choices and implementation details for the product-space learning experiments on the parametric diffusion--reaction equation.}
\label{tab:parametricDiffusionReactionPS}
\end{table}

\begin{table}[H]
\centering
\small
\begin{tabular}{lccccc}
\toprule
Method & $k$ & $k_W$ & $k_U$ & PCA & Configuration \\
\midrule
\KernelMOOV
& \makecell[l]{R ($\gamma=1$)\\M ($\gamma=1,\nu=5/2$)}
& --
& --
& $\alpha,v$ (10 PCs)
& -- \\

\KernelMOPS
& --
& \makecell[l]{R ($\gamma=1$)\\M ($\gamma=1,\nu=5/2$)}
& \makecell[l]{R ($\gamma=1000$)\\M ($\gamma=1,\nu=5/2$)}
& $\alpha,u,v$ (10 PCs)
& -- \\

DeepONet-C
& --
& --
& --
& --
& Medium \\

MIONet
& --
& --
& --
& --
& Small \\

MNO
& --
& --
& --
& --
& Large \\
\bottomrule
\end{tabular}
\caption{Hyperparameter choices and implementation details for the operator-valued learning experiments on the parametric wave equation.}
\label{tab:parametricWaveOV}
\end{table}

\begin{table}[H]
\centering
\small
\begin{tabular}{lccccc}
\toprule
Method & $k$ & $k_W$ & $k_U$ & PCA & Configuration \\
\midrule
KernelO
& \makecell[l]{R ($\gamma=10$)\\M ($\gamma=100,\nu=3/2$)}
& --
& --
& $u,v$ (10 PCs)
& -- \\

\KernelMOPS
& --
& \makecell[l]{R ($\gamma=1$)\\M ($\gamma=1,\nu=5/2$)}
& \makecell[l]{R ($\gamma=1000$)\\M ($\gamma=1000,\nu=5/2$)}
& $\alpha,u,v$ (10 PCs)
& -- \\

DeepONet
& --
& --
& --
& --
& Large \\

DeepONet-C
& --
& --
& --
& --
& Medium \\

MIONet
& --
& --
& --
& --
& Small \\

MNO
& --
& --
& --
& --
& Large \\
\bottomrule
\end{tabular}
\caption{Hyperparameter choices and implementation details for the product-space learning experiments on the parametric wave equation.}
\label{tab:parametricWavePS}
\end{table}

\end{document}